\documentclass{article}

\usepackage{amsmath,amsfonts,bm}
\usepackage{bbm}
\usepackage[mathscr]{eucal}

\DeclareMathOperator{\rank}{rank}

\DeclareMathOperator*{\argmin}{arg\,min}

\DeclareMathOperator{\diag}{diag}
\DeclareMathOperator{\supp}{supp}

\DeclareMathOperator{\sgn}{sgn}

\def\1{\bm{1}}

\def\<{\langle}
\def\>{\rangle}

\def\valpha{{\bm{\alpha}}}

\def\vtheta{{\bm{\theta}}}

\def\vnu{{\bm{\nu}}}

\def\vomega{{\bm{\omega}}}
\def\va{{\bm{a}}}
\def\vb{{\bm{b}}}
\def\vc{{\bm{c}}}

\def\vu{{\bm{u}}}

\def\vw{{\bm{w}}}
\def\vx{{\bm{x}}}
\def\vy{{\bm{y}}}
\def\vz{{\bm{z}}}

\def\mA{{\bm{A}}}

\def\mD{{\bm{D}}}

\def\mG{{\bm{G}}}

\def\mI{{\bm{I}}}

\def\mM{{\bm{M}}}

\def\mX{{\bm{X}}}

\DeclareMathAlphabet{\mathsfit}{\encodingdefault}{\sfdefault}{m}{sl}
\SetMathAlphabet{\mathsfit}{bold}{\encodingdefault}{\sfdefault}{bx}{n}

\def\cL{{\mathcal{L}}}

\def\cP{{\mathcal{P}}}

\def\bR{{\mathbb{R}}}

\newcommand{\R}{\mathbb{R}}

\usepackage{amsmath,amssymb,amsthm,amsfonts,bm,amsthm,color}
\usepackage{algorithm}
\usepackage{algpseudocode}
\usepackage{bbm}
\usepackage{fullpage}
\usepackage{enumitem}
\usepackage[normalem]{ulem}
\usepackage{pifont}
\usepackage{wrapfig}
\usepackage{pgfplots}
\pgfplotsset{compat=1.18}

\usepackage{caption}
\usepackage{subcaption}
\usepackage{appendix}

\usepackage{natbib}
\setcitestyle{authoryear,round}

\usepackage{tcolorbox}
\usepackage[utf8]{inputenc} 
\usepackage[T1]{fontenc}    
\usepackage{hyperref}       
\usepackage{url}            
\usepackage{booktabs}       
\usepackage{amsfonts}       
\usepackage{nicefrac}       
\usepackage{microtype}      
\usepackage{xcolor}         
\usepackage[nameinlink, capitalize]{cleveref}

\crefname{appendix}{Appendix}{Appendices}
\Crefname{appendix}{Appendix}{Appendices}
\AtBeginEnvironment{appendices}{%
  \crefalias{section}{appendix}%
  \crefalias{subsection}{appendix}%
  \crefalias{subsubsection}{appendix}%
}

\newtheorem{theorem}{Theorem}[section]
\newtheorem{corollary}{Corollary}[theorem]
\newtheorem{proposition}{Proposition}[section]
\newtheorem{lemma}{Lemma}[section]

\newtheorem{remark}{Remark}

\newcommand{\si}{s_{\mathrm{in}}}
\newcommand{\so}{s_{\mathrm{out}}}

\title{\textbf{Global Minimizers of $\ell^p$-Regularized Objectives \\ Yield the Sparsest ReLU Neural Networks}}
\author{%
  Julia Nakhleh \\
  Department of Computer Science\\
  University of Wisconsin-Madison\\
  \texttt{jnakhleh@wisc.edu} \\
  \and
  Robert D. Nowak \\
  Department of Electrical \& Computer Engineering \\
  University of Wisconsin-Madison\\
  \texttt{rdnowak@wisc.edu}
}
\date{}

\begin{document}

\maketitle

\begin{abstract}
Overparameterized neural networks can interpolate a given dataset in many different ways, prompting the fundamental question: which among these solutions should we prefer, and what explicit regularization strategies will provably yield these solutions? This paper addresses the challenge of finding the sparsest interpolating ReLU network—i.e., the network with the fewest nonzero parameters or neurons—a goal with wide-ranging implications for efficiency, generalization, interpretability, theory, and model compression. Unlike post hoc pruning approaches, we propose a continuous, almost-everywhere differentiable training objective whose global minima are guaranteed to correspond to the sparsest single-hidden-layer ReLU networks that fit the data. This result marks a conceptual advance: it recasts the combinatorial problem of sparse interpolation as a smooth optimization task, potentially enabling the use of gradient-based training methods. Our objective is based on minimizing $\ell^p$ quasinorms of the weights for $0 < p < 1$, a classical sparsity-promoting strategy in finite-dimensional settings. However, applying these ideas to neural networks presents new challenges: the function class is infinite-dimensional, and the weights are learned using a highly nonconvex objective. We prove that, under our formulation, global minimizers correspond exactly to sparsest solutions.  Our work lays a foundation for understanding when and how continuous sparsity-inducing objectives can be leveraged to recover sparse networks through training.
\end{abstract}

\section{Introduction}
Highly overparameterized neural networks have become the workhorse of modern machine learning. Because these networks can interpolate a given dataset in many different ways (see e.g. \cref{fig:2d_ridge,fig:2d_cc}), explicit regularization is frequently incorporated into the training procedure to favor solutions that are, in some sense, more regular or desirable. In this work, we focus on explicit regularizers which yield \textit{sparse} single-hidden-layer ReLU interpolating networks, which for our purposes are those with the fewest nonzero input weight/bias parameters among the active neurons.\footnote{In the univariate-input case, this is equivalent to the count of active neurons.}  Sparse models are particularly desirable for computational efficiency purposes, as they have lower storage requirements and computational overhead when deployed at inference time, and may have other attractive properties in terms of generalization, interpretability, and robustness (\cite{mozer1988skeletonization, guo2018sparse, liao2022achieving, liu2022robust}, among many others).

Although a myriad of sparsity-inducing training schemes have been proposed in the neural network literature, almost none of them have actually been proven to yield true \textit{sparsest} solutions, and the justifications for their use remain almost entirely heuristic and/or empirical. Furthermore, many such strategies rely on complex pruning pipelines---composed of iterative magnitude thresholding, fine-tuning, and sensitivity analyses---which are computationally costly, difficult to implement, and offer no theoretical guarantees in terms of the resulting sparsity. In contrast, we propose a simple regularization objective, based on the $\ell^p$ quasinorm of the network weights for $0 < p < 1$, whose global minimizer is \textit{provably} a sparsest interpolating ReLU network for sufficiently small $p$. This objective is continuous and differentiable away from zero, making it compatible with gradient descent. Although $\ell^p$-norm minimization with $0 < p < 1$ has been studied in finite-dimensional linear problems (most extensively in the context of compressed sensing), where it is known to guarantee sparsity under certain assumptions on the data/measurements, its behavior in the context of neural networks---wherein the features themselves are continuously parameterized and data-adaptive---is challenging to characterize mathematically, and to our knowledge, we are the first to do so. Specifically, our contributions are the following:
\begin{enumerate}[leftmargin=*]
    \item \textbf{Sparsity, uniqueness, and width/parameter bounds for univariate $\ell^p$-regularized networks.} In \cref{sec:univariate}, we prove that, for single-hidden-layer ReLU networks of input dimension one, minimizing the network's \textit{$\ell^p$ path norm} (see \eqref{opt:min_p_NN}) implicitly minimizes both its $\ell^1$ path norm (i.e., the total variation of its derivative) and, for sufficiently small $p > 0$, its $\ell^0$ path norm (total knot/neuron count). We show that for \textit{any} $0 < p < 1$, a minimum $\ell^p$ path norm interpolant of $N$ data points has no more than $N-2$ active neurons. In contrast, $\ell^1$ path norm minimization alone is \textit{not} guaranteed to implicitly minimize sparsity, and may yield solutions with arbitrarily many neurons (\cref{fig:l1_sparse_nonsparse}).  Our result follows from reframing the network training problem as an optimization over continuous piecewise linear (CPWL) functions which interpolate a dataset with minimal $p$-variation \eqref{eq:V_p} of the derivative. Using this variational framework, we can explicitly describe the optimal functions' behavior based on the geometry of the data points. This characterization provides data-dependent bounds on the sparsity and weight magnitudes of such minimum-$\ell^p$ solutions, and highlights an easily-verifiable condition on the data under which $\ell^p$ minimization for \textit{any} $0 < p < 1$ yields a sparsest interpolant ($\ell^0$ solution). Additionally, our analysis shows that the solution to the univariate $\ell^p$ minimization problem is \textit{unique} for almost every $0 < p < 1$; in contrast, univariate $\ell^0$ and $\ell^1$ solutions are both known to be non-unique in general (\cite{debarre2022sparsest, hanin2022implicit}). 
    \item \textbf{Exact sparsity in arbitrary input dimensions.} In \cref{sec:multivariate}, we show for networks of arbitrary input dimension that the problem of minimizing the network's $\ell^p$ path norm can be recast as a finite-dimensional minimization of a continuous, strictly concave function over a polytope.\footnote{We use \textit{polyhedron} to refer to an intersection of finitely many closed halfspaces, and \textit{polytope} to refer to a bounded polyhedron. Both are necessarily convex.} Using this reformulation, we show that there always exists some data-dependent threshold $p^*$ below which $\ell^p$ minimization recovers an $\ell^0$ (sparsest) solution, in terms of the count of nonzero parameters of the active neurons in the network. We also show that no $\ell^p$ (for any $0 < p < 1$) or $\ell^0$ solutions has more than $N$ active neurons and, if the data is in general position, any $\ell^0$ solution has $O(N)$ active input weight/bias parameters among these active neurons (\cref{prop:width_invariance_multivar}).
    \item \textbf{A principled, differentiable objective for sparse ReLU networks.} Our theory provides the first rigorous justification for using a smooth $\ell^p$ penalty for $0 < p < 1$ to obtain truly sparsest interpolating ReLU networks via gradient-based methods---no pruning or complex post-hoc approaches required.
\end{enumerate}
\begin{figure}
    \centering
    \begin{subfigure}{0.37\linewidth}
          \centering
          \includegraphics[width=0.9\linewidth]{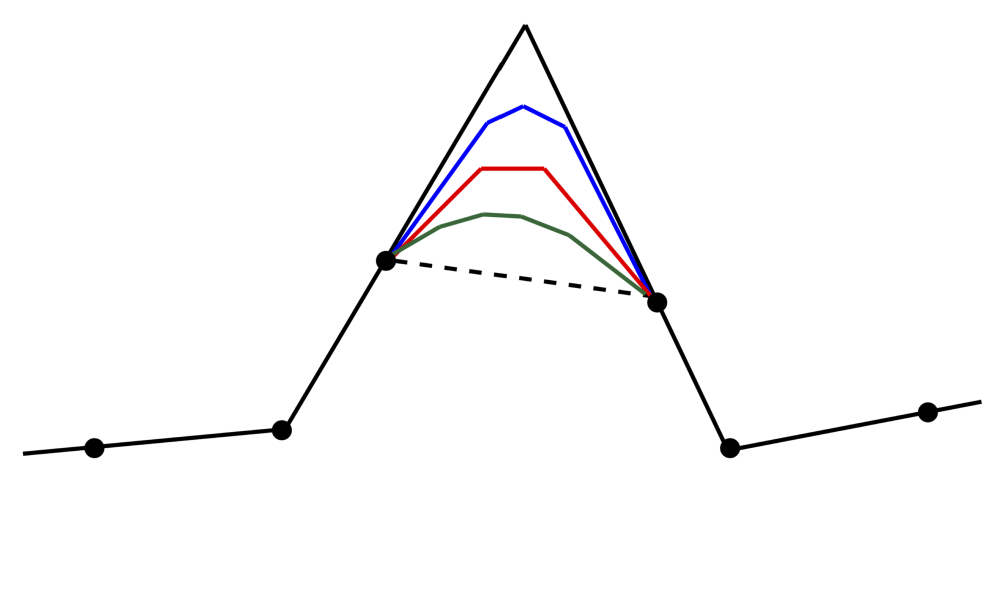}
          \caption{}
          \label{fig:l1_sparse_nonsparse}
    \end{subfigure}
    \begin{subfigure}{0.27\linewidth}
          \centering
          \includegraphics[width=0.9\linewidth]{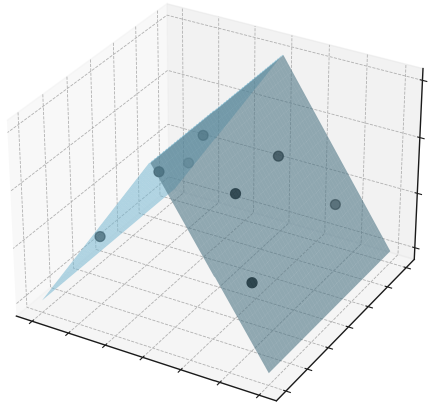}
          \caption{}
          \label{fig:2d_ridge}
    \end{subfigure}
    \begin{subfigure}{0.27\linewidth}
          \centering
          \includegraphics[width=0.9\linewidth]{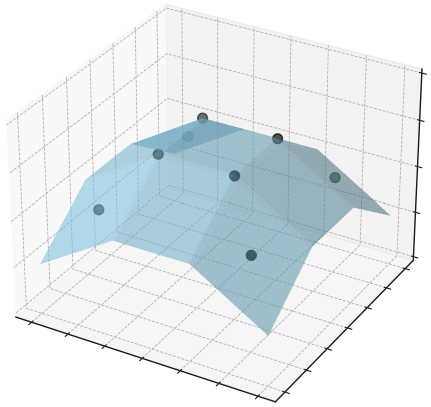}
          \caption{}
          \label{fig:2d_cc}
    \end{subfigure}
    \caption{\cref{fig:l1_sparse_nonsparse} shows several univariate min-$\ell^1$ path norm interpolants of a given dataset. Such solutions are generally non-unique, and always include at least one sparsest interpolant (black), but also include arbitrarily non-sparse interpolants (blue, red, green). \cref{fig:2d_ridge,fig:2d_cc}: two different ReLU network interpolants of the same 2D dataset with different numbers of active neurons and parameters. \cref{fig:2d_ridge} has 5 nonzero input weight/bias parameters (its $\ell^0$ path norm as in \eqref{opt:min_0_NN_multi}), while \cref{fig:2d_cc} has 16.}
    \label{fig:sparse_nonsparse}
\end{figure}

\section{Related work}
\paragraph{Sparsity via $\ell^p$ minimization in finite-dimensional linear models:} $\ell^p$ penalties with $0 < p \leq 1$ for linear constraint problems have been studied extensively in the compressed sensing literature, and have been shown to yield exact $\ell^0$ minimizers under certain conditions (typically involving restricted isometry and/or null-space constants) on the measurement matrix (\cite{candes2005decoding, chartrand2007exact, chartrand2008restricted, foucart2009sparsest}). Such penalties have also been studied in the statistics literature under the name \textit{bridge regression} (\cite{frank1993statistical, knight2000asymptotics, fan2001variable}). Existing theory in these areas is highly dependent on the fixed, finite-dimensional nature of the linear constraint, and is not readily adaptable to the neural network context, wherein the features themselves are adaptively learned.

\paragraph{$\ell^1$ path norm regularization in single-hidden-layer ReLU networks:} \cite{neyshabur2015norm} showed that the $\ell^1$ path norm of a single-hidden-layer ReLU network controls its Rademacher complexity and thus its generalization gap, but do not directly address the question of sparsity. In the context of \textit{infinite-width} ReLU networks, the problem of minimum-$\ell^1$ path norm interpolation is known to have solutions with no more active neurons than the number of data points (\cite{parhi2021banach,parhi2022kinds, shenouda2024variation}).\footnote{For input dimension greater than one, the $\ell^1$ path norm $\sum_{k=1}^K |v_k| \| \vw_k \|_2$ studied in those works differs from the one we consider in \eqref{opt:min_p_NN_multi}, which is equivalent to $\sum_{k=1}^K |v_k| \| \vw_k \|_1$ for $p = 1$.} However, solutions to that problem are known to be non-unique, and generally include interpolating ReLU networks with arbitrarily many active neurons (\cite{hanin2022implicit, debarre2022sparsest}). \cite{nakhleh2024new} show that a variant of $\ell^1$ path norm minimization applied to univariate-input, multi-output networks always yields a solution with no more than $N$ active neurons, but this solution rarely coincides with the sparsest solution unless the dataset is of a very particular form. Therefore, $\ell^1$ path norm regularization applied to single-hidden-layer ReLU networks is \textit{not} generally guaranteed to produce sparsest solutions.

\paragraph{Empirical methods for training sparse neural networks:} A large body of research has been dedicated to sparsity-promoting neural network training schemes. Here we briefly summarize some of the most well-known strategies as well as some which resemble our proposed regularization approach; our list is by no means comprehensive. Earlier works suggested using $\ell^1$ and $\ell^2$ penalties to encourage small network weights (\cite{ng2004feature, hinton1993keeping}) or applying post-training pruning approaches (\cite{lecun1989optimal, hassibi1993optimal}). More recent pruning schemes incorporate pruning iteratively into training (\cite{han2015learning, guo2016dynamic, frankle2018lottery, zhang2018systematic, zhou2019deconstructing}). Group lasso-type penalties to induce structured sparsity over neurons or channels have also been suggested (\cite{wen2016learning,scardapane2017group}). Other approaches include $\ell^0$ approximation using explicit gating mechanisms (\cite{louizos2017learning, srinivas2017training}) and variational dropout (\cite{molchanov2017variational}). Another line of research uses reparameterization tricks to replace non-smooth sparsifying objectives with smooth versions that share the same local and global minimizers (\cite{ziyin2023spred,kolb2023smoothing,kolb2025deep}). Finally, a number of different algorithms for $\ell^p$-type regularization ($p < 1$) in neural networks have been proposed and evaluated experimentally (\cite{wu2014batch,khan2018bridgeout,tang2023training,outmezguine2024decoupled,ji2025network}). While these methods have demonstrated empirical success in training sparse networks, existing theory does not guarantee that any of them will find sparsest solutions. Moreover, these approaches often require complex multi-stage pipelines and are computationally costly to implement.

\paragraph{Provable sparsest-recovery in specialized neural network settings:} In the 1D input case, \cite{boursier2023penalising} show that, under certain assumptions on the data---namely, that the data contains no more than three consecutive points on which the straight-line interpolant is strictly convex or concave---interpolation using a bias-penalized $\ell^1$ path norm regularizer will select a sparsest interpolant of the dataset. As we will see in \cref{sec:univariate}, this assumption on the data is rather restrictive, and our analysis does not require it. Their proof is also not readily extendable to multivariate inputs. \cite{debarre2022sparsest} characterize the \textit{sparsest} min-$\ell^1$ path norm interpolants in the univariate case and provide an algorithm for explicitly constructing one such solution. \cite{ergen2021convex} show that $\ell^1$ path norm minimization yields solutions with a minimal number of active neurons \textit{if} the data dimension is greater than the number of samples (precluding the univariate-input case) and the data satisfy special assumptions, such as whitened data. \cite{fridovich2025recovery} show that an iterative hard thresholding algorithm applied to shallow ReLU networks recovers sparsest solutions with high probability if the data is Gaussian. In contrast, our sparsity results do not require any assumptions on the data, and provide exact sparsity guarantees in arbitrary input dimension.

\section{Univariate $\ell^p$-regularized neural networks} \label{sec:univariate}
Here we consider single-hidden-layer $\R \to \R$ ReLU neural networks of the form
\begin{align} \label{eq:nn}
    f_{\vtheta}(x) := \sum_{k=1}^K v_k ( w_kx + b_k)_+ + ax + c
\end{align}
where $(\cdot)_+ := \max\{0,\cdot\}$ is the ReLU function, $\vtheta := \big\{ \{ w_k, b_k, v_k \}_{k=1}^K, a, c \big \}$ is the collection of network parameters, and all parameters are $\R$-valued. For a given dataset $(x_1, y_1), \dots, (x_N, y_N) \in\mathbb{R} \times \mathbb{R}$, a fixed $p \in (0,1]$, and a fixed width $K \geq N$,\footnote{Here and in \cref{sec:multivariate} we fix $K \geq N$ because interpolation in any dimension is possible with $K = N$ neurons (\cite{bubeck2020network}, Proposition 2). We will show that solution sets of the $\ell^p$ and $\ell^0$ path norm minimization problems for any input dimension are invariant to the selection of $K$ as long as $K \geq N$ (\cref{corr1,prop:width_invariance_multivar}).} consider the following problem:

\begin{equation} \label{opt:min_p_NN}
 \argmin_{\vtheta}  \sum_{k=1}^{K} |w_k v_k|^p \ , \ \mbox{subject to } f_{\vtheta}(x_i)=y_i, \,i=1,\dots,N
\end{equation}
We will refer to the quantity being minimized in \eqref{opt:min_p_NN} as the network's $\ell^p$ \textit{path norm}. Additionally, consider the ``sparsifying'' problem
\begin{equation} \label{opt:min_0_NN}
 \argmin_{\vtheta} \sum_{k=1}^K \mathbbm{1}_{w_k v_k \neq 0}  \ , \ \mbox{subject to } f_{\vtheta}(x_i)=y_i, \,i=1,\dots,N
\end{equation}
where the $\ell^0$ path norm $\sum_{k=1}^K \mathbbm{1}_{w_k v_k \neq 0}$---which is equivalent to the limit of the $\ell^p$ path norm as $p \downarrow 0$---counts the number of active neurons in the network. 

In this section, we will analyze the relationship between solutions of \eqref{opt:min_p_NN} and \eqref{opt:min_0_NN} in terms of their represented functions, and show that these functions can be explicitly described in terms of the geometry of the data points. This characterization (\cref{th:geom_char}) shows that solutions to \eqref{opt:min_p_NN} for any $0 < p < 1$ are necessarily also solutions for $p = 1$, immediately implying data-dependent bounds on the network's parameters and Lipschitz constant. This description also allows problem \eqref{opt:min_p_NN} to be reduced to a minimization of a continuous, strictly concave function over a polytope. From there, we show in \cref{th:main} that solutions to \eqref{opt:min_p_NN} are unique (in terms of their represented functions) for Lebesgue-almost every $0 < p < 1$ and that, for small enough $p$, this unique optimal function is also a \textit{sparsest} interpolant of the data (i.e., a solution to \eqref{opt:min_0_NN}). Furthermore, if the data meets certain easily-verifiable geometric assumptions, solutions to \eqref{opt:min_p_NN} for \textit{any} $0 < p < 1$ are solutions to the sparsest-interpolation problem \eqref{opt:min_0_NN}.

\subsection{Variational reformulation of \eqref{opt:min_p_NN} and \eqref{opt:min_0_NN}} \label{sec:uni_variational_reformulation}
We begin by showing that problems \eqref{opt:min_p_NN} and \eqref{opt:min_0_NN} can be equivalently expressed as a type of variational problem over the set of continuous piecewise linear (CPWL) functions which interpolate the data. This equivalence is critical for the analysis in this section, since it allows solutions to \eqref{opt:min_p_NN} and \eqref{opt:min_0_NN} to be characterized geometrically in terms of their represented functions and their local behavior around data points. Here, we let $S_{\vtheta,p}^*$ (resp. $S_{\vtheta,0}^*$)  denote the set of parameters of optimal neural networks which solve \eqref{opt:min_p_NN} (resp. \eqref{opt:min_0_NN}) for a given dataset, and let
\begin{align}
    S_p^* := \{ f: \R \to \R \ \rvert \ f = f_\vtheta, \ \vtheta \in S_{\vtheta,p}^* \}
\end{align}
be the set of functions represented by neural networks with optimal parameters in $S_{\vtheta,p}^*$, for any $0 \leq p \leq 1$. 
\begin{proposition} \label{prop:existence_and_variational}
    For any $0 \leq p \leq 1$, the set $S_p^*$ is exactly the solution set of
    \begin{equation} \label{opt:min_p_f}
     \argmin_{f} \   V_p(f) \ , \ \mbox{subject to } f(x_i)=y_i, \,i=1,\dots,N
    \end{equation}
    where the optimization in \eqref{opt:min_p_f} is taken over all  $f: \R \to \R$ which are continuous piecewise linear (CPWL) with at most $K$ knots. For such CPWL functions $f$, we define
    \begin{align} \label{eq:V_p}
        V_p(f) := \begin{cases} \sup_{\cP} \sum_{i=0}^{n_{\cP}-1} |Df(x_{i+1}) - Df(x_i)|^p =
            \sup_{\pi} \sum_{A \in \pi} |D^2 f (A)|^p, &\textrm{if $0 < p \leq 1$} \\
            \textrm{number of knots of $f$}, &\textrm{if $p = 0$}
        \end{cases}
    \end{align}
    with the first $\sup$ taken over all partitions $\cP = \{x_0 <  \dots < x_{n_{\cP}} \}$ of $\R$, and the second $\sup$ taken over partitions $\pi$ of $\R$ into countably many disjoint (Borel) measurable subsets. In particular, $S_0^*$ is non-empty.
\end{proposition}
\begin{remark}
    For $p \in (0,1]$, $V_p(f)$ is the $p$-variation (\cite{dudley2006product}, Part II.2) of the distributional derivative $Df$ (in the sense of functions), or equivalently of the second distributional derivative $D^2 f$ (in the sense of measures). In particular, for a CPWL function $f$ with knots at $u_1, \dots, u_K$ and corresponding slope changes $c_1, \dots, c_K$ at those knots, so that $D^2 f = \sum_{k=1}^K c_k \delta_{u_k}$, we have $$V_p(f) = \sum_{k=1}^K |c_k|^p$$ In the case $p = 1$, $V_1 (f)$ is exactly the total variation of $Df$ (in the sense of functions) and of $D^2 f$ (in the sense of measures), and the reformulation in \cref{prop:existence_and_variational} is equivalent to that of \cite{savarese2019infinite}. For a neural network where no two neurons ``activate'' at the same location (i.e., $b_k/w_k = b_{k'}/w_{k'}$ for $k \neq k'$), $V_p(f)$ is exactly the $\ell^p$ path norm of $f$ as defined above.
\end{remark}
The proof is in \cref{appendix:proof_prop_existence_and_variational}. \cref{prop:existence_and_variational} says that the set $S_p^*$ of functions represented by solutions to \eqref{opt:min_p_NN} is exactly the set of CPWL functions $f$ which interpolate the data with minimal sum of absolute slope changes, each taken to the $p^\textrm{th}$ power. In the case $p = 0$, solutions to \eqref{opt:min_0_NN} represent CPWL functions which interpolate the data with the fewest possible knots. This reformulation also shows that problem \eqref{opt:min_0_NN} is invariant to the choice of network width $K$, as long as $K$ is large enough to allow interpolation. As a consequence of \cref{th:geom_char}, we will see that this same width-invariance holds for problem \eqref{opt:min_p_NN}.

\subsection{Geometric characterization of solutions to \eqref{opt:min_p_f}} \label{sec:uni_geom_char}
Next, in \cref{th:geom_char}, we describe a set of geometric characteristics which any optimal network function $f \in S_p^*$ for $0 < p < 1$ must satisfy, and which at least one $f \in S_0^*$ satisfies. This characterization depends on the slopes $s_i := \frac{y_{i+1}-y_i}{x_{i+1}-x_i}$ of the straight lines $\ell_i$ connecting $(x_i, y_i)$ and $(x_{i+1}, y_{i+1})$. The \textit{discrete curvature} at a data point $x_i$ refers to $\epsilon_i := \sgn(s_i - s_{i-1})$, which is positive if the slope of the straight lines between consecutive data points increases at $x_i$, and negative if this slope decreases (with $\sgn(0) = 0$). 

In words, \cref{th:geom_char} says that the behavior of any $f \in S_p^*$ for $0 < p < 1$ is uniquely determined everywhere \textit{except} around sequences of more than three consecutive data points $x_i, \dots, x_{i+m}$ with the same discrete curvature. On these ``constant-curvature'' regions of potential ambiguity, solutions must be convex (resp. concave) if the curvature of the data is positive (resp. negative), and can have at most $m$ knots on any such region. Additionally, \cref{th:geom_char} says that solutions to \eqref{opt:min_p_f} for $0 < p < 1$ have at most $N-2$ knots. Therefore, as in the case $p = 0$, we see that problem \eqref{opt:min_p_NN} is invariant (in terms of represented functions) to the choice of network width $K$, as long as $K \geq N-2$.

\begin{theorem} \label{th:geom_char}
    For $0 < p < 1$, solutions exist to \eqref{opt:min_p_f} (hence to \eqref{opt:min_p_NN}). For any such solution, its represented function $f \in S_p^*$ is CPWL and obeys the following:
    \begin{enumerate}
        \item \label{th:geom_char_1} $f$ is linear before $x_2$ and after $x_N$; between any three or more consecutive collinear data points; and between any two consecutive points $x_i$ and $x_{i+1}$ with opposite discrete curvature $\epsilon_i \neq \epsilon_{i+1}$.
        \item \label{th:geom_char_2} On any maximal set of $m$ consecutive data points $x_i, \dots, x_{i+m}$ with the same discrete curvature (i.e., $\epsilon_{i-1} \neq \epsilon_i = \epsilon_{i+1} = \dots = \epsilon_{i+m} \neq \epsilon_{i+m+1}$):
        \begin{enumerate} 
            \item \label{th:geom_char_2a} If $m = 1$, then $f$ has a single knot between $x_i$ and $x_{i+1}$, with incoming/outgoing slopes $s_{i-1}$ at $x_{i}$ and $s_{i+1}$ at $x_{i+1}$.
            \item \label{th:geom_char_2b} If $m \geq 2$, then $f$ has incoming slope $s_{i-1}$ at $x_i$ and outgoing slope $s_{i+m}$ at $x_{i+m}$. Between $x_i$ and $x_{i+m}$, $f$ takes on at most $m-1$ slopes $u_1, \dots, u_{m-1}$ distinct from $s_{i-1}$ and $s_{i+m}$. Each $u_j$ is between $s_{i+j-1}$ and $s_{i+j}$, inclusive, and its corresponding segment passes through $(x_{i+j}, y_{i+j})$.
        \end{enumerate}
    \end{enumerate}
    Furthermore, there is always some $f \in S_0^*$ which obeys the above description. (See illustration in \cref{fig:geom_char}.)
\end{theorem}
\begin{corollary} \label{corr1}
    Any minimum $\ell^p$ path norm interpolant of the data for $0 < p < 1$ is also a minimum $\ell^1$ path norm interpolant, and can be represented by a network with no more than $N-2$ neurons.
\end{corollary}

\begin{figure}
    \begin{subfigure}{\textwidth}
          \centering
          \includegraphics[width=0.9\linewidth]{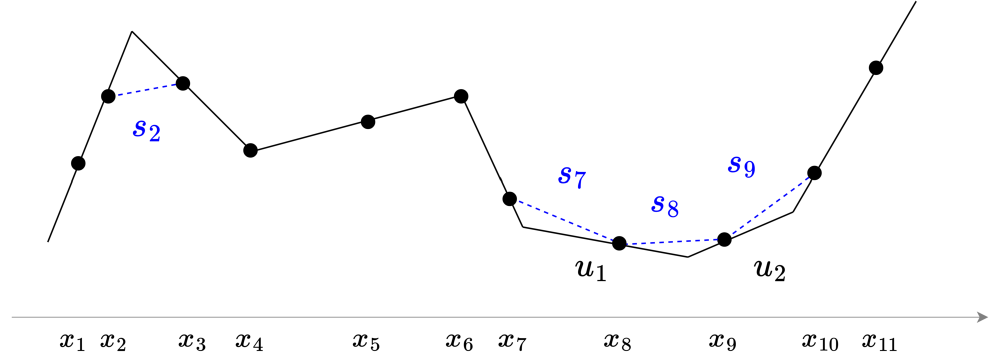}
          \caption{A function satisfying the description in \cref{th:geom_char}.}
          \label{fig:geom_char_sol0}
    \end{subfigure}
    \par\bigskip
    \begin{subfigure}{.5\textwidth}
          \centering
          \includegraphics[width=0.9\linewidth]{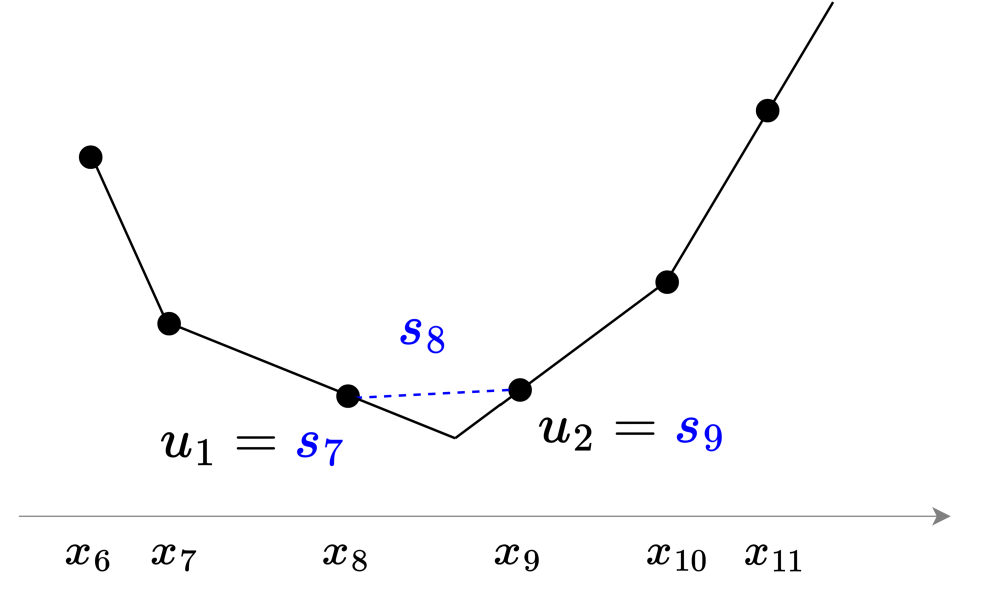}
          \caption{A second possible function on $[x_7, x_{10}]$.}
          \label{fig:geom_char_sol1}
    \end{subfigure}
    \begin{subfigure}{.5\textwidth}
          \centering
          \includegraphics[width=0.9\linewidth]{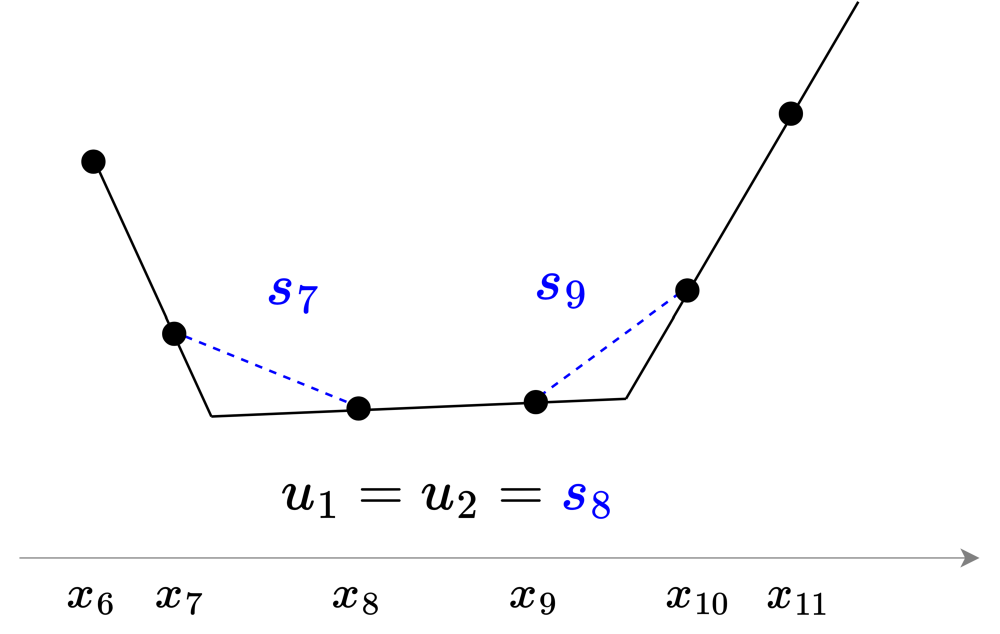}
          \caption{A third possible function on $[x_7, x_{10}]$.}
          \label{fig:geom_char_sol2}
    \end{subfigure}
    \caption{Illustration of \cref{th:geom_char}. By \cref{th:geom_char},\ref{th:geom_char_1}, any $f \in S_p^*$ for $0 < p < 1$ must agree with the function in \cref{fig:geom_char_sol0} on $(-\infty, x_7]$ and $[x_{10}, \infty)$. The only possible ambiguity occurs between $x_7$ and $x_{10}$, where all points have the same discrete curvature. Here the function behavior is described by \cref{th:geom_char},\ref{th:geom_char_2b}. \cref{fig:geom_char_sol1,fig:geom_char_sol2} show two other functions whose behavior on $[x_7, x_{10}]$ also concurs with \cref{th:geom_char},\ref{th:geom_char_2b}.}
    \label{fig:geom_char}
\end{figure}

The set $S_1^*$ of optimal neural network functions for $p = 1$ has been fully characterized in previous work (\cite{hanin2022implicit, debarre2022sparsest}), which showed that any interpolant $f$ obeying the description in \cref{th:geom_char} is in $S_1^*$. Therefore, \cref{th:geom_char} shows that any solution to \eqref{opt:min_p_f} (hence to \eqref{opt:min_p_NN}) for $0 < p < 1$ is also a solution for $p = 1$. This result is interesting because, as our proof of \cref{th:geom_char} shows, problem \eqref{opt:min_p_f} generally has multiple solutions for $p = 0$, many of which are \textit{not} solutions for $p = 1$ and may have arbitrarily large slope changes which cannot be bounded in terms of the data. Intuitively, the latter fact is unsurprising, since the objective $V_0(f)$ depends only on the number of knots of $f$, not on the magnitudes of the corresponding slope changes. One might therefore expect that penalizing $V_p$ for sufficiently small $p$ could also produce solutions with arbitrarily large slope changes (corresponding to networks with arbitrarily large weights), particularly in light of the equivalence between $V_p$ and $V_0$ penalization for sufficiently small $p$, as we demonstrate in \cref{sec:uni_uniqueness_sparsity}. However, \cref{th:geom_char} says that this is not the case. Therefore, in conjunction with \cref{th:main}, \cref{th:geom_char} says that for sufficiently small $p$, penalizing $V_p$ effectively penalizes both $V_0$ \textit{and} $V_1$ simultaneously: i.e., it selects a solution with the fewest possible knots (corresponding to a network with the fewest possible neurons), and whose weights are small in the sense that $\sum_{k=1}^K |v_k w_k|$ is minimal. In fact, \cref{th:geom_char} immediately implies the following data-dependent bounds on the parameters and on the network function's Lipschitz constant:
\begin{corollary}
    Any solution $\vtheta$ to \eqref{opt:min_p_NN} for $0 < p < 1$ has $\max_{k=1, \dots, K} |v_k w_k| \leq \sum_{k=1}^K |v_k w_k| = \sum_{i=1}^{N-2} |s_{i+1}-s_i|$, and Lipschitz constant $L \leq \max_{i=1, \dots, N-1} |s_i|$.
\end{corollary}

Regarding the $N-2$ neuron bound in \cref{corr1}, we note that this bound applies to \textit{any} minimum $\ell^p$ path norm solution for any $0 < p < 1$. In contrast, there exist minimum $\ell^1$ path norm solutions with $N-2$ knots, but also solutions with arbitrarily many knots (\cite{hanin2022implicit, debarre2022sparsest}); see \cref{fig:l1_sparse_nonsparse}. Solutions for $0 < p < 1$ are thus guaranteed a certain level of sparsity which is \textit{not} enforced by $p = 1$ minimization alone. Sparsest (minimum $\ell^0$) solutions---which we soon show will coincide with an $\ell^p$ path norm solution for small enough $p$---are known to have as many as $N-2$ active neurons and as few as $O(N/2)$ neurons, depending on the structure of data (\cite{debarre2022sparsest}).

The proof of \cref{th:geom_char} hinges mainly on two auxiliary results, detailed in \cref{appendix:aux_lemmas_univariate}, which describe the local behavior of any optimal $f \in S_p^*$ between consecutive data points in terms of $f$'s incoming and outgoing slopes at those points. This allows us to characterize when a knot can be removed from any interpolating function while maintaining interpolation and reducing its regularization cost $V_p$. The full proof is in \cref{appendix:proof_geom_char}.

\subsection{Uniqueness and sparsity of solutions to \eqref{opt:min_p_f} for $0 < p < 1$} \label{sec:uni_uniqueness_sparsity}

Using \cref{th:geom_char}, we show that solutions to \eqref{opt:min_p_f} are unique for almost every $0 < p < 1$, and for sufficiently small $0 < p < 1$, correspond with globally sparsest interpolants (i.e., interpolants with the fewest total knots). Additionally, \cref{th:geom_char} shows that under an easily-verifiable condition on the data, penalizing $V_p$ for \textit{any} $0 < p < 1$ yields a sparsest interpolant. In conjunction with \cref{th:geom_char}, this result tells us that for univariate data, $\ell^p$ path norm minimization for sufficiently small $p > 0$ simultaneously minimizes both the $\ell^1$ and $\ell^0$ path norms, producing a unique solution which is both maximally sparse and controlled in terms of its parameters' magnitudes. We note that almost-everywhere uniqueness of solutions to \eqref{opt:min_p_f} occurs \textit{only} in the $0 < p < 1$ case. In contrast, solutions for both $p = 0$ and $p = 1$ are non-unique in general, and for $p = 1$, they may have infinitely many knots/neurons (\cite{debarre2022sparsest}, \cite{hanin2022implicit}). 

\begin{theorem} \label{th:main}
    For all but finitely many $0 < p < 1$, the solution to \eqref{opt:min_p_f} is unique.\footnote{Uniqueness here and in the remainder of the discussion only in terms of functions which interpolate the data with the same set of absolute slope changes. If the data contains special symmetries, it may admit multiple distinct interpolating functions which have the same set of absolute slope changes (corresponding to interpolating networks with the same weights).} Furthermore, there is some data-dependent $p^*$ such that the unique solution to \eqref{opt:min_p_f} for any $0 < p < p^*$ is a solution for $p = 0$. If the data contains no more than two consecutive points with the same discrete curvature, then the solution to \eqref{opt:min_p_f} for any $0 < p < 1$ is also a solution for $p = 0$.
\end{theorem}

The proof of \cref{th:main} is in \cref{appendix:proof_main_th}. It relies on \cref{th:geom_char} in combination with the \textit{Bauer maximum principle} (\cite{aliprantis2006infinite}, Theorem 4.104), which states that any continuous, strictly concave function over a closed, convex set attains a minimum at an extreme point of that set. The main idea is that, using \cref{th:geom_char}, we can recast the problem of finding the minimum-$V_p$ interpolant $f \in S$ (where $S$ denotes the set of functions which meet the description in \cref{th:geom_char}) as a minimization of a continuous, strictly concave function over the hypercube $[0,1]^{m-1}$. This reformulation is possible because, by \cref{th:geom_char}, the only place where these interpolants $f \in S$ may differ is around sequences of points $x_i, \dots, x_{i+m}$ (for $m \geq 2$) which all have the same nonzero discrete curvature. Using the description in \cref{th:geom_char},\ref{th:geom_char_2b}, the slopes $u_1, \dots, u_{m-1}$ of any $f \in S$ on such an interval $[x_i, x_{i+m}]$ can be expressed as convex combinations $u_j := (1-\alpha_j) s_{i+j-1} + \alpha_j s_{i+j}$, and any such solution $f \in S$ can be fully identified with its corresponding vector of the parameters $[\alpha_1, \dots, \alpha_{j-1}]^\top \in [0,1]^{m-1}$. Expressed in terms of these parameters $[\alpha_1, \dots, \alpha_{j-1}]^\top \in [0,1]^{m-1}$, the cost $V_p$ is strictly concave. Therefore, by the Bauer maximum principle, any $f \in S$ with minimal $V_p$ for $0 < p < 1$ must correspond to one of the finitely many vertices of the cube $[0,1]^{m-1}$. Having restricted the set of possible candidate solutions to this finite set (which can be shown to include at least one sparsest solution), the theorem statement follows from standard analysis arguments. 

In the next section, we will show that this general line of reasoning---recast the neural network optimization as a minimization of a strictly concave function over a polytope, and apply the Bauer maximum principle---can also be used to characterize the sparsity of $\ell^p$-regularized multivariate-input ReLU networks, although the machinery underlying the argument is very different.

\section{Multivariate $\ell^p$-regularized neural networks} \label{sec:multivariate}

Here we consider single-hidden-layer $\R^d \to \R$ ReLU neural networks of the form
\begin{align} \label{eq:multi_nn}
    f_{\vtheta}(\vx) := \sum_{k=1}^K v_k ( \vw_k^\top \overline{\vx})_+
\end{align}
with output weights $v_k \in \R$, input weights $\vw_k \in \R^{d+1}$, and $\overline{\vx} := [\vx^\top, 1]^\top$ augments the dimension of the input $\vx$ to account for a bias term. As before, $\vtheta := \{ v_k, \vw_k \}_{k=1}^K$ is the collection of network parameters. For a given dataset $(\vx_1, y_1), \dots, (\vx_N, y_N) \in\mathbb{R}^d \times \mathbb{R}$, and fixed constants $K \geq N$ and $0 < p < 1$, consider the minimum $\ell^p$ path norm interpolation problem
\begin{align} \label{opt:min_p_NN_multi}
 \argmin_{\vtheta}  \sum_{k=1}^{K} \| v_k \vw_k \|_p^p  \ , \ \mbox{subject to } f_{\vtheta}(\vx_i)=y_i, \,i=1,\dots,N
\end{align}
We will prove that, for small enough $p$, any solution to \eqref{opt:min_p_NN_multi} also solves the  ``sparsifying'' problem
\begin{equation} \label{opt:min_0_NN_multi}
 \argmin_{\vtheta}  \sum_{k=1}^{K}  \| v_k \vw_k \|_0 \ , \ \mbox{subject to } f_{\vtheta}(\vx_i)=y_i, \,i=1,\dots,N
\end{equation}

The multivariate $\ell^0$ path norm objective in \eqref{opt:min_0_NN_multi} counts the number of nonzero input weight/bias parameters of the \textit{active} neurons\footnote{A neuron $\vx \mapsto v_k(\vw_k^\top \overline{\vx})_+$ is \textit{active} if $v_k \vw_k \neq \bm{0}$; i.e., that neuron has a nonzero contribution to the network function.} in the network. We begin by upper bounding the sparsity of solutions to \eqref{opt:min_0_NN_multi} and showing that, as in the univariate case, problems \eqref{opt:min_p_NN_multi} and \eqref{opt:min_0_NN_multi} are invariant to the selection of width $K$ as long as $K \geq N$.
\begin{proposition} \label{prop:width_invariance_multivar}
    For any $K \geq N$ and any $0 < p < 1$, solutions to \eqref{opt:min_p_NN_multi} exist, and any such solution has at most $N$ active neurons. The same holds for \eqref{opt:min_0_NN_multi}. Furthermore, if the data $\vx_1, \dots, \vx_N$ are in general position,\footnote{A set of points $\vx_1, \dots, \vx_N \in \R^d$ are in \textit{general (linear) position} if no $k$ of them lie in a $k-2$ dimensional affine subspace, for $k = 2, 3, \dots, d+1$. If $N \geq d+1$, this is equivalent to the statement that no hyperplane contains more than $d$ points.} then any solution to \eqref{opt:min_0_NN_multi} has $ \sum_{k=1}^K \| v_k \vw_k \|_0 = O(N)$.
\end{proposition}

See proof in \cref{appendix:proof_width_invariance_multivar} for explicit constants in various cases.

To show the equivalence of problems \eqref{opt:min_p_NN_multi} and \eqref{opt:min_0_NN_multi} for sufficiently small $p$, we first show that both problems can be recast as finite- (albeit high-) dimensional optimizations over a linear constraint set. This reformulation is heavily inspired by Theorem 1 in \cite{pilanci2020neural}. Here we denote element-wise inequality for vectors $\va, \vb$ as $\va \leq \vb$. Let $\overline{\mX} := [\overline{\vx}_1, \dots, \overline{\vx}_N]^\top \in \R^{N \times (d+1)}$ be the matrix of augmented data points $\overline{\vx}_i := [\vx_i^\top, 1]^\top$, $\vy := [y_1, \dots, y_N]^\top \in \R^N$ be the vector of labels, and $\{ \mD_j \}_{j=1}^J$ be the collection of all $N \times N$ binary matrices of the form $\diag(\mathbbm{1}[\overline{\mX} \vu \geq \bm{0}])$ for some $\vu \in \R^{d+1}$. It is known that $J \leq2  \sum_{k=0}^{d} \binom{N-1}{k} $ (\cite{cover2006geometrical}).

\begin{lemma} \label{lemma:multivar_reparam} For any $0 < p < 1$ and any $K \geq N$, let $\vtheta = \{ v_k, \vw_k \}_{k=1}^K$ be a solution to \eqref{opt:min_p_NN_multi}. Then there is another solution $\vtheta' = \{ v_k', \vw_k' \}_{k=1}^K$ to \eqref{opt:min_p_NN_multi}, which is reconstructed from a solution $\{ \vnu_j', \vomega_j' \}_{j=1}^J$---which necessarily satisfies $|\{ j \ \rvert \  \vnu_j \neq 0 \}| + |\{ j \ \rvert \  \vomega_j \neq 0 \}| \leq N$---to the problem
    \begin{align} \label{opt:min_p_NN_reparam}
      \argmin_{\{ \vnu_j, \vomega_j \}_{j=1}^J \subset \R^{d+1}}  \sum_{j=1}^J \| \vnu_j \|_p^p + \| \vomega_j \|_p^p  \ , \ &\mbox{subject to } \sum_{j=1}^J \mD_j \overline{\mX} (\vnu_j - \vomega_j) = \vy, \\ &(2 \mD_j - \mI) \overline{\mX} \vnu_j \geq \bm{0}, \  (2 \mD_j - \mI) \overline{\mX} \vomega_j \geq \bm{0}, \ j = 1, \dots, J
    \end{align}
    as
    \begin{align} \label{eq:block_to_NN_formula}
        \{ \vw_k' \}_{k=1}^K &= \left\{ \vnu_{j}'/\alpha_j \ \rvert \  \vnu_j' \neq 0 \right\} \cup \left\{ \vomega_j'/\beta_j \ \rvert \  \vomega_j' \neq 0 \right\} \\
        \{ v_k' \}_{k=1}^K &= \left\{ \alpha_j \  \rvert \ \vnu_j' \neq 0 \right\} \cup \left\{ -\beta_j \ \rvert \  \vomega_j' \neq 0 \right\}
    \end{align}
    for any choice of $\alpha_1, \beta_1, \dots, \alpha_J, \beta_J > 0$. Both solutions satisfy $\sum_{k=1}^K \| v_k \vw_k \|_0 = \sum_{k=1}^K \| v_k' \vw_k' \|_0$ as well as $\sum_{k=1}^K  \| v_k \vw_k \|_q^q = \sum_{k=1}^K  \| v_k' \vw_k' \|_q^q$ for any $0 < q < 1$. The same statement holds for solutions $\vtheta$ to \eqref{opt:min_0_NN_multi}, with the objective in \eqref{opt:min_p_NN_reparam} replaced by $\sum_{j=1}^J \| \vnu_j \|_0 + \| \vomega_j \|_0$. 
\end{lemma}
The proof of \cref{lemma:multivar_reparam} is in \cref{appendix:proof_multivar_concave_convex}. The main idea is that although there are uncountably many ways to choose the neurons' parameters, there are only $J$ possible binary activation patterns, i.e., vectors representing whether a given neuron is active on each data point. By combining all neurons which induce the same activation pattern into two neurons (one with positive output weight and one with negative output weight), the network's output and $\ell^p$ path norm can be expressed as a sum over these $J$ activation patterns. The equality constraint in \eqref{opt:min_p_NN_reparam} reflects the data-fitting requirement, and the inequality constraints ensure that each $\vnu_j, \vomega_j$ correspond appropriately to the activation pattern $\mD_j$ in order for the reconstruction formula \eqref{eq:block_to_NN_formula} to hold. With this reformulation in hand, we are ready for the main result of this section: 
\begin{theorem} \label{th:multivar}
    For any dataset, there is some data-dependent $p^*$ such that any solution to \eqref{opt:min_p_NN_multi} for any $0 < p < p^*$ is a solution to \eqref{opt:min_0_NN_multi}.
\end{theorem}

The proof, presented in \cref{appendix:proof_multivar_th}, follows the approaches of \cite{yang2022sparse,peng2015np} in analyzing sparsity of solutions to $\min_{\vz: \mA \vz = \vy} \| \vz \|_p^p$ for an underdetermined linear system $\mA \vz = \vy$, with modifications to account for the linear inequality constraints in \eqref{opt:min_p_NN_reparam}. The fundamental observation is that the linear constraints in \eqref{opt:min_p_NN_reparam} determine a polytope, and the map $\vz \mapsto \| \vz \|_p^p$ is strictly concave on each individual orthant and invariant to absolute values of vector elements. By projecting the constraint set of \eqref{opt:min_p_NN_reparam} into the nonnegative orthant, the problem turns into a minimization of a continuous, strictly concave function over a polytope. By the Bauer maximum principle, any solution to this problem occurs at one of the finitely many vertices of that polytope, and by appropriately normalizing the vertices of this polytope, we are able to demonstrate the desired result.

Although \cref{th:multivar} applies to any input dimension, thus recovering part of the result of \cref{th:main}, our multivariate analysis does not immediately recover the univariate results on functional structure or uniqueness of solutions; nor does it demonstrate that solutions for $0 < p < 1$ are always solutions for $p = 1$, as was shown in the univariate case (\cref{corr1}). Thus, although \cref{th:multivar} guarantees exact sparsest recovery for sufficiently small $p$ in arbitrary input dimensions, the multivariate problem leaves many interesting open questions, which we save for future work.

\section{Experiments}
We perform several simple experiments on synthetic data which suggest that our proposed $\ell^p$ path norm lends itself to practical application, recovering far sparser solutions more quickly than unregularized or weight decay-regularized gradient-based training. To implement our regularizer, we use a proximal gradient algorithm based on the iteratively reweighted $\ell^1$ method of \cite{candes2008enhancing,figueiredo2007majorization}, the details of which are summarized in \cref{appendix:rw_l1}. \cref{fig:sparsity_over_time_main_text} shows the sparsity over time of networks trained with our reweighted $\ell^1$ algorithm for three different values of $p \in \{0.4, 0.7, 1 \}$, as well as with unregularized Adam and AdamW weight decay, on two different synthetic datasets. For all values of $p$, the $\ell^p$-regularized networks are much sparser much earlier in training than the unregularized or weight decay regularized networks, with the $p = 0.4$ networks being the sparsest. For the univariate synthetic dataset, the $p = 0.4$ regularized network recovers the true sparsest solution, and for the multivariate synthetic dataset, all $\ell^p$ regularized networks recover solutions which obey the sparsity upper bound guaranteed by \cref{prop:width_invariance_multivar}. For further details, results, and discussion, see \cref{appendix:experimental_setup_results}. Code for these experiments is available at \url{https://github.com/julianakhleh/sparse_nns_lp}.

\begin{figure}
    \centering
    \begin{subfigure}{0.5\textwidth}
          \centering
          \includegraphics[width=\linewidth]{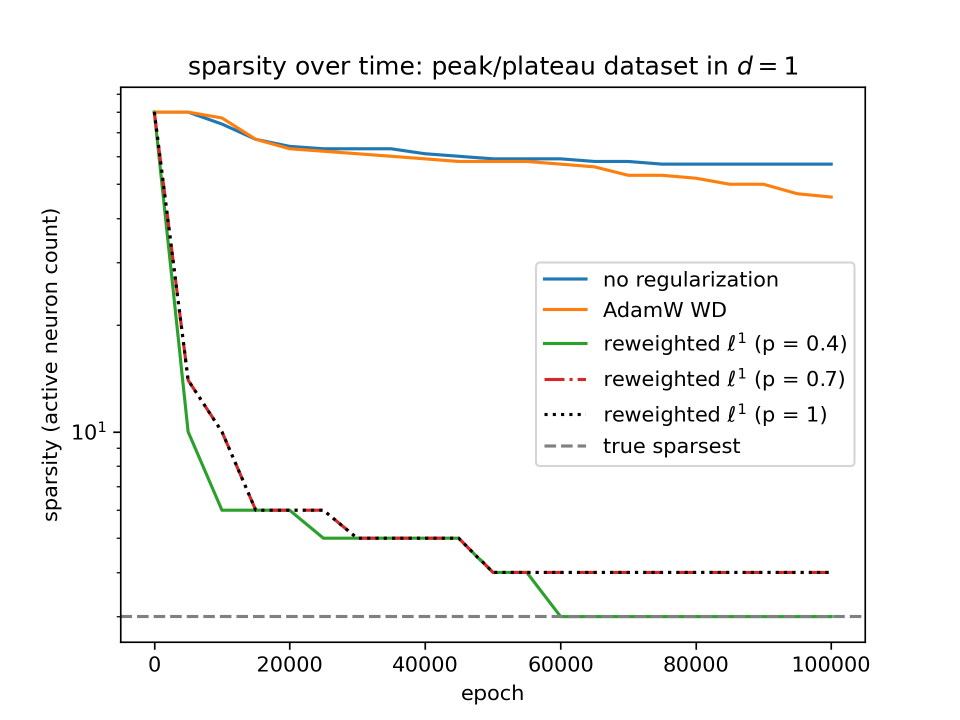}
          \label{fig:sparsity_over_time_univariate_main_text}
    \end{subfigure}%
    \begin{subfigure}{0.5\textwidth}
          \centering
          \includegraphics[width=\linewidth]{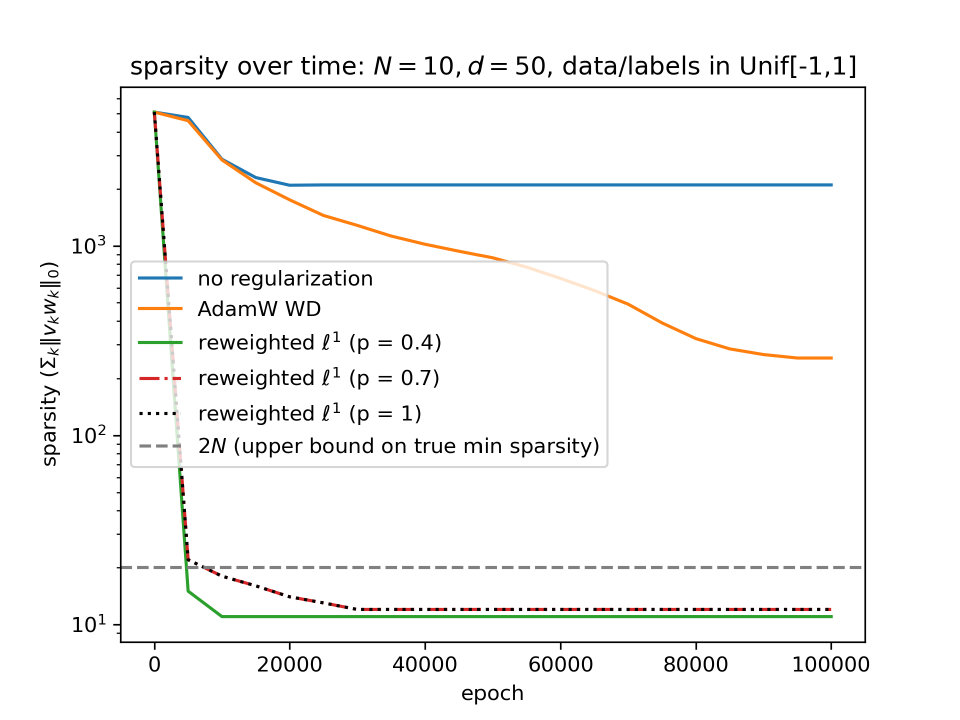}
          \label{fig:sparsity_over_time_high_d_main_text}
    \end{subfigure}
    \caption{Sparsity over time of networks trained to interpolation with a reweighted $\ell^1$ algorithm (see \cref{appendix:rw_l1}) for $\ell^p$ path norm regularization, $p \in \{0.4, 0.7, 1 \}$, and of unregularized and weight decay-regularized networks. Results on the left are for a synthetic univariate ``peak/plateau'' dataset, and results on the right are for a high-dimensional set of random data and labels. The gray dashed lines reflect the true minimal sparsity (in the univariate case, left) and the upper bound on the minimal sparsity guaranteed by \cref{prop:width_invariance_multivar} in the multivariate case (right). For further details, results, and discussion, see \cref{appendix:experimental_setup_results}.}
    \label{fig:sparsity_over_time_main_text}
\end{figure}

\section{Conclusion and Discussion} \label{sec:conclusion_discussion}
We have introduced a smooth, $\ell^p$ path norm ($0 < p < 1$) regularization framework whose global minimizers provably coincide with the sparsest ReLU network interpolants for sufficiently small $p$, thus recasting the combinatorial $\ell^0$ minimization problem as a differentiable objective compatible with gradient descent. In the univariate case, we showed minimum $\ell^p$ path norm interpolants are  unique for almost every $0 < p < 1$; never require more than $N-2$ neurons; and are also $\ell^1$ minimizers, yielding explicit data-dependent parameter and Lipschitz bounds. In arbitrary dimensions, we demonstrate a similar $\ell^p$-$\ell^0$ equivalence for sufficiently small $p$. Our proposed regularization objective  offers a principled, gradient-based alternative to heuristic pruning methods for training truly sparse neural networks.

While we demonstrate the existence of $p$ small enough for $\ell^p$/$\ell^0$ minimization equivalence, our proofs do not yield an efficient way to compute the ``critical threshold'' $p^*$, although they do demonstrate that estimating this $p^*$ is in theory possible by enumerating an exponential number of vertices of a data-dependent polytope. Whether or not $p^*$ can be computed or estimated \textit{efficiently} is an open question of interest for future work. Other possible directions of interest are to extend our results here to multi-output and deep architectures and to other notions of sparsity (such as sparsity over entire neurons vs. parameters in the multi-dimensional case).

\bibliography{refs}
\bibliographystyle{abbrvnat}

\newpage
\begin{appendices}
\section{Proofs of main results}
\subsection{Univariate results}
\subsubsection{Proof of \cref{prop:existence_and_variational}} \label{appendix:proof_prop_existence_and_variational}
\begin{proof}
    By homogeneity of the ReLU---meaning that $(\alpha x)_+ = \alpha (x)_+$ for any $\alpha > 0$---any ReLU neural network of the form \eqref{eq:nn} can have its parameters rescaled as $v_k \mapsto |w_k| v_k$, $(w_k, b_k) \mapsto |w_k|^{-1} (w_k,b_k)$ without changing the network's represented function or its $\ell^p$ path norm. Therefore, any $f \in S_p^*$ can be expressed as a neural network of the form \eqref{eq:nn} with $|w_k| = 1$ for all $k = 1, \dots, K$. Additionally, any $f \in S_p^*$ can be expressed as a network where no two neurons ``activate'' at the same location, i.e., $b_k/w_k \neq b_{k'}/w_{k'}$ whenever $k \neq k'$. To see this, consider a neural network $f_\vtheta$ with unit-norm input weights which contains two distinct neurons $k, k'$ with $b_k/w_k = b_{k'}/w_{k'}$. The sum of these neurons can be rewritten as
    \begin{align}
        v_k(w_k x + b_k)_+ + v_{k'}(w_{k'} x + b_{k'})_+ = (v_k + v_{k'})(w_kx + b_k)_+
    \end{align}
    if $w_k = w_{k'}$, or as
    \begin{align}
        v_k(w_k x + b_k)_+ + v_{k'}(w_{k'} x + b_{k'})_+ = (v_k + v_{k'})(w_kx+b_k)_+ - v_{k'} (w_k x + b_k)
    \end{align}
    if $w_k = - w_{k'}$. (The latter uses the identity $x = (x)_+ - (-x)_+$.) In either case, we see that the original two neurons $k, k'$ can be replaced with a single neuron and, in the latter case, an additive affine term. Because the affine term does not contribute to $\ell^p$ path norm, and because $|v_k + v_{k'}|^p \leq |v_k|^p + |v_{k'}|^p$ for $p \in (0,1]$, the resulting network represents the same function as the original one with no greater regularization cost.
    
    Furthermore, any neural network of the form \eqref{eq:nn} with unit-norm input weights and $K$ active neurons, where no two active neurons activate at the same location, is a CPWL function with $K$ knots, where knot $k$ is located at $-b_k/w_k$, and the slope change of the function at knot $k$ is $v_k$. Conversely, any $\bR \to \bR$ CPWL function $f$ with $K$ knots at locations $u_1 <  \dots < u_K$ and corresponding slope changes $v_1, \dots, v_K$ can be expressed as
    \begin{align}
        f(x) = f(u_0) + f'(u_0) (x-u_0) + \sum_{k=1}^K v_k (x - u_k)_+
    \end{align}
    for some arbitrary point $u_0 < u_1$. Any such $f$ has $D^2 f = \sum_{k=1}^K v_k \delta_{u_k}$, so that $V_p(f) = \sum_{k=1}^K |v_k|^p$, and $V_0(f) = \sum_{k=1}^K \mathbbm{1}_{v_k \neq 0} = K$.
    
    These facts are sufficient to establish the equivalence of problems \eqref{opt:min_p_NN} and \eqref{opt:min_p_f}. Indeed, let $\overline{S}_{\vtheta, p}^*$ denote the set of optimal parameters for a modified version of problem \eqref{opt:min_p_NN} which imposes the additional constraints that each $|w_k| = 1$ and that $b_k/w_k \neq b_{k'}/w_{k'}$ whenever $k \neq k'$. For some $\vtheta^* \in \overline{S}_{\vtheta, p}^*$, let $C^*$ denote its $\ell^p$ path norm. We have shown that $S_p^*$ can be equivalently expressed as
    \begin{align}
        S_p^* &= \{ f: \R \to \R \ \rvert \ f = f_\vtheta, \ \vtheta \in \overline{S}_{\vtheta,p}^* \}\\
        &= \{ f: \R \to \R \ \rvert \  f \mbox{ is CPWL with $\leq K$ knots},\ V_p(f) = C^*, \,   f(x_i) = y_i, \, i = 1, \dots, N \}
    \end{align}
    which is exactly the set of minimizers of \eqref{opt:min_p_f}. Non-emptiness of $S_{0, \vtheta}^*$ (and thus of $S_0^*$) follows from non-emptiness of the feasible set $\Theta$ of \eqref{opt:min_0_NN} when $K \geq N$, and the fact that the objective values of members of the feasible set lie in $\{1, \dots, K\}$, on which a minimum is achieved.
\end{proof}

\subsubsection{Auxiliary lemmas: local behavior of $f$ around same/opposite sign slope changes} \label{appendix:aux_lemmas_univariate}
Our proof of \cref{th:geom_char} relies strongly on the following two auxiliary lemmas, which describe the local behavior of any $f \in S_p^*$ for $0 \leq p < 1$ between consecutive data points. Here we denote the incoming and outgoing slopes of any interpolant $f$ at a data point $x_i$ as $\si(f, x_i)$ and $\so(f, x_i)$, respectively (sometimes dropping the explicit reference to $f$ if it is clear from context). First, we show in \cref{lemma:same_sign} that for any optimal network function $f \in S_p^*$, $0 \leq p < 1$, if the signs of $s_i - \si(f, x_i)$ and $\so(f, x_{i+1}) - s_i$ agree, then $f$ connects $(x_i, y_i)$ and $(x_{i+1}, y_{i+1})$ in a single ``peak'' (see \cref{fig:same_sign}).
\begin{figure}
    \begin{subfigure}{.5\textwidth}
          \centering
          \includegraphics[width=\linewidth]{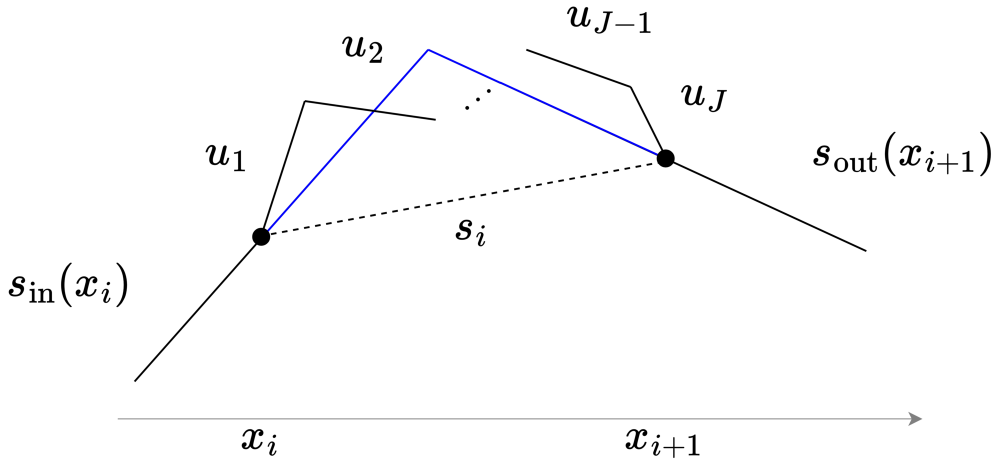}
          \caption{Illustration of \cref{lemma:same_sign}.}
          \label{fig:same_sign}
    \end{subfigure}
    \begin{subfigure}{.5\textwidth}
          \centering
          \includegraphics[width=0.85\linewidth]{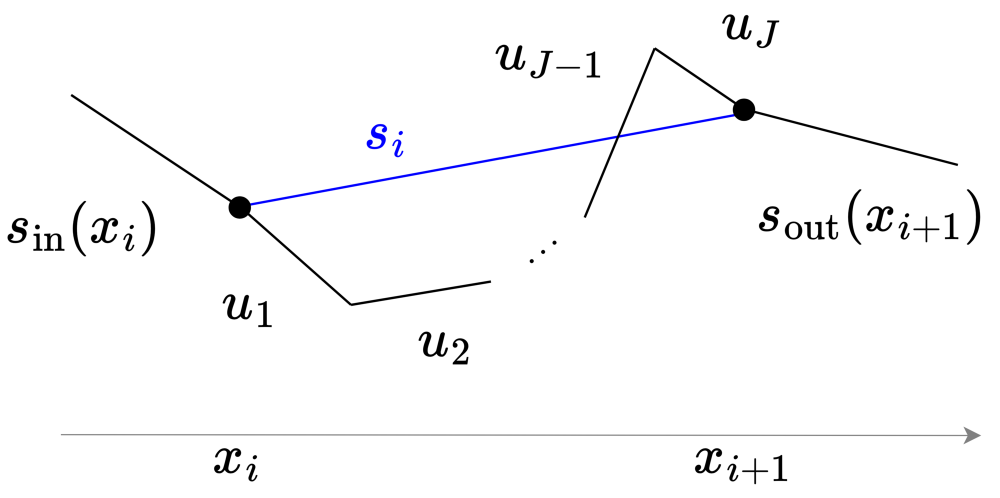}
          \caption{Illustration of \cref{lemma:opp_sign}.}
          \label{fig:opp_sign}
    \end{subfigure}
    \caption{Left: Illustration of the case $\sgn \left( s_i - \si(f, x_i) \right) = \sgn \left(\so(f, x_{i+1}) -s_i\right)$ addressed in \cref{lemma:same_sign}. Right: illustration of the case $\sgn \left( s_i - \si(f, x_i) \right) \neq \sgn \left(\so(f, x_{i+1}) -s_i\right)$ addressed in \cref{lemma:opp_sign}. In both cases, the functions in black have strictly greater $V_p$ for $0 \leq p < 1$ than the functions in blue.}
    \label{fig:same_sign_opp_sign}
\end{figure}

\begin{lemma}[Behavior of $f \in S_p^*$ around same-sign slope changes] \label{lemma:same_sign}
    For $0 \leq p < 1$, suppose that $f \in S_p^*$ has $\sgn \left( s_i - \si(f, x_i) \right) = \sgn \left(\so(f, x_{i+1}) -s_i\right)$ at consecutive data points $x_i$, $x_{i+1}$. If both signs are zero, then $f$ is linear on the interval $I := [x_i - \delta, x_{i+1} + \delta]$ surrounding $[x_i, x_{i+1}]$, for small $\delta > 0$. Otherwise, $f$ has a single knot on $I$, between $x_i$ and $x_{i+1}$. (See \cref{fig:same_sign}.)
\end{lemma}
\begin{proof}
    If both signs are zero, then $f$ must be linear on $I$, since anything else would have nonzero $V_p(f \rvert_I )$ for $0 \leq p < 1$. If both signs are nonzero, observe that
    $$ |\so(f, x_{i+1}) - \si(f, x_i)|^p < |\so(f, x_{i+1}) - u_J|^p + |u_J - u_{J-1}|^p + \dots + |u_2 - u_1|^p + |u_1 - \si(f, x_i)|^p$$
    for any $u_1, \dots, u_J$ which are all distinct from each other and from $\si(f, x_i)$ and $\so(f, x_{i+1})$. This is a simple consequence of the inequality $|a+b|^p \leq |a|^p + |b|^p$, which holds for any $a,b \in \R$ and any $0 < p < 1$ and is strict unless $a = 0$ or $b = 0$. Since any interpolant with more than one knot on $I$ has one or more intermediate slopes $u_1, \dots, u_J$ between $x_i$ and $x_{i+1}$, the result follows.
\end{proof}

Next, \cref{lemma:opp_sign} says that if the signs of $s_i- \si(f, x_i)$ and $\so(f, x_{i+1}) - s_i$ of an optimal $f \in S_p^*$, $0 < p < 1$ disagree, then $f$ is linear between $x_i$ and $x_{i+1}$.
\begin{lemma}[Behavior of $f \in S_p^*$ around opposite-sign slope changes]  \label{lemma:opp_sign}
    For $0 \leq p < 1$, suppose that $f \in S_p^*$ has $\sgn \left( s_i - \si(f, x_i) \right) \neq \sgn \left(\so(f, x_{i+1}) -s_i\right)$ at consecutive data points $x_i$, $x_{i+1}$. If $0 < p < 1$, then $f$ is linear between $x_i$ and $x_{i+1}$. If $p = 0$, then either $f$ is linear between $x_i$ and $x_{i+1}$, or it agrees outside of $[x_i, x_{i+1}]$ with some $g \in S_0^*$ which is linear between $x_i$ and $x_{i+1}$. (See illustration in \cref{fig:opp_sign}.)
\end{lemma}  

\begin{proof}
    \begin{figure}
        \centering
        \includegraphics[width=0.5\linewidth]{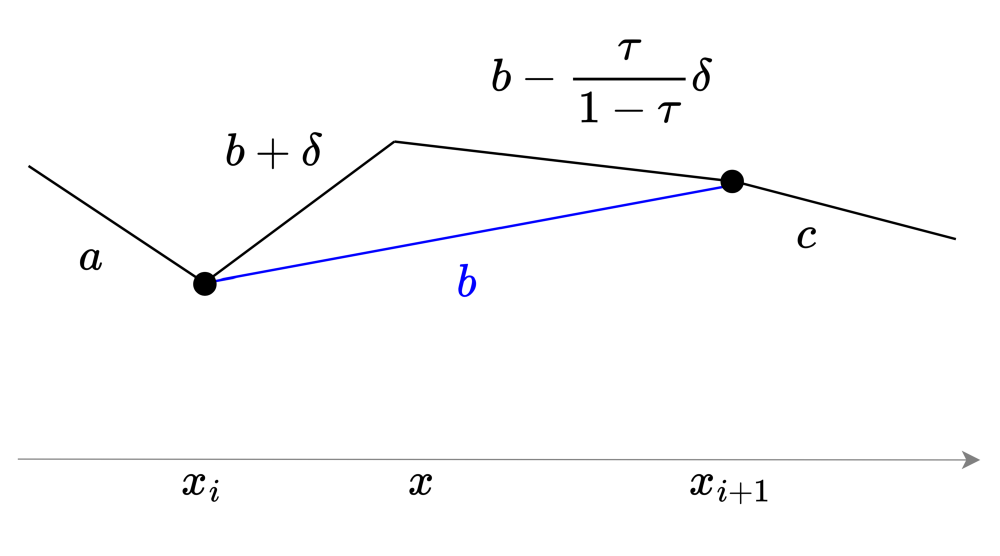}
        \caption{Base case of \cref{lemma:opp_sign}, where we consider the possibility that $f \in S_p^*$ for some $0 \leq p < 1$ has a single knot at some $x \in (x_i, x_{i+1})$ where $\sgn(a-b) \neq \sgn(b-c)$. Here $\tau := \frac{x - x_i}{x_{i+1}-x_i}$.}
        \label{fig:opp_sign_base_case}
    \end{figure}
    First consider the base case illustrated in \cref{fig:opp_sign_base_case}, where we suppose that $f \in S_p^*$ for some $0 \leq p < 1$ has a single knot at some $x \in (x_i, x_{i+1})$. To simplify the notation, we denote $a := \si(f, x_i)$, $b := s_i$, $c := \so(f, x_{i+1})$ and $\tau := \frac{x - x_i}{x_{i+1}-x_i}$ and assume that $\sgn(a-b) \neq \sgn(b-c)$. The intermediate slopes $u_1$ and $u_2$ can be parameterized as $u_1 = b+\delta$ and $u_2 = b - \frac{\tau}{1-\tau} \delta$ for some $\delta \in \R$. Consider the cost $V_p(f \big\rvert_I)$ of $f$ on the interval $I := (x_i - \epsilon, x_{i+1} + \epsilon)$ (for some arbitrary $\epsilon > 0$) as a function $C(\delta)$ of the parameter $\delta$. If $p = 0$, then clearly $C(0) = 2 \leq C(\delta) \in \{2, 3\}$ for $\delta \neq 0$. This shows that the function $g$ whose slope is $b$ on $[x_i, x_{i+1}]$ has no greater cost than $f$, and thus $g \in S_0^*$. In the case $0 < p < 1$, we have
    \begin{align}
        C(\delta) := |\delta + b - a|^p + \frac{1}{(1-\tau)^p} |\delta|^p + \big| c - b + \frac{\tau}{1-\tau} \delta \big|^p
    \end{align}
    and we will show that $C(0) < C(\delta)$ for $\delta \neq 0$, contradicting the assumption that $f \in S_p^*$. 
    
    Note that $C$ is coercive and continuous on $\delta \in \R$, so it attains a minimizer (this follows from the Weierstrass extreme value theorem as applied to the compact sublevel sets of $C$). By Fermat's theorem, any minimizer of $C$ must occur at critical points, i.e., points where the derivative  $C'$ is zero or undefined. The three points where $C'$ is undefined are $\delta_1 = a-b$, $\delta_2 = 0$, and $\delta_3 = \frac{1-\tau}{\tau}(b-c)$. Assuming without loss of generality that $\delta_1 < \delta_2 < \delta_3$, note that $C$ is strictly concave on the intervals $(-\infty, \delta_1)$, $(\delta_1, \delta_2)$, $(\delta_2, \delta_3)$, and $(\delta_3, \infty)$. This is because compositions of concave and affine functions are concave, and the function $x \mapsto |x|^p$ for $p \in (0,1]$ is concave on any subinterval of $\R$ over which $x$ does not change sign. Therefore, any point at which $C' = 0$ will be a local maximum rather than a minimum, and hence any minimum of $C$ can only occur at the critical points $\delta_1, \delta_2, \delta_3$. We have
    \begin{align}
        C(\delta_1) &= \frac{1}{(1-\tau)^p}|a-b|^p + \big| c + \frac{\tau}{1-\tau} a - \frac{1}{1-\tau} b \big|^p \\
        C(\delta_2) &= |b-a|^p + |c-b|^p\\
        C(\delta_3) &= \big| \frac{1}{\tau} b - \frac{1-\tau}{\tau} c -a \big|^p + \frac{1}{\tau^p} |b-c|^p
    \end{align}
    Now, for the variable $t \in [0,1)$, define
    \begin{align}
        h_1(t) := \frac{1}{(1-t)^p} |a-b|^p + \big|c + \frac{t}{1-t} a - \frac{1}{1-t} b \big|^p
    \end{align}
    and observe that $h_1(0) = C(\delta_2)$ and $h_1(\tau) = C(\delta_1)$. Its derivative is
    \begin{align}
        h_1'(t) &= \frac{p}{(1-t)^{p+1}} |a-b|^p + p \big|c + \frac{t}{1-t} a-\frac{1}{1-t}b \big|^{p-1} \textrm{sgn}\left(c + \frac{t}{1-t} a-\frac{1}{1-t}b \right) \frac{a-b }{(1-t)^2} \\
        &= \frac{p}{(1-t)^{p+1}} |a-b|^p + p \big|c + \frac{t}{1-t} a-\frac{1}{1-t}b \big|^{p-1} \textrm{sgn} \left(\frac{(1-t)(c-b)+t(a-b)}{1-t} \right) \frac{a-b}{(1-t)^2} 
    \end{align}
    Assuming that $\sgn(a-b) \neq \sgn(b-c)$ with  $a \neq b$ (and thus $\delta_1 \neq \delta_2$), we see that $h_1'(t) > 0$ for all $t \in [0,1)$. This is because the term inside the $\sgn$ above is positive if $a > b$ (so that $b \leq c$) and negative if $a < b$ (so that $b \geq c$). This shows that $h_1(0) = C(\delta_2) < h_1(\tau) = C(\delta_1)$. Similarly, define 
    \begin{align}
        h_2(t) := \big| \frac{1}{t} b - \frac{1-t}{t} c -a \big|^p + \frac{1}{t^p} |b-c|^p
    \end{align}
    for $t \in (0,1]$, so that $h_2(\tau) = C(\delta_3)$ and $h_2(1) = C(\delta_2)$. Its derivative is
    \begin{align}
        h_2'(t) &= p \big| \frac{1}{t} b - \frac{1-t}{t} c -a \big|^{p-1} \sgn \left( \frac{1}{t} b - \frac{1-t}{t} c -a \right) \frac{c-b}{t^2} - \frac{p}{t^{p+1}} |b-c|^p \\
        &= p \big| \frac{1}{t} b - \frac{1-t}{t} c -a \big|^{p-1} \sgn \left( \frac{t(b-a) + (1-t)(b-c)}{t} \right) \frac{c-b}{t^2} - \frac{p}{t^{p+1}} |b-c|^p
    \end{align}
    Assuming that $\sgn(a-b) \neq \sgn(b-c)$ with $b \neq c$ (and thus $\delta_2 \neq \delta_3$), we see that $h_2'(t) > 0$ for all $t \in (0,1]$. This is because the term inside the $\sgn$ above is positive if $b > c$ (so that $a \leq b$) and negative if $b < c$ (so that $a \geq b$). This shows that $h_2(\tau) = C(\delta_3) < h_2(1) = C(\delta_2)$. Therefore, $C(0) < C(\delta)$ for $\delta \neq 0$, as desired.

    Next, consider the general case, where we assume by contradiction that $f \in S_p^*$ for $0 \leq p < 1$ may have multiple knots inside $(x_i, x_{i+1})$. As before, in the case $p = 0$, $f$ can't have fewer knots than the function $g$ whose slope is $b$ on $[x_i, x_{i+1}]$; the only way for $f$ to be in $S_0^*$ is if it has a single knot inside $(x_{i}, x_{i+1})$ and a single knot at either $x_i$ or $x_{i+1}$, in which case we also have $g \in S_0^*$. In the case $0 < p < 1$, let $u_1, \dots, u_J$ denote the slopes of $f$ on $[x_i, x_{i+1}]$. If the line segments with slopes $u_1$ and $u_J$ lie on the same side of the line segment with slope $s_i$, then we can apply the argument in the proof of \cref{lemma:same_sign} to remove the segments with slopes $u_2, \dots, u_{J-1}$ and connect the segments with $u_1$ and $u_J$ in a single knot inside $(x_i, x_{i+1})$; this strictly reduces $V_p(f)$, contradicting $f \in S_p^*$. (See \cref{fig:opp_sign_same_side}.) If the line segments with slopes $u_1$ and $u_J$ lie on opposite sides of the line segment with slope $s_i$, then either one of the intermediate segments, whose slope we call $u_{j_0}$, crosses the segment with slope $s_i$, or else one of the intermediate segments (again call its slope $u_{j_0}$) lies on one side of $s_i$, and $u_{j_0 + 1}$ lies on the other side. In either case, the segments $u_1$ and $u_{j_0}$ can be connected and the segments between them removed, as can the segments $u_{j_0}$ (or $u_{j_0+1}$) and $u_J$. (See \cref{fig:opp_sign_opp_side}.) Again, by the logic in the proof of \cref{lemma:same_sign}, this strictly reduces $V_p(f)$, contradicting $f \in S_p^*$.  If $f$ is already of the form in \cref{fig:opp_sign_opp_s_i_red}, with only two knots inside $(x_i, x_{i+1})$ on opposite sides of the line $s_i$, then the second knot can be removed by directly connecting $u_1$ and $\so(f, x_{i+1})$ (see \cref{fig:opp_sign_opp_s_i_red_2}). By the same logic, this strictly reduces $V_p(f)$, contradicting $f \in S_p^*$.
\end{proof}
\begin{figure}
    \centering
    \begin{subfigure}{0.5\textwidth}
          \centering
          \includegraphics[width=0.9\linewidth]{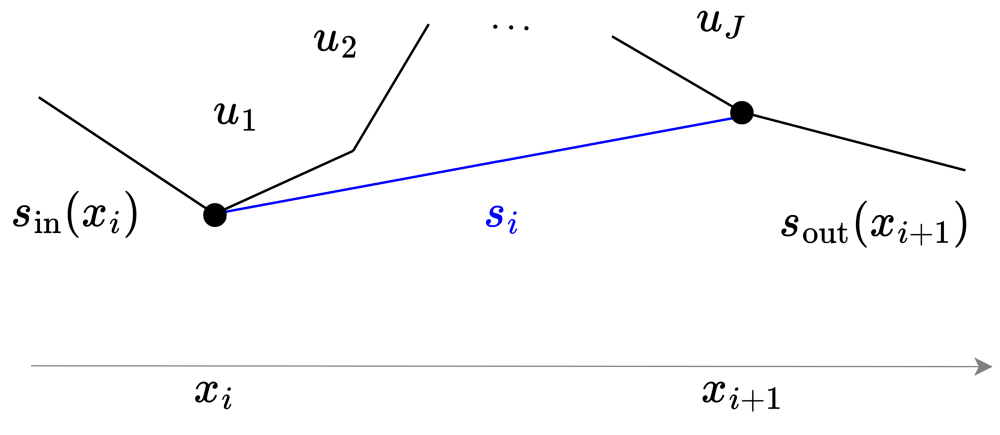}
          \caption{$u_1$, $u_J$ are on the same side of $s_i$.}
          \label{fig:opp_sign_same_side_s_i}
    \end{subfigure}%
    \begin{subfigure}{0.5\textwidth}
          \centering
          \includegraphics[width=0.9\linewidth]{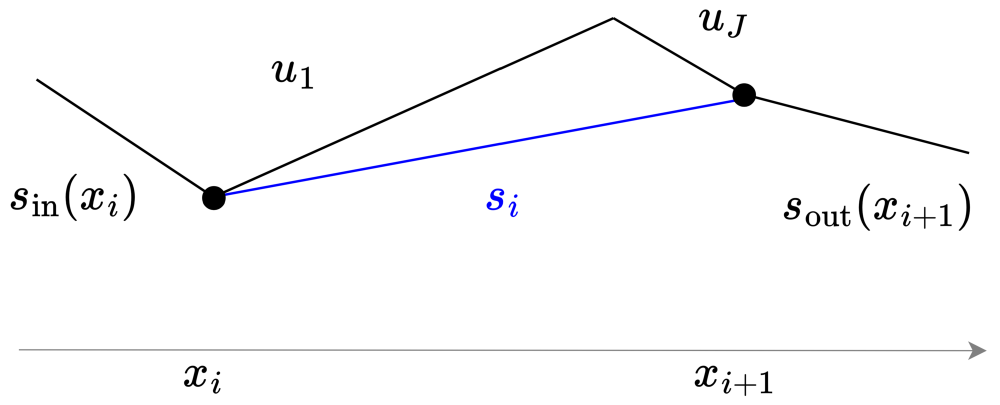}
          \caption{$u_1$ and $u_J$ can be connected, reducing $V_p(f)$.}
          \label{fig:opp_sign_same_side_s_i_red}
    \end{subfigure}
    \caption{General case of \cref{lemma:opp_sign}, where the outgoing line segment at $x_i$ and the incoming line segment at $x_{i+1}$ both lie on the same side of the straight line between $(x_i, y_i)$ and $(x_{i+1}, y_{i+1})$. We can apply the argument in the proof of \cref{lemma:same_sign} to connect these two segments in a single knot inside $(x_i, x_{i+1})$ and strictly reduce $V_p(f)$.}
    \label{fig:opp_sign_same_side}
\end{figure}
\begin{figure}
    \centering
    \begin{subfigure}{0.5\textwidth}
          \centering
          \includegraphics[width=0.9\linewidth]{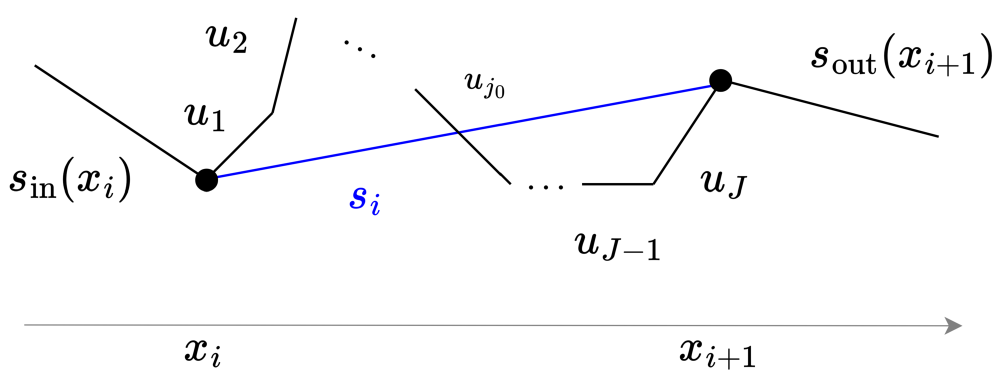}
          \caption{$u_1$, $u_J$ are on opposite sides of $s_i$.}
          \label{fig:opp_sign_opp_s_i}
    \end{subfigure}%
    \begin{subfigure}{0.5\textwidth}
          \centering
          \includegraphics[width=0.9\linewidth]{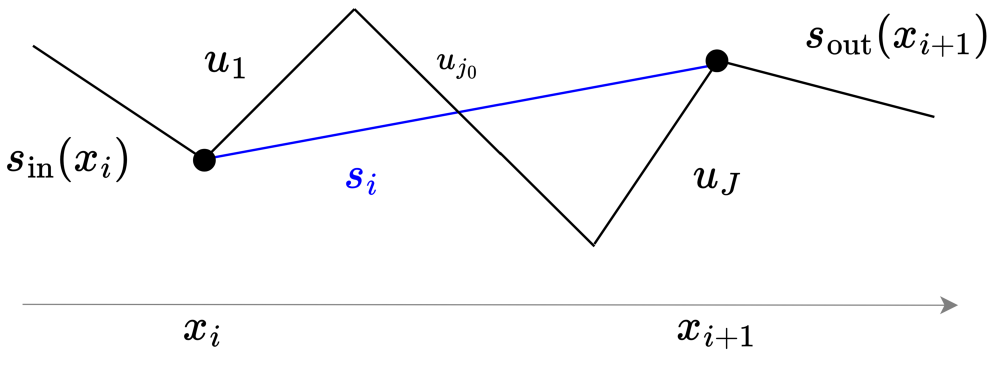}
          \caption{$u_1$, $u_{j_0}$ and $u_{j_0}, u_J$ can be connected, reducing $V_p(f)$.}
          \label{fig:opp_sign_opp_s_i_red}
    \end{subfigure}
    \par\bigskip
    \begin{subfigure}{0.5\textwidth}
          \centering
          \includegraphics[width=0.9\linewidth]{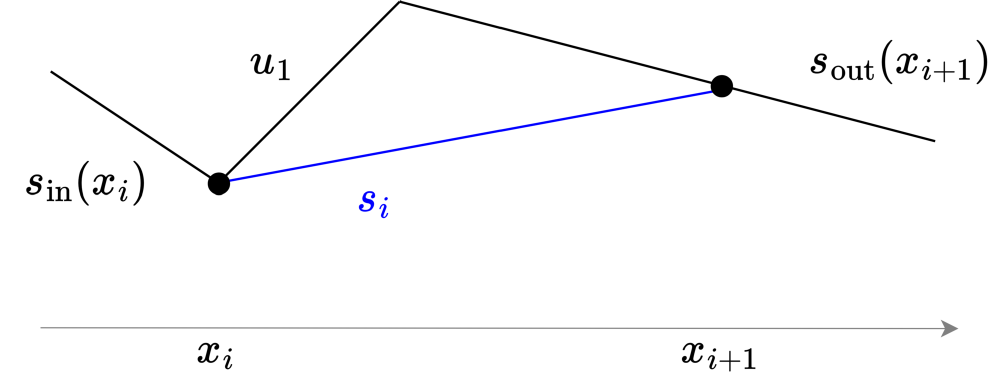}
          \caption{$u_1$ and $\so(f, x_{i+1})$ can be connected, reducing $V_p(f)$.}
          \label{fig:opp_sign_opp_s_i_red_2}
    \end{subfigure}
    \caption{General case of \cref{lemma:opp_sign}, where the outgoing line segment at $x_i$ and the incoming line segment at $x_{i+1}$ lie on opposite sides of the straight line between $(x_i, y_i)$ and $(x_{i+1}, y_{i+1})$. We can apply the argument in the proof of \cref{lemma:same_sign} to connect the segments $u_1$ and $u_{j_0}$ and $u_{j_0}$ and $u_J$, resulting in a function with two knots inside $(x_i, x_{i+1})$ and strictly reducing $V_p(f)$. By the same argument, we can further reduce $V_p(f)$ by connecting $u_1$ and $\so(f, x_{i+1})$, resulting in a single knot inside $(x_i, x_{i+1})$.}
    \label{fig:opp_sign_opp_side}
\end{figure}

\subsubsection{Proof of \cref{th:geom_char}} \label{appendix:proof_geom_char}
\begin{proof}
    We first use  \cref{th:geom_char} and \cref{lemma:same_sign,lemma:opp_sign} to show that any $f \in S_p^*$ for $0 < p < 1$ must obey the description in \cref{th:geom_char}, and that there is always some $f \in S_0^*$ which fits this description. Using this result, we argue non-emptiness of $S_p^*$. We break the proof into the following sections.
    \paragraph{Linearity before $x_2$ and after $x_{N-1}$.} We will prove the statement for $(-\infty, x_2]$; the proof for $[x_{N-1}, \infty)$ is analogous. No $f \in S_p^*$ for $0 \leq p \leq 1$ can have a knot at or before $x_1$ as this would strictly increase the cost $V_p(f)$ without affecting the ability of $f$ to interpolate the data points. In the case $0 < p < 1$, assume by contradiction that some $f \in S_p^*$ has a knot at some $x \in (x_1, x_2)$. By \cref{lemma:opp_sign}, it must be the case that $\sgn(s_1 - \si(f, x_1)) = \sgn(\so(f, x_2) - s_1)$, and by \cref{lemma:same_sign}, this knot is the only one inside $(x_1, x_2)$, with $\si(f, x_1) = \so(f, x_1)$ and $\si(f, x_2) = \so(f, x_2)$. (See \cref{fig:first_orig}.) Assuming without loss of generality that $\sgn(s_1 - \si(f, x_1)) = \sgn(\so(f, x_2) - s_1) = -1$, we have $\si(f, x_1) > s_1 > \so(f, x_2)$, and therefore $|\so(f, x_2) - \si(f, x_1)| > |\so(f, x_2) - s_1|$. But this shows that $V_p(f) > V_p (g)$, where $g = \ell_1$ on $(-\infty, x_2]$ and is otherwise identical to $f$. (See \cref{fig:first_red}.) This contradicts $f \in S_p^*$.

    In the case $p = 0$, fix some $f \in S_0^*$. As argued above, $f$ has no knots on $(-\infty, x_1]$. If $\sgn(s_1 - \si(f, x_1)) \neq \sgn(\so(f, x_2) - s_1)$, then by \cref{lemma:opp_sign}, either $f = \ell_1$ on $[x_1, x_2]$ (hence it also must agree with $\ell_1$ on $(-\infty, x_1]$), or there is some $g \in S_0^*$ which agrees with $\ell_1$ on $[x_1, x_2]$ (hence also on $(-\infty, x_1]$, since $g$ must also not have any knots on $(-\infty, x_1]$). If $\sgn(s_1 - \si(f, x_1)) =  \sgn(\so(f, x_2) - s_1) = 0$, then by \cref{lemma:same_sign}, $f = \ell_1$ on $[x_1, x_2]$ and thus also on $(-\infty, x_1]$. If $\sgn(s_1 - \si(f, x_1)) =  \sgn(\so(f, x_2) - s_1)$ are both nonzero, then by \cref{lemma:same_sign}, $f$ has a single knot inside $(x_1, x_2)$ with $\si(f, x_1) = \so(f, x_1)$ and $\si(f, x_2) = \so(f, x_2)$, as in \cref{fig:first_orig}. Then function depicted in \cref{fig:first_orig_red}, which agrees with $\ell_1$ on $(-\infty, x_2]$ and with $f$ on $[x_2, \infty)$, has the same number of knots as $f$, so $g \in S_0^*$.
    \begin{figure}
    \centering
    \begin{subfigure}{0.5\textwidth}
          \centering
          \includegraphics[width=0.9\linewidth]{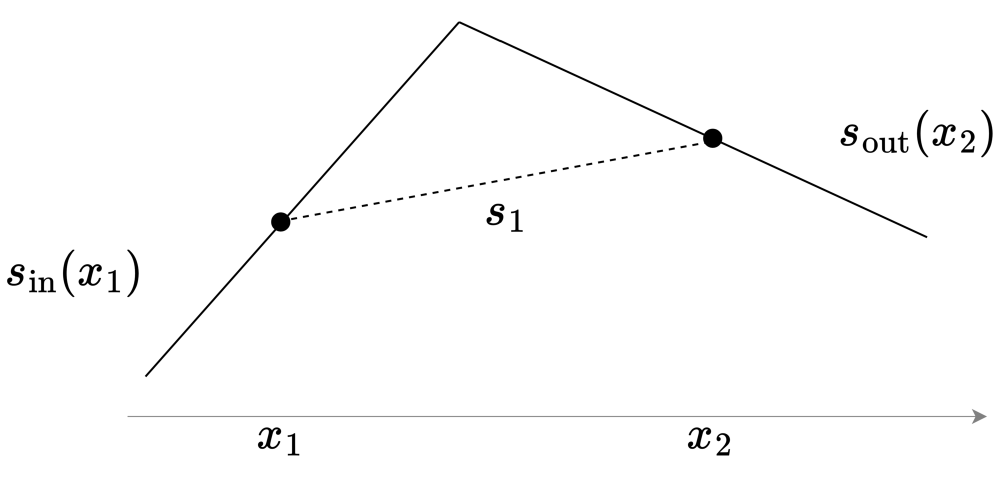}
          \caption{A function with a knot inside $(x_1, x_2)$.}
          \label{fig:first_orig}
    \end{subfigure}%
    \begin{subfigure}{0.5\textwidth}
          \centering
          \includegraphics[width=0.8\linewidth]{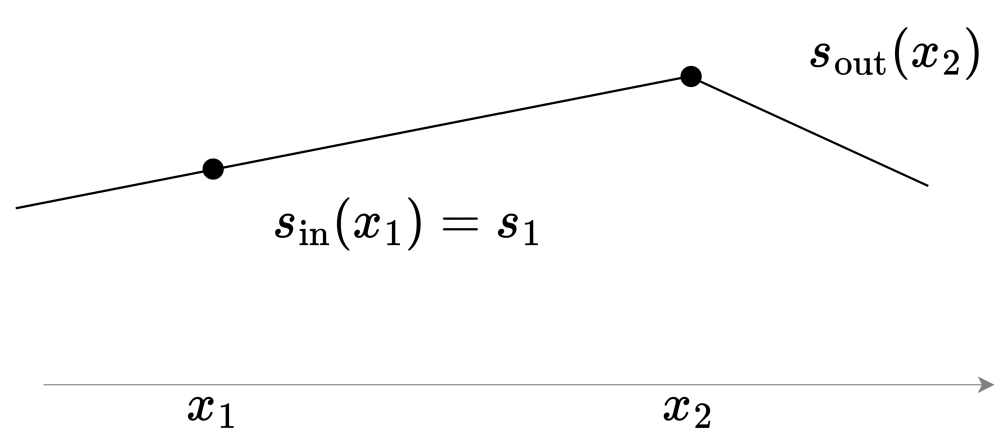}
          \caption{A function which agrees with $\ell_1$ on $(-\infty, x_2]$.}
          \label{fig:first_red}
    \end{subfigure}
    \caption{Behavior of $f \in S_p^*$ before $x_2$ and after $x_N$. A knot inside $(x_1, x_2)$ can be moved to $x_2$, maintaining the same outgoing slope at $x_2$, which strictly decreases the magnitude of the slope change at the knot.}
    \label{fig:first_orig_red}
\end{figure}
    \paragraph{Linearity between data points of opposite curvature.} For $0 < p < 1$, assume by contradiction that some $f \in S_p^*$ does not agree with $\ell_i$ on an interval $[x_i, x_{i+1}]$ where $\sgn(s_i - s_{i-1}) \neq \sgn(s_{i+1} - s_i)$. By \cref{lemma:same_sign,lemma:opp_sign}, it must be the case that $\sgn(s_i - \si(f, x_i)) = \sgn(\so(f, x_{i+1}) - s_i)$ are both nonzero, and that $\si(f, x_i) = \so(f, x_i)$ and $\si(f, x_{i+1}) = \so(f, x_{i+1})$ and $f$ has a single knot inside $(x_i, x_{i+1})$ where the incoming line at $x_i$ and the outgoing line at $x_{i+1}$ meet. It must be the case that $\sgn(s_i - s_{i-1}) \neq \sgn(s_i - \si(f, x_i))$ and/or that $\sgn(s_{i+1} - s_i) \neq \sgn(\so(f, x_{i+1}) - s_i)$. Assume without loss of generality that $\sgn(s_{i+1} - s_i) \neq \sgn(\so(f, x_{i+1}) - s_i) = 1$, so that $s_{i+1} \leq s_i < \so(f, x_{i+1}) = \si(f, x_{i+1})$. Then clearly $s_{i+1} \neq  \so(f, x_{i+1})$ (in other words, $f$ does not agree with $\ell_{i+1}$ on all of $[x_{i+1}, x_{i+2}]$), so by \cref{lemma:same_sign} and \cref{lemma:opp_sign}, it must be the case that $-1= \sgn(s_{i+1} - \si(f, x_{i+1})) = \sgn(\so(f, x_{i+2}) - s_{i+1})$, that $f$ has a single knot inside $(x_{i+1}, x_{i+2})$, and that $\si(f, x_{i+2}) = \so(f, x_{i+2})$. (See \cref{fig:geom_char_opp_sign_orig}.) Therefore, $\si(f, x_{i+2}) = \so(f, x_{i+2}) < s_{i+1} \leq s_i < \so(f, x_{i+1}) = \si(f, x_{i+1})$. Furthermore, because $1 = \sgn(\so(f, x_{i+1}) - s_i) = \sgn(s_i - \si(f, x_i))$, we have $\si(f, x_i) < s_i < \so(x_{i+1})$. On $I := [x_{i-1}-\epsilon, x_{i+2} + \epsilon]$ for small $\epsilon > 0$, we thus have
    \begin{align}
        V_p (f \big\rvert_I ) &= |\so(f, x_{i+1}) - \si(f, x_i)|^p + |\so(f, x_{i+2}) - \so(f, x_{i+1})|^p \\
        &> |s_i - \si(f, x_i)|^p + |\so(f, x_{i+2}) - s_i|^p  = V_p(g \big\rvert_I)
    \end{align}
    where $g$ agrees with $f$ outside of $[x_i, x_{i+2}]$, agrees with $\ell_i$ on $[x_i, x_{i+1}]$, and has a single knot inside $[x_{i+1}, x_{i+2}]$ with $\so(g, x_{i+1}) = s_i$ and $\si(g, x_{i+2}) = \so(g, x_{i+2}) = \so(f, x_{i+2})$. (See \cref{fig:geom_char_opp_sign_red}.) This contradicts $f \in S_p^*$. 

    For $p = 0$, consider some $f \in S_0^*$. If $\sgn(s_i - \si(f, x_i)) \neq \sgn(\so(f, x_{i+1}) - s_i)$, then by \cref{lemma:opp_sign}, there is some $g \in S_0^*$ which agrees with $f$ outside of $[x_i, x_{i+1}]$ and agrees with $\ell_i$ on $[x_i, x_{i+1}]$. By \cref{lemma:same_sign}, if $\sgn(s_i - \si(f, x_i)) = \sgn(\so(f, x_{i+1}) - s_i) = 0$, then $f = \ell_i$ on $[x_i, x_{i+1}]$. If $\sgn(s_i - \si(f, x_i)) = \sgn(\so(f, x_{i+1}) - s_i)$ are both nonzero, then by \cref{lemma:same_sign}, $\si(f, x_i) = \so(f, x_i)$ and $\si(f, x_{i+1}) = \so(f, x_{i+1})$, and $f$ has a single knot inside $(x_i, x_{i+1})$ where the incoming line at $x_i$ and the outgoing line at $x_{i+1}$ meet. As before, it must be the case that $\sgn(s_i - s_{i-1}) \neq \sgn(s_i - \si(f, x_i))$ and/or that $\sgn(s_{i+1} - s_i) \neq \sgn(\so(f, x_{i+1}) - s_i)$. Assume without loss of generality that $\sgn(s_{i+1} - s_i) \neq \sgn(\so(f, x_{i+1}) - s_i) = 1$, so that $s_{i+1} \leq s_i < \so(f, x_{i+1}) = \si(f, x_{i+1})$. Because $1 = \sgn(\so(f, x_{i+1}) - s_i) = \sgn(s_i - \si(f, x_i))$, we also have $\si(f, x_i) < s_i < \so(f, x_{i+1})$. If $\sgn(\so(f, x_{i+2}) - s_{i+1}) \neq \sgn(s_{i+1} - \si(f, x_{i+1})) = -1$, then by \cref{lemma:opp_sign}, there is some $g \in S_0^*$ which agrees with $f$ outside $[x_{i+1}, x_{i+2}]$ and agrees with $\ell_{i+1}$ on $[x_{i+1}, x_{i+2}]$. Then this $g$ has $\so(g, x_{i+1}) = s_i$ and $\si(g, x_i) = \si(f, x_i)$, so $\sgn(\so(x_{i+1}) - s_i) \in \{-1, 0\}$, and $\sgn(s_i - \si(g, x_i)) = 1$; hence by \cref{lemma:opp_sign}, there is some $h \in S_0^*$ which agrees with $g$ outside of $[x_i, x_{i+1}]$ and agrees with $\ell_i$ on $[x_i, x_{i+1}]$. On the other hand, if $\sgn(\so(f, x_{i+2}) - s_{i+1}) = \sgn(s_{i+1} - \si(f, x_{i+1})) = -1$, then by \cref{lemma:same_sign}, $f$ has a single knot inside $(x_{i+1}, x_{i+2})$, and $\si(f, x_{i+2}) = \so(f, x_{i+2})$, as in \cref{fig:geom_char_opp_sign_orig}. This function has two knots on $I := [x_{i-1} - \epsilon, x_{i+2} + \epsilon]$ (for small $\epsilon > 0$). The function $g$ depicted in \cref{fig:geom_char_opp_sign_red}, which agrees with $f$ outside of $[x_i, x_{i+2}]$, agrees with $\ell_i$ on $[x_i, x_{i+1}]$, and has a single knot inside $[x_{i+1}, x_{i+2}]$ with $\so(g, x_{i+1}) = s_i$ and $\si(g, x_{i+2}) = \so(g, x_{i+2}) = \so(f, x_{i+2})$, also has two knots on $I$. Therefore $g \in S_0^*$. 
    \begin{figure}
    \centering
    \begin{subfigure}{0.5\textwidth}
          \centering
          \includegraphics[width=0.9\linewidth]{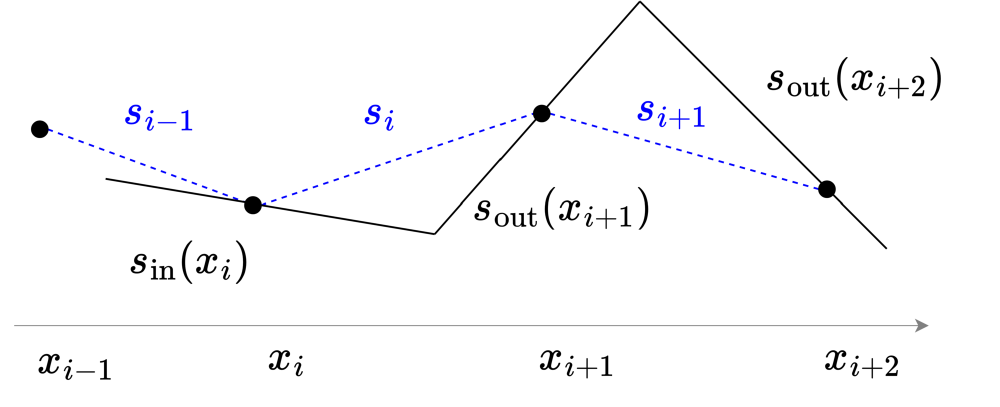}
          \caption{A function with knots inside $(x_i, x_{i+1})$ and $(x_{i+1}, x_{i+2})$.}
          \label{fig:geom_char_opp_sign_orig}
    \end{subfigure}%
    \begin{subfigure}{0.5\textwidth}
          \centering
          \includegraphics[width=0.8\linewidth]{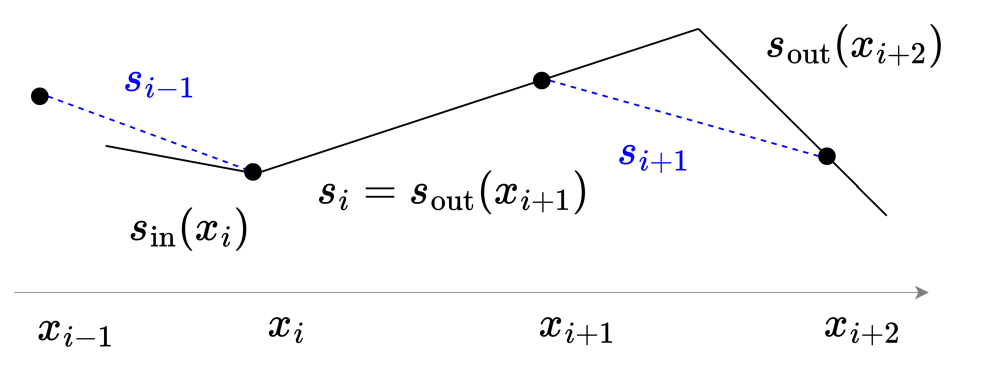}
          \caption{A function which agrees with $\ell_i$ on $[x_i, x_{i+1}]$.}
          \label{fig:geom_char_opp_sign_red}
    \end{subfigure}
    \caption{Behavior of $f \in S_p^*$ between data points of opposite curvature. The knot inside $(x_i, x_{i+1})$ on the left can be moved to $x_i$, and the knot inside $(x_{i+1}, x_{i+2})$ can be adjusted accordingly (right); this reduces the magnitudes of the slope changes of both knots.}
    \label{fig:geom_char_opp_sign}
\end{figure}
\paragraph{Linearity between collinear data points.} 
For $0 < p < 1$, fix $f \in S_p^*$. If $\si(f, x_i) = s_i = s_{i+1} = \dots = s_{i+m-1} = \so(f, x_{i+m})$, then $f$ must agree with $\ell_i = \dots = \ell_{i+m-1}$ on $[x_i, x_{i+m}]$, since any other function $g$ would have $V_p(g \big\rvert_I) > 0 = V_p (f \big\rvert_I )$ on $I := [x_i - \epsilon, x_{i+m} + \epsilon]$ for small $\epsilon > 0$. If $\sgn(s_i - \si(f, x_i)) \neq \sgn(\so(f, x_{i+m}) - s_i)$, then the argument in the proof of \cref{lemma:opp_sign} shows that $f$ must agree with $\ell_i = \dots = \ell_{i+m-1}$ on $[x_i, x_{i+m}]$. So we need only consider the case where $\sgn(s_i - \si(f, x_i)) = \sgn(\so(f, x_{i+m}) - s_i)$ are both nonzero; say without loss of generality that they both equal 1, so that $\si(f, x_i) < s_i < \so(f, x_{i+m})$. If $f = \ell_i$ on both $[x_i, x_{i+1}]$ and $[x_{i+m-1}, x_{i+m}]$, then it also must agree with $\ell_i$ on $[x_{i+1}, x_{i+m-1}]$ (otherwise it would have $V_p (f \big\rvert_{[x_i, x_{i+m}]}) > 0$), so assume by contradiction that $f \neq \ell_i$ on at least one of these intervals, say without loss of generality on $[x_i, x_{i+1}]$. Then by \cref{lemma:same_sign,lemma:opp_sign}, it must be the case that $f$ has a single knot inside $(x_i, x_{i+1})$ and that $\si(f, x_i) = \so(f, x_i) < s_i < \si(f, x_{i+1}) = \so(f, x_{i+1})$. This implies that $f$ also disagrees with $\ell_i$ on $[x_i, x_{i+1}]$, so again by \cref{lemma:same_sign,lemma:opp_sign}, $f$ must have a single knot inside $(x_{i+1}, x_{i+2})$ with $\si(f, x_{i+1}) = \so(f, x_{i+1}) > s_{i+1} > \si(f, x_{i+2}) = \so(f, x_{i+2})$. The same logic applies on the remaining intervals up to and including $[x_{i+m-1}, x_{i+m}]$ (see \cref{fig:geom_char_collinear}). Note that if $m$ is even, we will have $\si(f, x_{i+m-1}) = \so(f, x_{i+m-1}) > s_{i+m-1} = s_i > \si(f,x_{i+m}) = \so(f, x_{i+m})$, contradicting the assumption that $\sgn(\so(f, x_{i+m}) - s_i) = 1$ (see \cref{fig:geom_char_collinear_ex1}). If $m$ is odd, as in \cref{fig:geom_char_collinear_ex2}, we have
\begin{align}
    V_p (f \big\rvert_I ) &= |\so(f, x_{i+1}) - \si(f, x_i)|^p + |\so(f, x_{i+2}) - \so(f, x_{i+1})|^p \\
    &\ \  + \dots + |\so(f, x_{i+m}) - \so(f, x_{i+m-1})|^p \\
    &> |s_i - \si(f, x_i)|^p + |\so(f, x_{i+m}) - s_{i+m-1}|^p = V_p(g \big\rvert_I)
\end{align}
where $g$ is the function which agrees with $f$ outside of $[x_i, x_{i+m}]$ and agrees with $\ell_i = \dots = \ell_{i+m-1}$ on $[x_i, x_{i+m}]$; this contradicts $f \in S_p^*$.
\begin{figure}
    \centering
    \begin{subfigure}{0.45\textwidth}
          \centering
          \includegraphics[width=\linewidth]{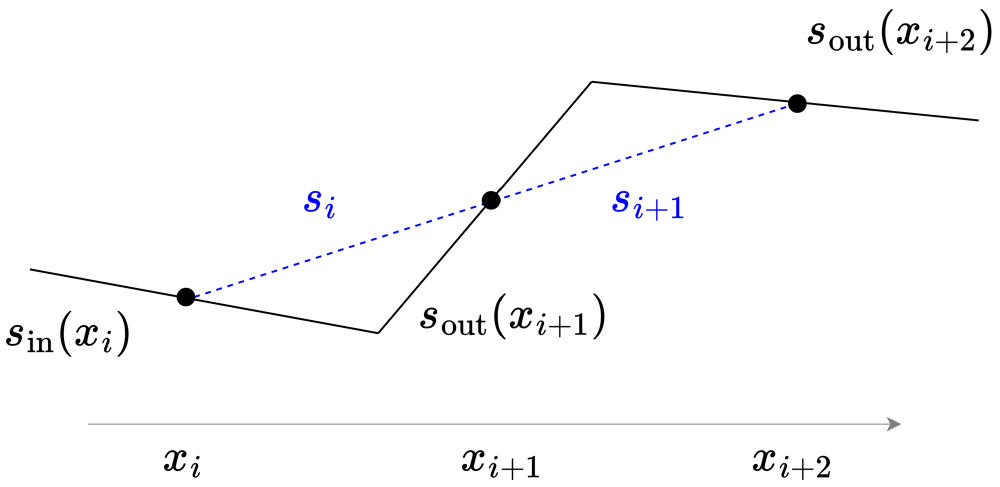}
          \caption{A nonlinear function between $m+1$ collinear points, $m$-even.}
          \label{fig:geom_char_collinear_ex1}
    \end{subfigure}\hspace{7mm}%
    \begin{subfigure}{0.45\textwidth}
          \centering
          \includegraphics[width=\linewidth]{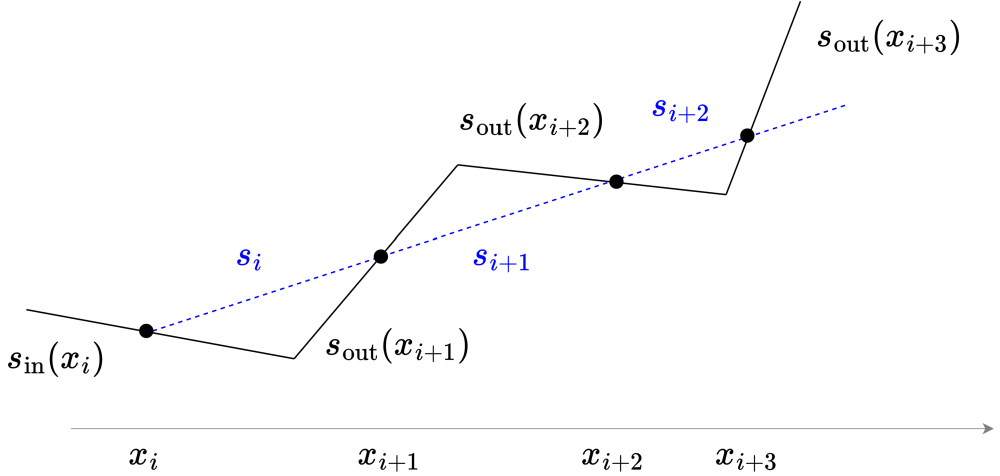}
          \caption{A nonlinear function between $m+1$ collinear points, $m$-odd.}
          \label{fig:geom_char_collinear_ex2}
    \end{subfigure}
    \caption{Behavior of $f \in S_p^*$ between collinear points. If $f \in S_p^*$ is \textit{not} a straight line between collinear points $(x_i, y_i), \dots, (x_{i+m}, y_{i+m})$, it must look like \cref{fig:geom_char_collinear_ex1} (if $m$ is even) or \cref{fig:geom_char_collinear_ex2} (if $m$ is odd). In both cases, the sum of absolute slope changes of these functions is greater than the sum of absolute slope changes of the function $g$ which agrees with $f$ outside of $[x_i, x_{i+m}]$ and connects $(x_i, y_i), \dots, (x_{i+m}, y_{i+m})$ with a straight line. Such a $g$ has two knots, whereas functions of the form $f$ depicted here have $m \geq 2$ knots.}
    \label{fig:geom_char_collinear}
\end{figure}

In the case $p = 0$, fix $f \in S_0^*$. If $\si(f, x_i) = s_i = \dots = s_{i+m-1} = \so(f, x_{i+m})$, then $f$ must agree with $\ell_i = \dots = \ell_{i+m-1}$ on $[x_i, x_{i+m}]$ and if $\sgn(s_i - \si(f, x_i)) \neq \sgn(\so(f, x_{i+m}) - s_i)$, then the proof of \cref{lemma:opp_sign} shows that there is some $g \in S_0^*$ which agrees with $f$ outside of $[x_i, x_{i+m}]$ and agrees with $\ell_i$ on $[x_i, x_{i+m}]$. If $\sgn(s_i - \si(f, x_i)) = \sgn(\so(f, x_{i+m}) - s_i)$ are both nonzero, then there must be at least one knot on $[x_i, x_{i+m}]$ in order for the slope to change from $\si(f, x_i)$ to $\so(f, x_{i+m})$. It is impossible for $f$ to interpolate the data with a single knot on $[x_i, x_{i+m}]$ where the slope changes from $\si(f, x_i)$ to $\so(f, x_{i+m})$, since this would require at least two of the points $(x_i, y_i), \dots, (x_{i+m}, y_{i+m})$ to both lie on either the incoming line at $x_i$ or the outgoing point at $x_{i+m}$, but this is impossible because $s_i \neq \si(f, x_i)$ and $s_i \neq \so(f, x_{i+m})$. Therefore, $f$ must have at least two knots on $[x_i, x_{i+m}]$. The function $g$ which agrees with $\ell_i$ on $[x_i, x_{i+m}]$ and has $\si(g, x_i) = \si(f, x_i)$ and $\so(g, x_{i+m}) = \so(f, x_{i+m})$ interpolates the points $(x_i, y_i), \dots, (x_{i+m}, y_{i+m})$ with exactly two knots on $[x_i - \epsilon, x_{i+m} + \epsilon]$, and thus $g \in S_0^*$.
\paragraph{Single knot between two data points with the same curvature.}  For $0 < p < 1$, fix $f \in S_p^*$. If $i = 2$, then $f = \ell_1$ on $(-\infty, x_2]$ by \cref{th:geom_char},\ref{th:geom_char_1}. If $i > 2$, then by assumption, $\sgn(s_{i-1}-s_{i-2}) \neq \sgn(s_i - s_{i-1})$, so by \cref{th:geom_char},\ref{th:geom_char_1}, $f = \ell_{i-1}$ on $[x_{i-1}, x_i]$. In either case, we have $\si(f, x_i) = s_{i-1}$. An analogous argument shows that $\so(f, x_{i+1}) = s_{i+1}$. Similarly, \cref{th:geom_char},\ref{th:geom_char_1} says that there is some $g \in S_0^*$ for which $\si(g, x_i) = s_{i-1}$ and $\so(g, x_{i+1}) = s_{i+1}$. In both cases, the conclusion then follows from \cref{lemma:opp_sign}. 
\paragraph{Characterization around $\geq 2$ points with the same curvature.}  For $0 < p < 1$, fix some $f \in S_p^*$. As in the proof of \cref{th:geom_char},\ref{th:geom_char_2a} above, the assumptions guarantee that $s_{i-1} = \si(f, x_i)$ and $s_{i+m} = \so(f, x_{i+m})$. Using this fact, we will proceed by (strong) induction, assuming without loss of generality that $\sgn(s_i - s_{i-1}) = \sgn(s_{i+1} - s_i) = \dots = \sgn(s_{i+m}-s_{i+m-1})=1$.

In the base case $m = 2$, first suppose that $\sgn(s_i - \si(f, x_i)) \neq \sgn(\so(f, x_{i+1}) - s_i)$. Since $\si(f, x_i) = s_{i-1} < s_i$ by assumption, it must be the case that $\sgn(\so(f, x_{i+1}) - s_i) \in \{0, -1\}$. If $\sgn(\so(f, x_{i+1}) - s_i) = -1$, \cref{lemma:opp_sign} implies that $f = \ell_i$ on $[x_i, x_{i+1}]$, and thus $\si(f, x_{i+1}) = s_i$. But then we have $\si(f, x_{i+1}) = s_i < s_{i+1} < \so(f, x_{i+2}) = s_{i+2}$, so by \cref{lemma:same_sign}, it must be the case that $\si(f, x_{i+1}) = \so(f, x_{i+1})$, contradicting $\sgn(\so(f, x_{i+1}) - s_i) = -1$ (see \cref{fig:4b_base_imp_2}). If $\sgn(\so(f, x_{i+1}) - s_i) = 0$, then \cref{lemma:opp_sign} implies that $f = \ell_i$ on $[x_i, x_{i+1}]$, and therefore $\si(f, x_{i+1}) = \so(f, x_{i+1}) = s_i$. Then $\si(f, x_{i+1}) = s_i < s_{i+1} < \so(f, x_{i+2}) = s_{i+2}$, so by \cref{lemma:same_sign}, $f$ has a single knot inside $[x_{i+1}, x_{i+2}]$, with $\si(f, x_{i+1}) = \so(f, x_{i+1}) = s_i$ (as we already know) and $\si(f, x_{i+2}) = \so(f, x_{i+2}) = s_{i+2}$. The conclusion then holds with $u_1 := s_i$ (see \cref{fig:4b_base_ui_eq_si}).

\begin{figure}
    \centering
    \begin{subfigure}{0.5\textwidth}
          \centering
          \includegraphics[width=0.9\linewidth]{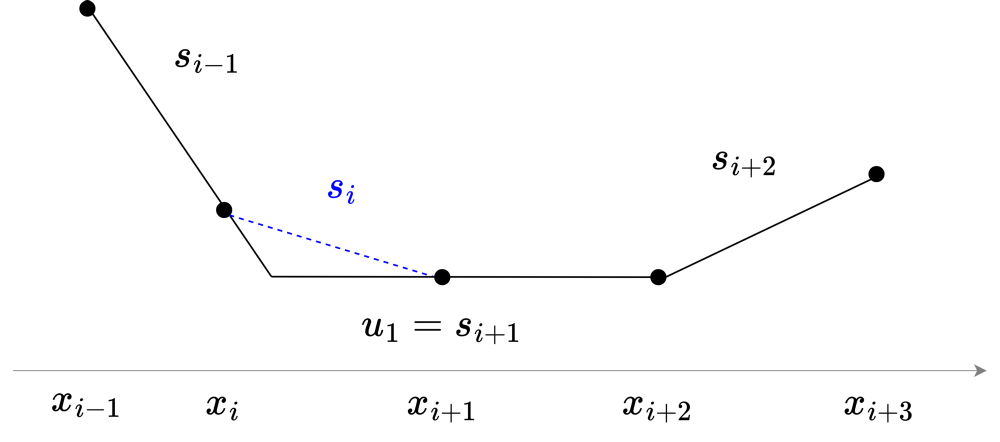}
          \caption{$u_1 = s_{i+1}$}
          \label{fig:4b_base_ui_eq_sip1}
    \end{subfigure}%
    \begin{subfigure}{0.5\textwidth}
          \centering
          \includegraphics[width=0.9\linewidth]{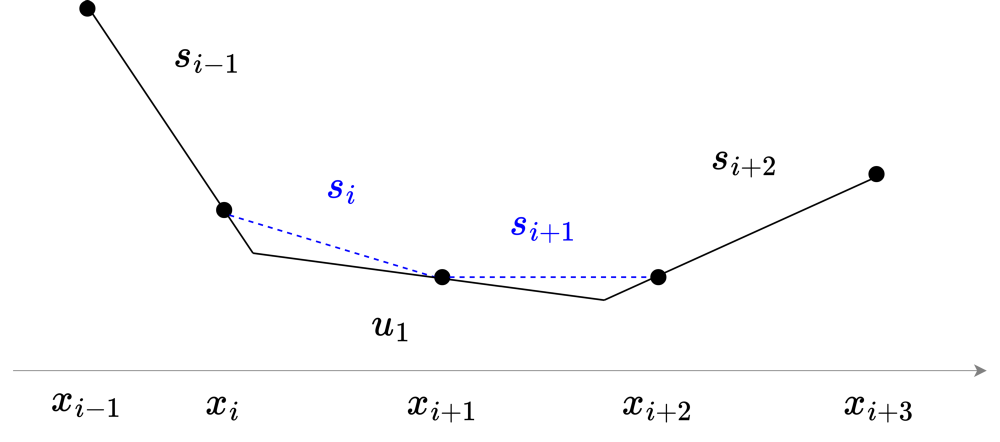}
          \caption{$s_i < u_1 < s_{i+1}$}
          \label{fig:4b_base_ui_betw_si_sip1}
    \end{subfigure}
    \par\bigskip
    \begin{subfigure}{0.5\textwidth}
          \centering
          \includegraphics[width=0.9\linewidth]{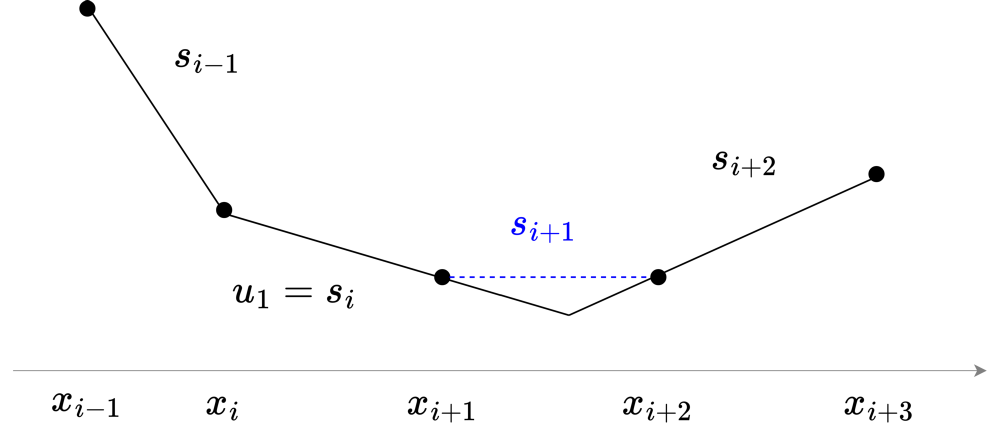}
          \caption{$u_i = s_i$}
          \label{fig:4b_base_ui_eq_si}
    \end{subfigure}
    \caption{Possible behavior of $f \in S_p^*$ between three consecutive data points of the same discrete curvature. All possibilities satisfy $s_i \leq u_1 \leq s_{i+1}$.}
    \label{fig:4b_base_possible}
\end{figure}

\begin{figure}
    \centering
    \begin{subfigure}{0.5\textwidth}
          \centering
          \includegraphics[width=0.9\linewidth]{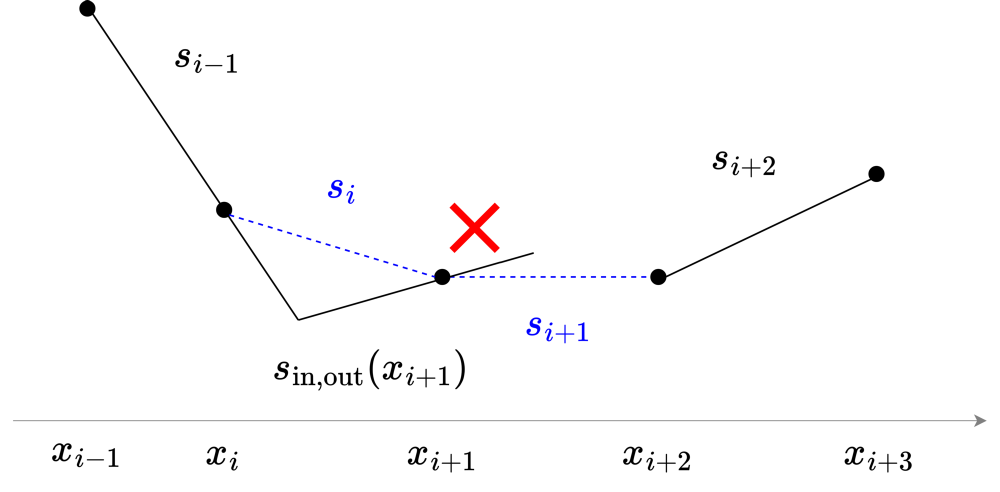}
          \caption{$\si(x_{i+1}) = \so(x_{i+1}) > s_{i+1}$}
          \label{fig:4b_base_imp_1}
    \end{subfigure}%
    \begin{subfigure}{0.5\textwidth}
          \centering
          \includegraphics[width=0.9\linewidth]{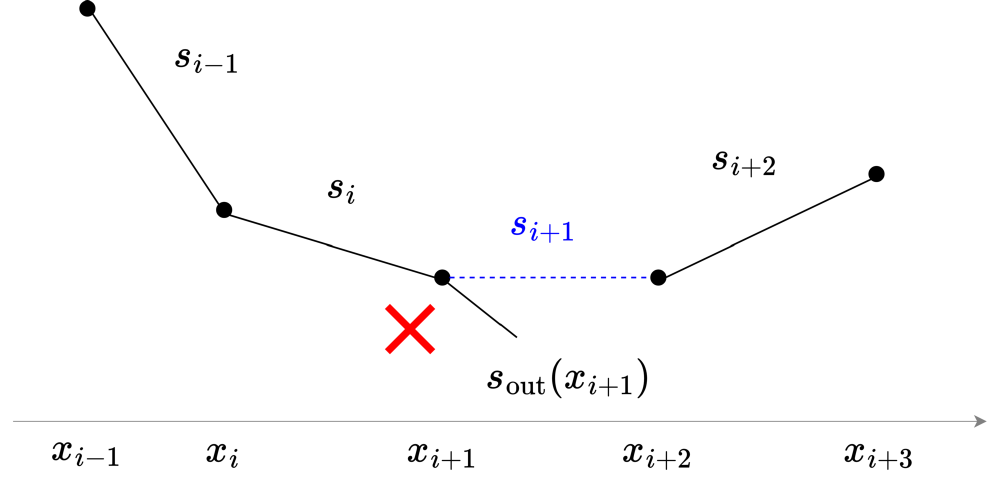}
          \caption{$\so(x_{i+1}) < \si(x_{i+1}) = s_i$}
          \label{fig:4b_base_imp_2}
    \end{subfigure}
    \caption{Behaviors which $f \in S_p^*$ for $0 < p < 1$ \textit{cannot} exhibit around three consecutive points of the same discrete curvature. The case on the left violates \cref{lemma:opp_sign}, and the case on the right violates \cref{lemma:same_sign}.}
    \label{fig:4b_base_impossible}
\end{figure}

On the other hand, still for the base case $m=2$, suppose that $\sgn(s_i - \si(f, x_i)) = \sgn(\so(f, x_{i+1}) - s_i)$. Then by \cref{lemma:same_sign}, there is a single knot inside $[x_i, x_{i+1}]$, with $s_{i-1} = \si(f, x_i) = \so(f, x_i)$ and $\si(f, x_{i+1}) = \so(f, x_{i+1})$. It cannot be the case that $\so(f, x_{i+1}) > s_{i+1}$, because if this were true, we would have $-1 = \sgn(s_{i+1} - \si(f, x_{i+1})) \neq \sgn(\so(f, x_{i+2})-s_{i+1}) = 1$, and that would imply by \cref{lemma:opp_sign} that $f = \ell_{i+1}$ on $[x_{i+1}, x_{i+2}]$, contradicting $\so(f, x_{i+1}) > s_{i+1}$ (see \cref{fig:4b_base_imp_1}). Therefore, we must have $\so(f, x_{i+1}) \leq s_{i+1}$. If $\so(f, x_{i+1}) < s_{i+1}$, then by \cref{lemma:same_sign}, there is a single knot on $[x_{i+1}, x_{i+2}]$, with $\si(f, x_{i+1}) = \so(f, x_{i+1})$ (as we already knew) and $\si(f, x_{i+2}) = \so(f, x_{i+2}) = s_{i+2}$. The conclusion then holds with $u_1 := \si(f, x_{i+1}) = \so(f, x_{i+1})$ (see \cref{fig:4b_base_ui_betw_si_sip1}). If $\so(f, x_{i+1}) = s_{i+1}$, then $0 = \sgn(s_{i+1} - \si(f, x_{i+1}) \neq \sgn(\so(f, x_{i+2}) - s_{i+1}) = 1$, so by \cref{lemma:opp_sign}, $f = \ell_{i+1}$ on $[x_{i+1}, x_{i+2}]$. The conclusion then holds with $u_1 := s_{i+1}$ (see \cref{fig:4b_base_ui_eq_sip1}).

Next, for the (strong) inductive step, fix some integer $m \geq 4$ and assume the conclusion holds for all integers $2, \dots, m-1$. First suppose that $\so(f, x_{i+m-1}) > s_{i+m-2}$. Then by the inductive hypothesis, $f$ has slopes $u_1, \dots, u_{m-2}$---some of which may be equal to each other, but all of which are distinct from $\si(f, x_i) = s_{i-1}$ and $\so(f, x_{i+m-1})$---on $[x_i, x_{i+m-1}]$ satisfying $s_{i+j-1} \leq u_j \leq s_{i+j}$ for all $j = 1, \dots, m-2$. It cannot be the case that $\so(f, x_{i+m-1}) > s_{i+m-1}$, because if this were true, we would have $-1 = \sgn(s_{i+m-1} - \si(f, x_{i+m-1})) \neq \sgn(\so(f, x_{i+m}) - s_{i+m-1}) = 1$, and thus \cref{lemma:opp_sign} would imply that $f = \ell_{i+m-1}$ on $[x_{i+m-1}, x_{i+m}]$, contradicting $\so(f, x_{i+m-1}) > s_{i+m-1}$ (see \cref{fig:4b_base_imp_2}). Therefore, we must have $\so(f, x_{i+m-1}) \leq s_{i+m-1}$. If $\so(f, x_{i+m-1}) < s_{i+m-1}$, then by \cref{lemma:same_sign}, there is a single knot inside $[x_{i+m-1}, x_{i+m}]$ and $\si(f, x_{i+m-1}) = \so(f, x_{i+m-1})$ and $\si(f, x_{i+m}) = \so(f, x_{i+m}) = s_{i+m}$. The conclusion then holds for $m$ with $u_{m-1} := \so(f, x_{i+m-1})$ (see \cref{fig:4b_ind_umm1_betw_simm2+simm1}). If $\so(f, x_{i+m-1}) = s_{i+m-1}$, then by \cref{lemma:same_sign,lemma:opp_sign}, it must be the case that $\{0, -1\} \ni \sgn(s_{i+m-1}-\si(f, x_{i+m-1})) \neq \sgn(\so(f, x_{i+m}) - s_{i+m-1})) = 1$. It is impossible that $\sgn(s_{i+m-1}-\si(f, x_{i+m-1})) = -1$ because by \cref{lemma:same_sign,lemma:opp_sign}, for $f$ to disagree with $\ell_{i+m-2}$ on $[x_{i+m-2}, x_{i+m-1}]$, it must be the case that $\si(f, x_{i+m-1}) = \so(f, x_{i+m-1})$, contradicting $\si(f, x_{i+m-1}) < \so(f, x_{i+m-1}) = s_{i+m-1}$ (see \cref{fig:4b_ind_imp_1}, red). Therefore, in this case we have $\si(f, x_{i+m-1}) = \so(f, x_{i+m-1}) = s_{i+m-1}$, and the conclusion holds for $m$ with $u_{m-1} := s_{i+m-1}$ (see \cref{fig:4b_ind_imp_1}, green).

\begin{figure}
    \centering
    \begin{subfigure}{0.5\textwidth}
          \centering
          \includegraphics[width=0.95\linewidth]{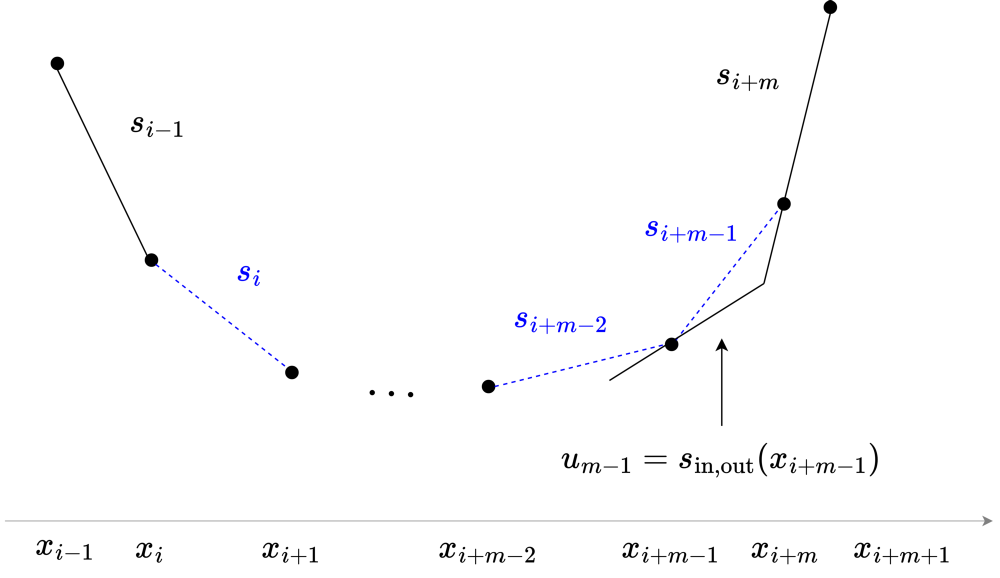}
          \caption{$s_{i+m-2} < u_{m-1} < s_{i+m-1}$}
          \label{fig:4b_ind_umm1_betw_simm2+simm1}
    \end{subfigure}%
    \begin{subfigure}{0.5\textwidth}
          \centering
          \includegraphics[width=0.95\linewidth]{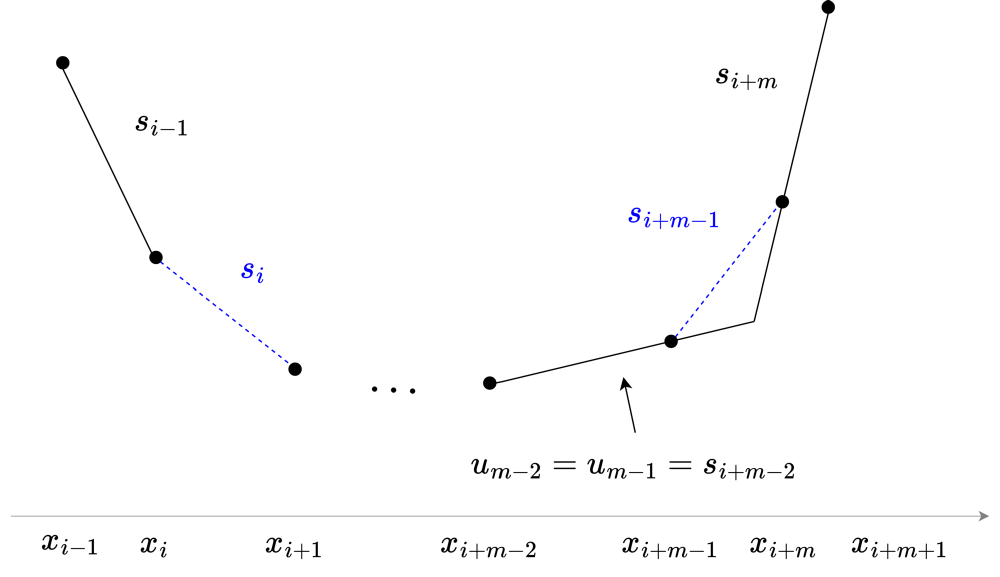}
          \caption{$u_{m-2} = u_{m-1} = s_{i+m-2}$}
          \label{fig:4b_ind_umm2_umm1_eq_simm2}
    \end{subfigure}
    \caption{Possible behavior of $f \in S_p^*$ around $m$ consecutive data points of the same discrete curvature. Assuming inductively that \cref{th:geom_char},\ref{th:geom_char_2b} holds for $2, \dots, m-1$, both satisfy $s_{i+j-1} \leq u_j \leq s_{i+j}$ for $j = 1, \dots, m-1$.}
    \label{fig:4b_ind_possible}
\end{figure}

\begin{figure}
    \centering
    \begin{subfigure}{0.5\textwidth}
          \centering
          \includegraphics[width=0.95\linewidth]{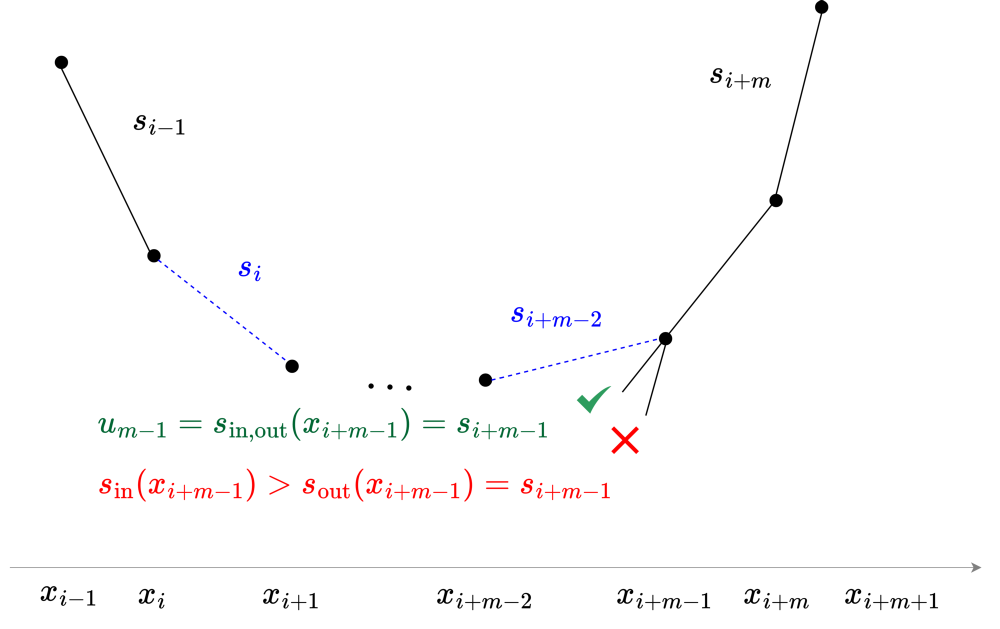}
          \caption{$s_{i+m-1} = \so(x_{i+m-1}) \leq \si(x_{i+m-1})$}
          \label{fig:4b_ind_imp_1}
    \end{subfigure}%
    \begin{subfigure}{0.5\textwidth}
          \centering
          \includegraphics[width=0.95\linewidth]{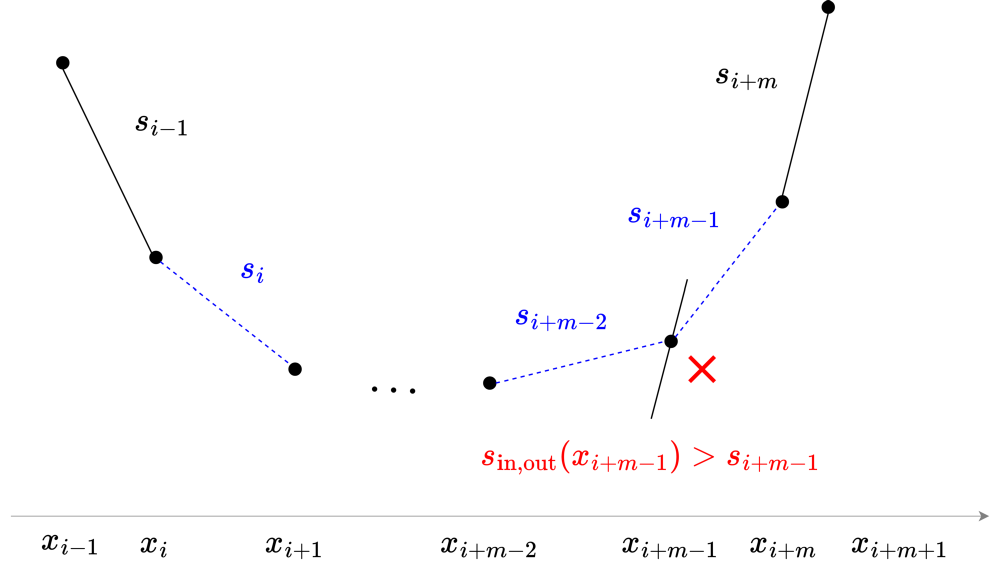}
          \caption{$\si(x_{i+m-1}) = \so(x_{i+m-1}) > s_{i+m-1}$}
          \label{fig:4b_ind_imp_2}
    \end{subfigure}
    \caption{Behaviors which $f \in S_p^*$ can and cannot exhibit between $m$ consecutive points of the same discrete curvature. Assuming inductively that \cref{th:geom_char},\ref{th:geom_char_2b} holds for $2, \dots, m-1$, the case with the green check mark on the left satisfies $s_{i+j-1} \leq u_j \leq s_{i+j}$ for $j = 1, \dots, m-1$. The case with the red \textit{x} on the left violates \cref{lemma:same_sign}, and the case on the right violates \cref{lemma:opp_sign}.}
    \label{fig:4b_ind_impossible}
\end{figure}

\begin{figure}
    \centering
    \begin{subfigure}{0.5\textwidth}
          \centering
          \includegraphics[width=0.95\linewidth]{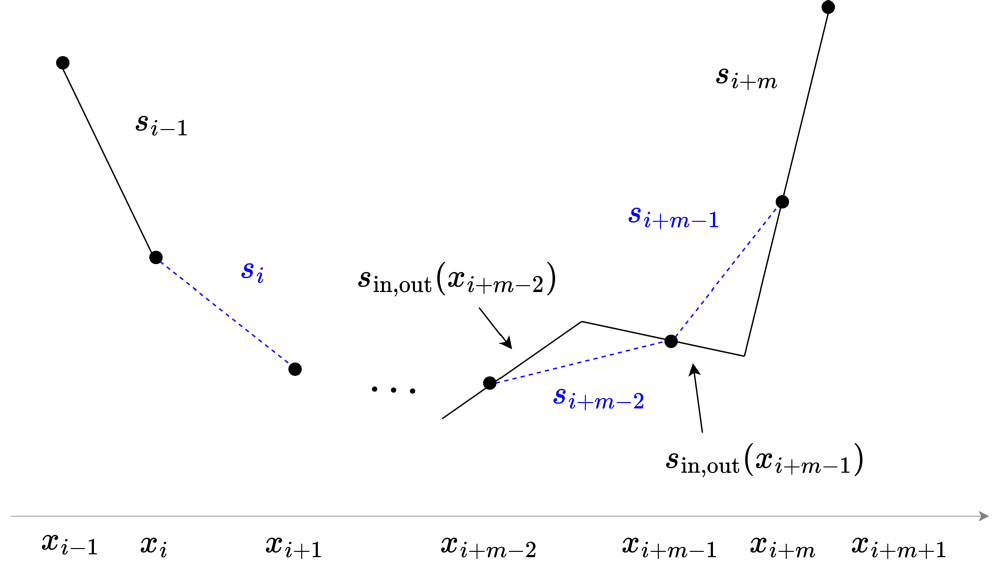}
          \caption{A function with $\so(x_{i+m-1}) < s_{i+m-2}$.}
          \label{fig:4b_ind_contra_orig}
    \end{subfigure}%
    \begin{subfigure}{0.5\textwidth}
          \centering
          \includegraphics[width=0.95\linewidth]{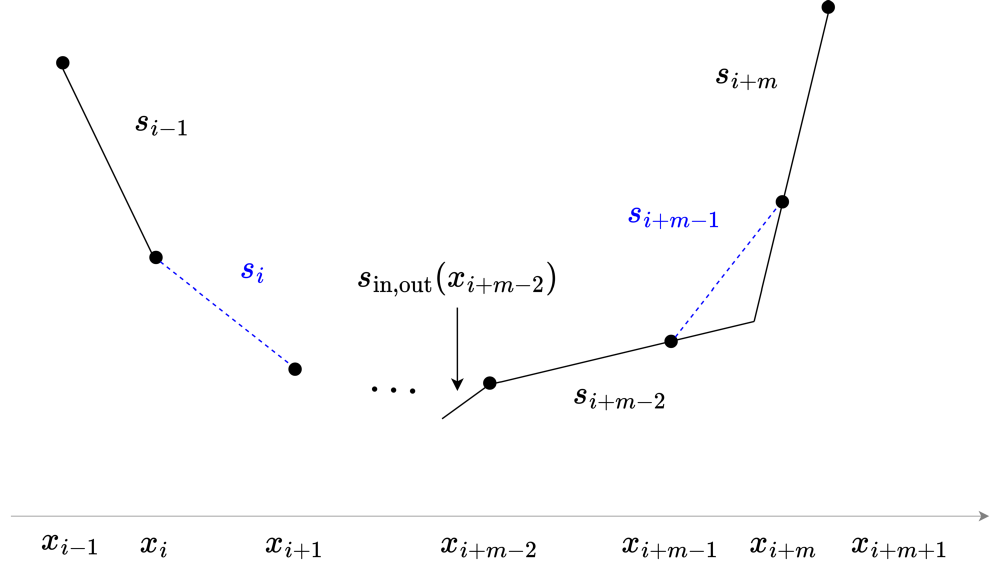}
          \caption{A function with $\si(x_{i+m-1}) = \so(x_{i+m-1}) = s_{i+m-2}$.}
          \label{fig:4b_ind_contra_red}
    \end{subfigure}
    \caption{Possible behavior of $f \in S_p^*$ around $m$ consecutive slope changes of the same discrete curvature. The magnitude of slope change at each knot of the function $f$ on the left, which has $\so(f, x_{i+m-1}) < s_{i+m-2}$, is greater than that of the corresponding knot in the function $g$ on the right, which has $\si(x_{i+m-1}) = \so(x_{i+m-1}) = s_{i+m-2}$.}
    \label{fig:4b_ind_contra}
\end{figure}

On the other hand, still for the (strong) inductive step, suppose that $\so(f, x_{i+m-1}) \leq s_{i+m-2}$. If $\so(f, x_{i+m-1}) = s_{i+m-2}$, then by \cref{lemma:same_sign,lemma:opp_sign}, $f$ has a single knot inside $[x_{i+m-1}, x_{i+m}]$ with $\si(f, x_{i+m-1}) = \so(f, x_{i+m-1}) = s_{i+m-2}$ and $\si(f, x_{i+m}) = \so(f, x_{i+m}) = s_{i+m}$. This implies, again by \cref{lemma:same_sign,lemma:opp_sign}, that $f = \ell_{i+m-2}$ on $[x_{i+m-2}, x_{i+m-1}]$. By the (strong) inductive hypothesis, $f$ has slopes $u_1, \dots, u_{m-3}$ on $[x_i, x_{i+m-2}]$, all distinct from $\si(f, x_i) = s_{i-1}$ and $\so(f, x_{i+m-2}) = s_{i+m-2}$, which satisfy $s_{i+j-1} \leq u_j \leq s_{i+j}$ for $j = 2, \dots, m-3$. The conclusion then holds for $m$ with $u_{m-2} = u_{m-1} := s_{i+m-2}$ (see \cref{fig:4b_ind_umm2_umm1_eq_simm2}). It remains only to consider the case $\so(f, x_{i+m-1}) < s_{i+m-2}$, and show that this is impossible for $f \in S_p^*$. If $\so(f, x_{i+m-1}) < s_{i+m-2}$, then by \cref{lemma:same_sign,lemma:opp_sign}, there is a single knot inside $[x_{i+m-1}-x_{i+m}]$ and $\si(f, x_{i+m-1}) = \so(f, x_{i+m-1})$ and $\si(f, x_{i+m}) = \so(f, x_{i+m}) = s_{i+m}$. This in turn implies, again by \cref{lemma:same_sign,lemma:opp_sign}, that there is a single knot inside $[x_{i+m-2}, x_{i+m-1}]$ and $\si(f, x_{i+m-2}) = \so(f, x_{i+m-2})$. (See \cref{fig:4b_ind_contra_orig}.) On the interval $I := [x_{i+m-2} - \epsilon, x_{i+m} + \epsilon]$ for small $\epsilon > 0$, we thus have
\begin{align}
    V_p (f \big\rvert_I ) &= |\so(f, x_{i+m-1}) - \si(f, x_{i+m-2})|^p + |s_{i+m}-\so(f, x_{i+m-1})|^p \\
    &> |s_{i+m-2} - \si(f, x_{i+m-2})|^p + |s_{i+m}-s_{i+m-2}|^p 
\end{align}
where the inequality holds because $\so(f, x_{i+m-1}) < s_{i+m-2} < s_{i+m}$ and $\si(f, x_{i+m-2}) > s_{i+m-2} > \so(f, x_{i+m-1})$. The latter is exactly $V_p(g \big\rvert_I)$, where $g$ is the function which agrees with $f$ outside of $[x_{i+m-2}, x_{i+m}]$, agrees with $\ell_{i+m-2}$ on $[x_{i+m-2}, x_{i+m-1}]$, and has a single knot in $[x_{i+m-1}, x_{i+m}]$ with $\si(f, x_{i+m-1}) = \so(f, x_{i+m-1}) = s_{i+m-2}$ and $\si(f, x_{i+m}) = \so(f, x_{i+m}) = s_{i+m}$. (See \cref{fig:4b_ind_contra_red}.) This contradicts $f \in S_p^*$.

For the case $p = 0$: again, as in the proof of \cref{th:geom_char},\ref{th:geom_char_2a}, the assumptions guarantee that there is some $f \in S_0^*$ for which $\si(f, x_i) = s_{i-1}$ and $\so(f, x_{i+1}) = s_{i+1}$. The inductive argument above for $0 < p < 1$ also shows the desired result in the $p = 0$ case, with each reference to \cref{lemma:opp_sign} as well as the last portion of the inductive step instead justifying the existence of some $g \in S_0^*$ which exhibits the desired local behavior and agrees with $f$ elsewhere. 
\paragraph{Non-emptiness of $S_p^*$ for $0 < p < 1$.} As noted in \cref{appendix:proof_prop_existence_and_variational}, restricting the input weights to $|w_k| = 1$ in optimization \eqref{opt:min_p_NN} recovers the same set of optimal functions $S_p^*$. The geometric characterization proved above shows that any solution to this modified \eqref{opt:min_p_NN} must have no knots outside of $[x_2, x_{N-1}]$, and thus its biases satisfy $|b_k| \leq B := \max\{ |x_2|, |x_{N-1}| \}$. Additionally, any such solution has slopes absolutely bounded by $ C := \max_{i=1, \dots, N-1} |s_i|$, so that each $|v_k w_k| = |v_k| \leq 2C$, and thus its skip connection parameters can be bounded as
\begin{align}
    |a| - \left| \sum_{w_k > 0} v_k \right| \leq  \left| a + \sum_{w_k > 0} v_k \right| = |f'(x_N + 1)| \leq C \implies |a| \leq A := C + \sum_{w_k > 0} |v| \leq C + 2KC 
\end{align}
and
\begin{align}
    c = y_1 - \sum_{k=1}^K v_k (w_k x_1 - b_k)_+ - a x_1 \implies |c| &\leq |y_1| + \sum_{k=1}^K |v_k|(|x_1| + |b_k|) + |a x_1| \\
    &\leq C_0 := |y_1| +  2KC(|x_1| + B) + |x_1| (C + 2KC)
\end{align}
Therefore, any $f \in S_p^*$ is recovered by a restricted version of \eqref{opt:min_p_NN} which requires that $|w_k| = 1, |b_k| \leq B, |v_k| \leq 2C, |a| \leq A, |c| \leq C_0$. For any fixed choice of $w_1, \dots, w_K \in \{-1,1\}^K$, this modified optimization (in the remaining variables) constitutes a minimization of a continuous function over a compact set, so by the Weierstrass extreme value theorem, a solution exists. Taking the minimum over all such solutions for all possible choices of  $w_1, \dots, w_K \in \{-1,1\}^K$ proves the result.
\end{proof}

\subsubsection{Proof of \cref{th:main}} \label{appendix:proof_main_th}
\begin{proof}
    If the data contain no more than two consecutive points with the same discrete curvature, there is only one interpolant $f$ which fits the description in \cref{th:geom_char}. By Theorem 4 in \cite{debarre2022sparsest}, this $f \in S_0^*$. Otherwise, if the data do contain some $x_i, \dots, x_{i+m}$ with the same discrete curvature for $m \geq 2$, the slopes $u_1, \dots, u_{m-1}$ of any interpolant satisfying the description in \cref{th:geom_char},\ref{th:geom_char_2b} have $s_{i+j-1} \leq u_j \leq s_{i+j}$ for each $j = 1, \dots, m-1$. Indeed, any choice of $u_1, \dots, u_{m-1}$ satisfying  $s_{i+j-1} \leq u_j \leq s_{i+j}$ for each $j$ defines an CPWL interpolant of the data, given by the pointwise maximum of $\ell_{i-1}$, $\ell_{i+m}$, and the lines $L_j$, each  of which has slope $u_j$ and passes through $(x_{i+j}, y_{i+j})$. Therefore, the set $S$ of functions described by \cref{th:geom_char},\ref{th:geom_char_2b} on any such $x_i, \dots, x_{i+m}$ can be fully associated with the set of numbers $u_1, \dots, u_m$ satisfying  $s_{i+j-1} \leq u_j \leq s_{i+j}$ for each $j$. Since any such $u_j = (1-\alpha_j) s_{i+j-1} + \alpha_j s_{i+j}$ for a unique $\alpha_j \in [0,1]$, we can equivalently identify $S$ with the unit cube $[0,1]^{m-1}$. 
    
    Viewed as a function of its corresponding $\valpha = [\alpha_1, \dots, \alpha_{m-1}]^\top \in [0,1]^{m-1}$, the regularization cost $V_p(f \rvert_I)$ (for $0 < p < 1$) of any $f \in S$ on $I := [x_{i-1}-\delta, x_{i+m+1} + \delta]$ for small $\delta > 0$ is
    \begin{align}
        V_p(\valpha) = |u_1 - s_{i-1}|^p + \sum_{j=2}^{m-1} |u_j - u_{j-1}|^p + |s_{i+m}-u_{m-1}|^p = \| \mA \valpha + \vc \|_p^p
    \end{align}
    where the rows $\va_1, \dots, \va_m$ of $\mA \in \R^{m \times (m-1)}$ and entries $c_1, \dots, c_m$ of $\vc \in \R^m$ are
    \begin{align}
        \va_1 &= [s_{i+1}-s_i, 0, \dots, 0]^\top, \qquad \qquad \ \ c_1 = s_i - s_{i-1} \\
        \va_m &= [0, \dots, 0, s_{i+m-1}-s_{i+m}]^\top, \qquad c_1 = s_{i+m} - s_{i+m-1}
    \end{align}
    and
    \begin{align}
        \va_j = [0, \dots, 0, -(s_{i+j-1} - s_{i+j-2}), s_{i+j}-s_{i+j-1}, 0, \dots, 0]^\top, \qquad c_j = s_{i+j-1}-s_{i+j-2}
    \end{align}
    for $j = 2, \dots, m-1$, with the nonzero entries of $\va_j$ in positions $j-1$ and $j$. By the assumption that $\epsilon_i = \dots = \epsilon_{i+m}$ are all nonzero, the rows $\va_1, \dots, \va_m$ of $\mA$ span $\R^{m-1}$, and thus $\valpha \mapsto \mA \valpha + \vc$ is injective. For any distinct $\valpha_1, \valpha_2 \in [0,1]^{m-1}$, we thus have $\mA \valpha_1 + \vc \neq \mA \valpha_2 + \vc$, and therefore
    \begin{align}
        V_p(t \valpha_1 + (1-t) \valpha_2) = \| t (\mA \valpha_1 + \vc)  + (1-t) (\mA \valpha_2 + \vc) \|_p^p > t \| \mA \valpha_1 + \vc \|_p^p + (1-t) \| \mA \valpha_2 + \vc \|_p^p
    \end{align}
    for any $t \in (0,1)$ by strict concavity of $\| \cdot \|_p^p$ on $[0,1]^{m-1}$. This shows that $V_p$ is strictly concave on $[0,1]^{m-1}$. By the Bauer maximum principle (\cite{aliprantis2006infinite}, Theorem 4.104), $V_p(\valpha)$ thus attains a minimum on $[0,1]^{m-1}$ at an extreme point of $[0,1]^{m-1}$. Moreover, by strict concavity of $V_p(\valpha)$, \textit{any} minimum of $V_p(\valpha)$ over $[0,1]^{m-1}$ must occur at an extreme point. Therefore, when searching for an $f \in S$ with minimal $V_p$, we may restrict our attention to those $f$ corresponding to the $2^{m-1}$ vertices $\{0,1\}^{m-1}$ of the cube $[0,1]^{m-1}$. 

    Among these $2^{m-1}$ vertices, there is at least one corresponding to a sparsest solution $f \in S_0^* \cap S$. This is because, by Theorem 4 in \cite{debarre2022sparsest}, any $f \in S_0^* \cap S$ has $\lceil \frac{m+1}{2} \rceil$ knots on $I$, and there is one such $f$ if $m$ is odd, or uncountably many if $m$ is even. If $m$ is odd, this unique $f$ corresponds to the vertex $[1,0,\dots,1,0]^\top \in \{0,1\}^{m-1}$; i.e., this $f$ has $u_j = s_{i+j}$ for odd $j$ and $u_j = s_{i+j-1}$ for even $j$. If $m$ is even, there are multiple vertices $\valpha \in \{0,1\}^{m-1}$ which attain the minimal number $\lceil \frac{m+1}{2} \rceil$ of knots on $I$: two examples are $[1,0,\dots,1,0,1]^\top \in \{0,1 \}^{m-1}$ (see \cref{fig:4b_m_4_sp2}) and $[0,1,\dots,0,1,0]^\top \in \{0,1\}^{m-1}$  (see \cref{fig:4b_m_4_sp1}).

    For each of the $2^{m-1}$ functions $f \in S$ corresponding to the vertices $\valpha \in \{0,1\}^{m-1}$, consider the associated ``cost curves'' $C_f(p) := V_p(f \rvert_I)$, which is simply the regularization cost $V_p(f \rvert_I )$ for that individual $f$ over $I$, viewed as a function of the variable $p \in [0,1]$. Each $C_f(p)$ is a \textit{generalized Dirichlet polynomial}\footnote{\textit{Generalized Dirichlet polynomials} are functions of the form $f(x) = \sum_{i=1}^n a_i b_i^x$, where $a_i, x \in \R$ and $b_1 \geq \dots \geq b_n > 0$.} of the variable $p$. By the generalized Descartes rule of signs for Dirichlet polynomials (\cite{jameson2006counting}, Theorem 3.1), any two cost curves $C_f(p), C_g(p)$ for distinct $f, g$ can only intersect at finitely many $p \in [0,1]$. Therefore, for any given $p \in [0,1]$ outside of that finite set (which has Lebesgue measure zero), a unique one of these $2^{m-1}$ candidate solutions $f$ has smaller cost $C_f(p) = V_p(f \rvert_I)$ than the others. Furthermore, the sparsest of these $2^{m-1}$ functions (i.e., the ones in $S \cap S_0^*$) will necessarily have smaller $C_f(0) = V_0(f \rvert_I)$ than the rest, and because all of the cost curves $C_f(p)$ are continuous, a unique one of these sparsest solutions will have smaller cost $C_f(p)$ than the others for all $p$ between 0 and $p^*$, which is the location of the first intersection of any two of these $2^{m-1}$ candidate solutions' cost curves. 
\end{proof}

\begin{figure}
    \centering
    \begin{subfigure}{0.45\textwidth}
          \centering
          \includegraphics[width=0.95\linewidth]{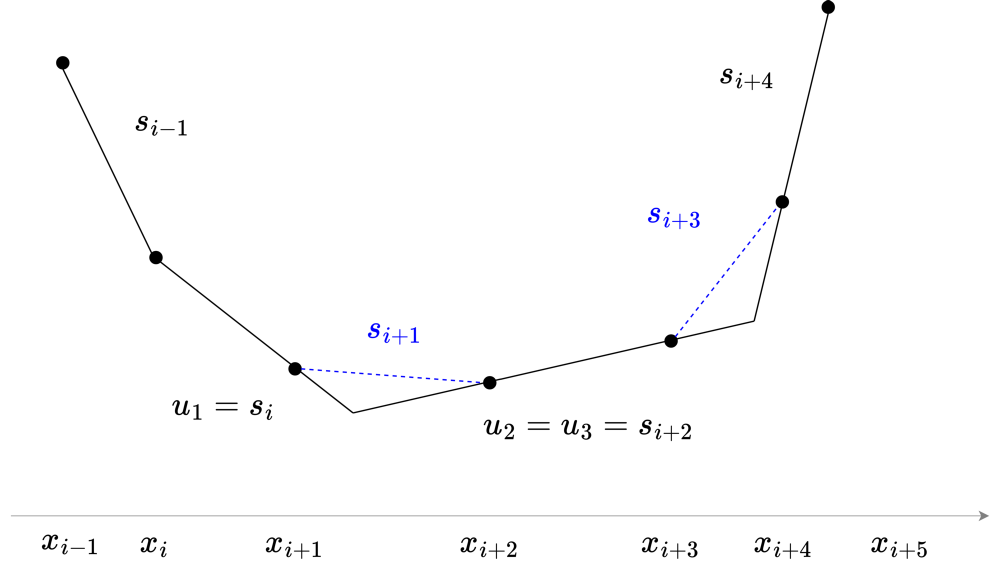}
          \caption{One sparsest interpolant, corresponding to $\valpha = [0,1,0]$.}
          \label{fig:4b_m_4_sp1}
    \end{subfigure}\hspace{5mm}%
    \begin{subfigure}{0.45\textwidth}
          \centering
          \includegraphics[width=0.95\linewidth]{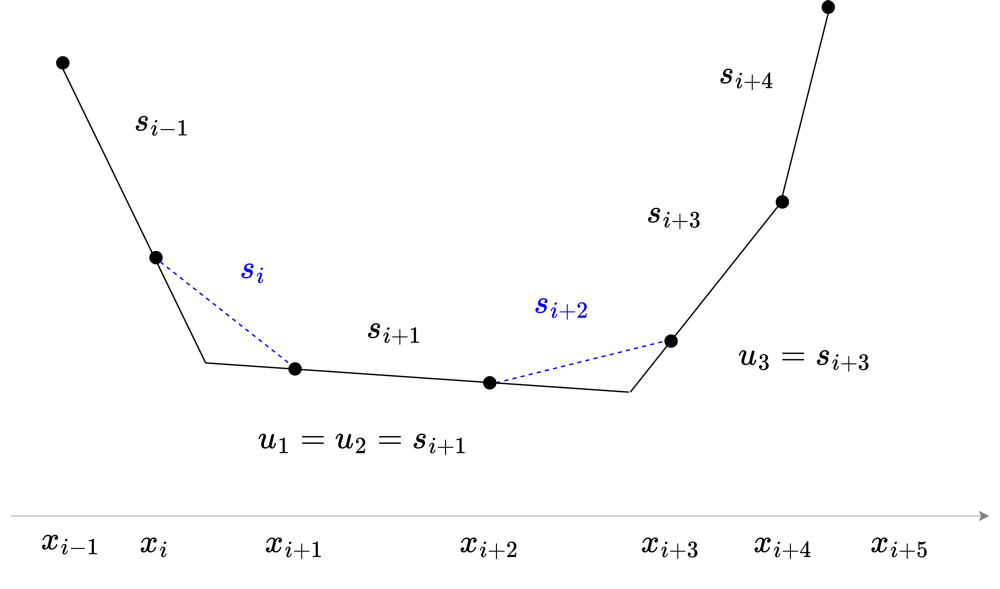}
          \caption{Another sparsest interpolant, corresponding to $\valpha = [1,0,1]$.}
          \label{fig:4b_m_4_sp2}
    \end{subfigure}
    \caption{Illustration of two sparsest interpolants in the scenario of \cref{th:geom_char},\ref{th:geom_char_2b} with $m = 4$. Both have $\lceil \frac{m+1}{2} \rceil = 3$ knots on $[x_i-1, x_{i+m}+1]$, consistent with Theorem 4 of \cite{debarre2022sparsest}.}
    \label{fig:4b_m_4_sp}
\end{figure}

\subsection{Multivariate results}
\subsubsection{Proof of \cref{prop:width_invariance_multivar}} \label{appendix:proof_width_invariance_multivar}
\begin{proof}
We address the statements individually.
    \paragraph{Existence of solutions to \eqref{opt:min_p_NN_multi} and \eqref{opt:min_0_NN_multi}.} Existence of solutions to \eqref{opt:min_0_NN_multi} is a simple consequence of the fact that interpolation is possible whenever $K \geq N$, so the feasible set of \eqref{opt:min_0_NN_multi} is non-empty, and objective values of \eqref{opt:min_0_NN_multi} lie in $\{1, \dots, K(d+1) \}$ on which a minimum is necessarily achieved.

    To show existence of solutions to \eqref{opt:min_p_NN_multi}, recall that by homogeneity of the ReLU, we can rescale the input and output weights of any neural network as $\vw_k \mapsto \alpha_k \vw_k$ and $v_k \mapsto \alpha_k^{-1} v_k$ for any $\alpha_k > 0$ without changing the network's represented function or its $\ell^p$ or $\ell^0$ path norms. Therefore, the optimal value of \eqref{opt:min_p_NN_multi} is equal to that of 
    \begin{align} \label{opt:min_p_NN_multi_w_constrained}
     \argmin_{\vtheta}  \sum_{k=1}^{K} |v_k|^p  \ , \ \mbox{subject to } f_{\vtheta}(\vx_i)=y_i, \,i=1,\dots,N, \ \| \vw_k \|_p = 1, \ k = 1, \dots, K
    \end{align}
    Picking an arbitrary feasible $\vtheta'$ for \eqref{opt:min_p_NN_multi_w_constrained}, any solution to \eqref{opt:min_p_NN_multi_w_constrained} must have
    \begin{align}
        \max_{k=1, \dots, K} |v_k|^p \leq \sum_{k=1}^K |v_k|^p \leq \sum_{k=1}^K |v_k'|^p \implies \max_{k=1, \dots, K} |v_k| \leq C := \left( \sum_{k=1}^K |v_k'|^p \right)^{1/p}
    \end{align}
    so we can further recast \eqref{opt:min_p_NN_multi_w_constrained} as
    \begin{align} \label{opt:min_p_NN_multi_w_v_constrained}
     \argmin_{\vtheta}  \sum_{k=1}^{K} |v_k|^p  \ , \ \mbox{subject to } f_{\vtheta}(\vx_i)=y_i, \,i=1,\dots,N, \ \| \vw_k \|_p = 1, \ |v_k| \leq C, \ k = 1, \dots, K
    \end{align}
    Because the sets $\{ \vw: \| \vw \|_p = 1 \}$ and $\{ v: |v| \leq C \}$ are compact, so is their Cartesian product $\{ \vtheta: \| \vw_k \|_p = 1, |v_k| \leq C, k = 1, \dots, K \}$. Since each $\vtheta \mapsto f_\vtheta(\vx_i)$ is continuous, the preimage of the singleton sets $\{ y_i \}$ under those maps are closed, and so is their finite intersection. As the intersection of a closed set with a compact set, the feasible set of \eqref{opt:min_p_NN_multi_w_v_constrained} is compact. Problem \eqref{opt:min_p_NN_multi_w_v_constrained} is therefore a minimization of a continuous function over a compact set, so it attains a solution by the Weierstrass extreme value theorem.

    \paragraph{Solutions to \eqref{opt:min_p_NN_multi} and \eqref{opt:min_0_NN_multi} have no more than $N$ active neurons.} Assume by contradiction that a solution $\{v_k, \vw_k\}_{k=1}^K$ to \eqref{opt:min_p_NN_multi} for $K > N$ has $K_0 > N$ active neurons $\{v_k, \vw_k\}_{k=1}^{K_0}$. Because $K_0 > N$, the vectors $\va_k := [(\vw_k^\top \overline{\vx}_1)_+, \dots, (\vw_k^\top \overline{\vx}_N)_+]^\top$, where $\overline{\vx}_i := [\vx_i, 1]$, are linearly dependent, meaning that there are constants $c_1, \dots, c_{K_0}$ (not all zero) for which $\sum_{k=1}^{K_0} c_k \va_k = \bm{0}$. Then for any real $t$:
    $$
    \sum_{k=1}^{K_0} (v_k + t c_k) \va_k = \sum_{k=1}^{K_0} v_k \va_k + t \sum_{k=1}^{K_0} c_k \va_k = \sum_{k=1}^{K_0} v_k \va_k = \vy
    $$
    where $\vy := [y_1, \dots, y_N]^\top$. In other words, the network with parameters $\{ v_k + t c_k, \vw_k \}_{k=1}^{K_0}$ interpolates the data, for any real $t$.

    In the case of \eqref{opt:min_p_NN_multi} for any $0 < p < 1$, choose $t > 0$ small enough that $\sgn(v_k + t c_k) = \sgn(v_k)$ for each $k$, and thus $\sgn((v_k + t c_k) w_{k,i}) = \sgn(v_k w_{k,i})$ for each $k,i$. Then by strict concavity of $t \mapsto |t|^p$ on $t \in (-\infty, 0)$ and $t \in (0, \infty)$ we have
    \begin{equation}
        |v_k w_{k,i}|^p = \left| \frac{(v_k + t c_k) w_{k,i} + (v_k - t c_k) w_{k,i}}{2} \right|^p > \frac{|(v_k + t c_k) w_{k,i}|^p + |(v_k - t c_k) w_{k,i}|^p}{2}
    \end{equation}
    for each $i$ and each $k$ with $c_k \neq 0$ (if $c_k = 0$ the above holds with equality). Since $c_k \neq 0$ for at least one $k$, this implies that
    \begin{equation}
        \sum_{k=1}^{K_0} \| v_k \vw_k \|_p^p > \frac{1}{2} \left( \sum_{k=1}^{K_0}  \| (v_k + t c_k) \vw_k \|_p^p + \sum_{k=1}^{K_0}  \| (v_k - t c_k) \vw_k \|_p^p   \right)
    \end{equation}
    but then at least one of $\sum_{k=1}^{K_0}  \| (v_k + t c_k) \vw_k \|_p^p$ or $\sum_{k=1}^{K_0}  \| (v_k - t c_k) \vw_k \|_p^p$ must be strictly less than $\sum_{k=1}^{K_0}  \| v_k \vw_k \|_p^p$. This contradicts optimality of $\{ v_k, \vw_k \}_{k=1}^K$.

    In the case of \eqref{opt:min_0_NN_multi}, choose $t = -v_{k'}/c_{k'}$ for one of the $c_{k'} \neq 0$. We then have
    \begin{equation}
        \sum_{k=1}^{K_0} \| (v_k + t c_k) \vw_k \|_0 < \sum_{k=1}^{K_0} \| v_k \vw_k \|_0
    \end{equation}
    strictly, because all of the $v_k \vw_k $ are nonzero, whereas at least one of the $(v_k + t c_k) \vw_k$ on the left is zero (for $k = k'$), and $\| v_k \vw_k \|_0 = \| (v_k + t c_k) \vw_k \|_0$ whenever both $v_k \vw_k $ and $(v_k + tc_k) \vw_k $ are nonzero. This again contradicts optimality of $\{ v_k, \vw_k \}_{k=1}^K$.

    \paragraph{Sparsity bound on solutions to \eqref{opt:min_0_NN_multi}.} If the data are in general position and $N \geq d+1$, then \cite{bubeck2020network} show that there exists an interpolating single-hidden-layer ReLU network with $4 \lceil N/d \rceil$ neurons. Any such network clearly has at most $4(d+1) \lceil N/d \rceil \leq 4(N+1) + 4(N+1)/d \leq 8 (N+1) = O(N)$ nonzero input weight/bias parameters across those $4 \lceil N/d \rceil$ neurons. 

     If the data are in general position and $N \leq d+1$, the points $\vx_1, \dots, \vx_N$ must be affinely independent, meaning that
    \begin{align}
        \sum_{i=1}^N \alpha_i \vx_i = \bm{0} \ \textrm{and}  \sum_{i=1}^N \alpha_i = \bm{0} \implies \alpha_1 = \dots = \alpha_N = 0
    \end{align}
    Because this condition is equivalent to linear independence of the vectors $\overline{\vx}_i := [\vx_i^\top, 1]^\top$, the general position assumption ensures that augmented data matrix $\overline{\mX} = [\overline{\vx}_1, \dots, \overline{\vx}_N]^\top \in \R^{N \times (d+1)}$ has full rank $N$. Therefore, there exists a solution $\vw \in \R^{d+1}$ to the system 
    \begin{align} \label{eq:hyperplane_interp_coeff}
        \overline{\mX} \vw = \vy := [y_1, \dots, y_N]^\top
    \end{align}
    with $\| \vw \|_0 = N$. (To see this, choose $N$ linearly independent columns of $\overline{\mX}$, express $\vy$ as a linear combination with respect to this basis, and let $\vw$ be the vector of coefficients of this linear combination.) For any such $\vw$, \eqref{eq:hyperplane_interp_coeff} says that the affine function
    \begin{align} \label{eq:affine_interpolant}
        f(\vx) = \vw^\top \overline{\vx} = (\vw^\top \overline{\vx})_+ - (-\vw^\top \overline{\vx})_+
    \end{align}
    interpolates the data (recall $\overline{\vx} := [\vx^\top, 1]^\top$). The term on the right is a two-neuron ReLU network with $\ell^0$ path norm of $2 \| \vw \|_0 = 2N$. Also note that if the labels $y_i$ are all nonnegative (resp. nonpositive), we may discard the second (resp. first) ReLU term in \eqref{eq:affine_interpolant}, achieving interpolation with $\ell^0$ path norm of $\| \vw \|_0 = N$. Thus in all cases, the $\ell^0$ path norm of any solution to \eqref{opt:min_0_NN_multi} is $O(N)$.
 \end{proof}
\subsubsection{Proof of \cref{lemma:multivar_reparam}} \label{appendix:proof_multivar_concave_convex}
\begin{proof}
We break the proof into the following steps.
\paragraph{Network output on data as a sum over activation patterns.} Note that the data-fitting constraint in problems \eqref{opt:min_p_NN_multi} and \eqref{opt:min_0_NN_multi} can be expressed in matrix form as 
\begin{equation}
    \sum_{k=1}^K v_k \left( \overline{\mX} \vw_k \right)_+ = \vy
\end{equation}
recalling that $\overline{\mX} = [\overline{\vx}_1, \dots, \overline{\vx}_N]^\top \in \R^{N \times (d+1)}$ is the matrix of augmented data points $\overline{\vx}_i = [\vx_i^\top, 1]^\top$, $\vy = [y_1, \dots, y_N]^\top \in \R^N$ is the vector of labels, and the ReLU $(\cdot)_+$ is applied element-wise. Also recalling that $\mD_1, \dots, \mD_J$ is the set of all possible $N \times N$ binary ``activation pattern'' matrices of the form $\diag(\mathbbm{1}[\overline{\mX} \vu \geq \bm{0}])$ for $\vu \in \R^{d+1}$, it must be the case that the matrix $\diag(\mathbbm{1}[\overline{\mX} \vw_k \geq \bm{0}])$ is among the $\mD_1, \dots, \mD_J$. For any $\vw_k$ whose corresponding activation pattern is $\mD_{\textrm{pattern}(k)}$, we have
$$
(\overline{\mX} \vw_k)_+  = \mD_{\textrm{pattern}(k)} \overline{\mX}  \vw_k \implies (\overline{\mX} \vw_k)_+ v_k = \mD_{\textrm{pattern}(k)} \overline{\mX} \tilde{\vw}_k
$$
where $\tilde{\vw}_k := v_k \vw_k$. 

For any $j = 1, \dots, J$, let $K_j = \{k: \textrm{pattern}(k) = j \}$ be the set of neuron indices which share the same pattern $\mD_j$. Then the sum of those neurons can be rewritten as

$$
\sum_{k \in K_j} (\overline{\mX} \vw_k)_+ v_k = \sum_{k \in K_j} \mD_j \overline{\mX} \tilde{\vw}_k = \mD_j \overline{\mX} \sum_{k \in K_j} \tilde{\vw}_k = \mD_j \overline{\mX} (\vnu_j -\vomega_j)
$$
where $\vnu_j$ and $\vomega_j$ represent the positive and negative parts of the aggregate vector $ \sum_{k \in K_j} \tilde{\vw}_k$, respectively, i.e.
$$
\vnu_j = \sum_{k \in K_j^+} v_k \vw_k, \qquad \vomega_j = -\sum_{k \in K_j^-} v_k \vw_k
$$
where $K_j^+ := \{ k \in K_j, v_k > 0 \}$ and $K_j^- := \{ k \in K_j, v_k < 0 \}$, so that
$$
\vnu_j - \vomega_j = \sum_{k \in K_j^+} v_k \vw_k + \sum_{k \in K_j^-} v_k \vw_k = \sum_{k \in K_j}  \tilde{\vw}_k
$$ 
Therefore, the entire network output can be written as
$$
\sum_{k=1}^K (\overline{\mX} \vw_k)_+ v_k = \sum_{j=1}^J \mD_j \overline{\mX} (\vnu_j - \vomega_j)
$$
with the understanding that, if the set $K_j$ is empty for some $j$, the vector $\vnu_j - \vomega_j := \sum_{k \in K_j} v_k \vw_k$ is the zero vector. 

\paragraph{Objective achieves its lower bound with two neurons per activation pattern.} Following the notation above, the objectives of \eqref{opt:min_p_NN_multi} and \eqref{opt:min_0_NN_multi} can be rewritten as:
$$
\sum_{k=1}^K \| v_k \vw_k \|_p^p = \sum_{k=1}^K \| \tilde{\vw}_{k} \|_p^p = \sum_{j=1}^J \left( \sum_{k \in K_j^+} \| \tilde{\vw}_{k} \|_p^p  + \sum_{k \in K_j^-} \| \tilde{\vw}_{k} \|_p^p \right)
$$
$$
\sum_{k=1}^K \| v_k \vw_k \|_0 = \sum_{k=1}^K \| \tilde{\vw}_{k} \|_0 = \sum_{j=1}^J \left( \sum_{k \in K_j^+} \| \tilde{\vw}_{k} \|_0  + \sum_{k \in K_j^-} \| \tilde{\vw}_{k} \|_0 \right)
$$
Observe that:
\begin{align} \label{eq:reparam_ineq}
\sum_{k \in K_j^+} \| \tilde{\vw}_{k} \|_p^p \geq \bigg\| \sum_{k \in K_j^+}  \tilde{\vw}_{k} \bigg\|_p^p = \| \vnu_{j} \|_p^p, \qquad \sum_{k \in K_j^-} \| \tilde{\vw}_{k} \|_p^p \geq \bigg\| \sum_{k \in K_j^-}  \tilde{\vw}_{k} \bigg\|_p^p = \| \vomega_{j} \|_p^p \\
\sum_{k \in K_j^+} \| \tilde{\vw}_{k} \|_0 \geq \bigg\| \sum_{k \in K_j^+}  \tilde{\vw}_{k} \bigg\|_0 = \| \vnu_{j} \|_0, \qquad \sum_{k \in K_j^-} \| \tilde{\vw}_{k} \|_0 \geq \bigg\| \sum_{k \in K_j^-}  \tilde{\vw}_{k} \bigg\|_0 = \| \vomega_{j} \|_0 
\end{align}
where in all cases, equality holds if and only if the supports of each vector in the sum (i.e., the set of indices at which each vector is nonzero) are disjoint. This follows from applying the inequality $(a+b)^p \leq a^p + b^p$---which holds for any $a, b \geq 0$ if $0 < p < 1$ and for any $a, b \in \R$ if $p = 0$ (defining $0^0 = 0$), and in both cases is strict unless $a = 0$ or $b = 0$---coordinate wise.

At a global minimizer of either \eqref{opt:min_p_NN_multi} or \eqref{opt:min_0_NN_multi}, this lower bound will be achieved. To see this, note that it is always possible to replace a single one of the vectors $\tilde{\vw}_k$ in each group $K_j^+$ (resp. $K_j^-$) with the vector $\vnu_j$ (resp. $-\vomega_j$), and set the remaining vectors in each group to zero. By definition $\vnu_j = \sum_{k \in K_j^+} \tilde{\vw}_k$ and $\vomega_j = -\sum_{k \in K_j^-} \tilde{\vw}_k$, so clearly the network output  $\sum_{j=1}^J \mD_j \overline{\mX} \left( \sum_{k \in K_j^+} \tilde{\vw}_k + \sum_{k \in K_j^-} \tilde{\vw}_k \right) = \sum_{j=1}^J \mD_j \overline{\mX}(\vnu_j - \vomega_j)$ on the data $\overline{\mX}$ remains unchanged by this modification. And with this modification, all inequalities in \eqref{eq:reparam_ineq} will clearly hold with equality. This shows that, for any solution to \eqref{opt:min_p_NN_multi} or \eqref{opt:min_0_NN_multi}, all input weight vectors $\vw_k$ in any individual activation pattern group $K_j^+$ or $K_j^-$ will have disjoint supports (which is the only circumstance under which the lower bounds in \eqref{eq:reparam_ineq} are achieved). In any such case, the neurons in each individual positive/negative activation pattern groups can be merged into a single nonzero neuron containing their sum, without affecting either the network's ability to interpolate the data or the value of the sums $\sum_{k \in K_j^+} \| \tilde{\vw}_k \|_0$ or $\sum_{k \in K_j^+} \| \tilde{\vw}_k \|_q^q$ for any $0 < q < 1$. Note that, although this merging may alter the function represented by the neural network, it will preserve the values of $\sum_{k=1}^K \| v_k \vw_k \|_0$ and $\sum_{k=1}^K \| v_k \vw_k \|_q^q$ for any $0 < q < 1$, which is the only thing required for the statement of the lemma and its subsequent use in proving \cref{th:multivar}. Therefore, we may enforce that there is at most one positively-weighted neuron $v_j^+ \vw_j^+ = \vnu_j$ and at most one negatively-weighted neuron $v_j^- \vw_j^- = \vomega_j$ corresponding to any possible activation pattern $j$ on the data. 

\paragraph{Constrain the variables $\vnu_j$ and $\vomega_j$ to correspond to ReLU activation patterns.} In order for a particular binary pattern $\mD_j$ to actually correspond to an input weight/bias $\vw_k$, it must be the case that $(\overline{\mX} \vw_k)_i \geq 0$ wherever $(\mD_j)_{ii} = 1$ and $(\overline{\mX} \vw_k)_i \leq 0$ wherever $(\mD_j)_{ii} = 0$. This is exactly the requirement that every entry of the vector $(2 \mD_j - \mI) \overline{\mX} \vw_k \in \mathbb{R}^N$ is nonnegative, since
$$
((2 \mD_j - \mI) \overline{\mX} \vw_k)_{i} = \begin{cases} (\overline{\mX} \vw_k)_i, &\textrm{if $(\mD_j)_{ii} = 1$} \\ -(\overline{\mX}\vw_k)_i, &\textrm{if $(\mD_j)_{ii} = 0$} \end{cases}
$$
When we re-parameterize as $\tilde{\vw}_k = v_k \vw_k$ and split the neuron indices $K_j$ correponding to activation pattern $\mD_j$ into the groups $K_j^+$ and $K_j^-$, the requirement that $(2 \mD_j - \mI) \overline{\mX} \vw_k \geq \bm{0}$ is equivalent to requiring that $(2 \mD_j - \mI) \overline{\mX} \tilde{\vw}_k \geq \bm{0}$ if $k \in K_j^+$ and $(2 \mD_j - \mI) \overline{\mX}\tilde{\vw}_k \leq \bm{0}$ if $k \in K_j^-$. Because we enforce that there is at most one nonzero neuron $\tilde{\vw}_k = \vnu_j$ (resp. $\tilde{\vw}_k= -\vomega_j$) in each activation pattern group $K_j^+$ (resp. $K_j^-$), this condition is also clearly equivalent to $(2 \mD_j - \mI) \overline{\mX} \vnu_j \geq \bm{0}$ and $(2 \mD_j - \mI) \overline{\mX} \vomega_j \geq \bm{0}$. 

\paragraph{Reconstruction of solutions to \eqref{opt:min_p_NN_multi} and \eqref{opt:min_0_NN_multi} from solutions to \eqref{opt:min_p_NN_reparam}.} By incorporating the above constraints, we have fully reparameterized the neural network problems \eqref{opt:min_p_NN_multi} and \eqref{opt:min_0_NN_multi} as claimed in the lemma.
Because we enforce that there is at most one nonzero neuron $\tilde{\vw}_k = \vnu_j$ (resp. $\tilde{\vw}_k= -\vomega_j$) in each activation pattern group $K_j^+$ (resp. $K_j^-$), solutions to problem \eqref{opt:min_p_NN_multi} can be recovered from solutions to \eqref{opt:min_p_NN_reparam} as
\begin{align}
    \{ \vw_k \}_{k=1}^K &= \left\{ \frac{\vnu_j}{\alpha_j}, \vnu_j \neq 0 \right\} \cup \left\{ \frac{\vomega_j}{\beta_j}, \vomega_j \neq 0 \right\} \\
    \{ v_k \}_{k=1}^K &= \left\{ \alpha_j, \vnu_j \neq 0 \right\} \cup \left\{ -\beta_j, \vomega_j \neq 0 \right\}
\end{align}
for any constants $\alpha_1, \beta_1 \dots, \alpha_J, \beta_J > 0$, the choice of which affects neither the network's represented function, nor its value of $\sum_{k=1}^K \| v_k \vw_k \|_0$ or $\sum_{k=1}^K  \| v_k \vw_k \|_q^q$ for any $0  < q < 1$. Note that, if there were a solution to \eqref{opt:min_p_NN_reparam} with $|\{ j: \vnu_j \neq 0 \}| + |\{ j: \vomega_j \neq 0 \}| > N$, this would yield a solution to \eqref{opt:min_p_NN_multi} or \eqref{opt:min_0_NN_multi} with $K = |\{ j: \vnu_j \neq 0 \}| + |\{ j: \vomega_j \neq 0 \}| > N$ active neurons, contradicting \cref{prop:width_invariance_multivar}. 
\end{proof}

\subsubsection{Proof of \cref{th:multivar}} \label{appendix:proof_multivar_th}
\begin{proof}

Problem \eqref{opt:min_p_NN_reparam} can be expressed more compactly in matrix form as
\begin{align} \label{opt:min_p_NN_reparam_mat}
    \argmin_{\vz \in \R^{2J
    (d+1)}} \| \vz \|_p^p\ , \ \mbox{subject to } \mA \vz = \vy, \ \mG \vz \geq \bm{0}
\end{align}
in the case $0 < p < 1$, or as
\begin{align} \label{opt:min_0_NN_reparam_mat}
    \argmin_{\vz \in \R^{2J
    (d+1)}} \| \vz \|_0\ , \ \mbox{subject to } \mA \vz = \vy, \ \mG \vz \geq \bm{0}
\end{align}
in the case $p = 0$, where
\begin{align}
    \vz &:= [\vnu_1^\top, \vomega_1^\top, \dots, \vnu_J^\top, \vomega_J^\top]^\top \in \R^{2J(d+1)} \\
    \mA &:= [\mD_1 \overline{\mX}, -\mD_1 \overline{\mX}, \dots,  \mD_J \overline{\mX}, -\mD_J \overline{\mX}] \in \R^{N \times 2J(d+1)} \\
    \mG &:= \diag \left( (2 \mD_1 - \mI) \overline{\mX}, (2 \mD_1 - \mI) \overline{\mX}, \dots, (2 \mD_J - \mI) \overline{\mX}, (2 \mD_J - \mI) \overline{\mX} \right) \in \R^{2JN \times 2J(d+1)}
\end{align}
We proceed in the following steps, which employ arguments similar to those of \cite{yang2022sparse} (Theorem 2.1) and \cite{peng2015np} (Theorem 1), with minor modifications to account for the inequality constraint in \eqref{opt:min_p_NN_reparam_mat}. We note that the justification of $p$-independent $\ell^\infty$ boundedness of solutions given in \cite{peng2015np} appears to be incorrect, with \cite{yang2022sparse} presenting the correct justification that we follow here. 

\paragraph{Solutions to \eqref{opt:min_p_NN_reparam_mat} for any $0 < p < 1$ are contained in an $\ell^\infty$ ball of $p$-independent radius $C$.}

We let $\supp(\vu)$ denote the set of nonzero indices of a vector $\vu$. $\mM_S$ denotes the submatrix formed by restricting its columns to an index set $S$, and $\mM_{I,S}$ denotes restriction of the rows to an index set $I$ and columns to an index set $S$. 

Let $\vz^*$ be a solution to \eqref{opt:min_p_NN_reparam_mat} for arbitrary $0 < p < 1$. Let $S = \supp(\vz^*)$. Let $I = \{i: (\mG \vz^*)_i = 0 \}$. We begin by showing that the matrix
\begin{align}
    \widetilde{\mA}_{I,S} := \begin{bmatrix}
    \mA_S \\ \mG_{I,S} 
    \end{bmatrix} \in \R^{(N+|I|) \times |S|}
\end{align}
has full column rank $|S|$. Assume by contradiction that $\rank(\widetilde{\mA}_{I,S}) < |S|$, and therefore $\widetilde{\mA}_{I,S} \vc_S = \bm{0}$ for some nonzero $\vc_S \in \R^{|S|}$. Extending $\vc_S$ to a vector $\vc \in \R^{2J(d+1)}$ by zero-padding, we thus have $\mA \vc = \widetilde{\mA}_{I,S} \vc_S = \bm{0}$ and therefore $\mA (\vz^* \pm t \vc) = \mA \vz^* = \vy$ for any $t \in \R$. Similarly, $\mG_I \vc = \mG_{I,S} \vc_S = \bm{0}$, and therefore $\mG_I (\vz^* \pm t \vc) = \mG_I \vz^* = \bm{0}$ for any $t \in \R$. If $t > 0$ is chosen small enough that $\sgn((\mG(\vz^* \pm t \vc))_i ) = \sgn((\mG \vz^*)_i)$ for $i \notin I$, we will thus have $\mG(\vz^* \pm t \vc) \geq \bm{0}$, so that $\vz^* \pm t \vc$ are both feasible for \eqref{opt:min_p_NN_reparam_mat}. 
    
Now choose $t$ small enough that, in addition to the previous sign requirement involving $\mG$, we also have $\sgn(z^*_i \pm t  c_i) = \sgn(z_i^*)$ for each $i \in S$. By strict concavity of $t \mapsto |t|^p$ on $t \in (-\infty, 0)$ and $t \in (0, \infty)$ we have
\begin{equation}
    |z_i^*|^p > \frac{|z_i^* + t c_i|^p + |z_i^* - t c_i|^p}{2}
\end{equation}
for each $i \in S$ with $c_i \neq 0$ (if $c_i = 0$ the above holds with equality). Since at least one of the $c_i \neq 0$, this implies that
\begin{equation}
    \| \vz^* \|_p^p > \frac{\| \vz^* + t \vc \|_p^p + \| \vz^* - t \vc \|_p^p}{2}
\end{equation}
strictly. But then at least one of $\| \vz^* + \vc \|_p^p < \| \vz^* \|_p^p$ or $\| \vz^* - \vc \|_p^p < \| \vz^* \|_p^p$ holds strictly. Because $\vz^* \pm t \vc$ are both feasible for \eqref{opt:min_p_NN_reparam_mat}, this contradicts optimality of $\vz^*$.

Having shown that $\widetilde{\mA}_{I,S}$ is full column rank, the rank-nullity theorem implies that $\ker(\widetilde{\mA}_{I,S}) = \{ \bm{0} \}$. Therefore $\widetilde{\mA}_{I,S}$ is injective, and thus 
$\vz_S^*$ is the \textit{unique} solution to
\begin{equation}
    \widetilde{\mA}_{I,S} \vz = \begin{bmatrix}
        \vy \\ \bm{0} 
    \end{bmatrix} =: \widetilde{\vy}
\end{equation}
This implies that $\vz^*$ lies in the finite set
\begin{equation}
    Z = \left\{ \vz \ \big\rvert \ I \subset \{1, \dots, 2JN \}, S \subset \{1, \dots,  2J(d+1) \}, \rank(\widetilde{\mA}_{I,S}) = |S|, \widetilde{\mA}_{I,S} \vz_S = \widetilde{\vy}, \vz_{S^c} = \bm{0} \right\}
\end{equation}
which clearly depends only on $\mA, \mG, \vy$ and not on $p$. Therefore $\| \vz^* \|_\infty \leq C := \max_{\vz \in Z} \| \vz \|_\infty < \infty$, where $C$ is independent of $p$.

\paragraph{Projection of the feasible set into the positive orthant.} 

Define $R_0 := \| \vz_0 \|_\infty$ for an arbitrary solution $\vz_0$ to \eqref{opt:min_0_NN_reparam_mat} with $p = 0$. Let $R := \max\{C, R_0 \}$ for the $C$ defined above. The set
\begin{equation}
    \Omega := \{ \vz \in \R^{2J(d+1)} \ \rvert \ \mA \vz = \vy, \mG \vz \geq \bm{0}, \| z \|_\infty \leq R \}
\end{equation}
is a polytope. As shown above, any solution to \eqref{opt:min_p_NN_reparam_mat} for any $0 < p < 1$ is attained on $\Omega$, and by definition, at least one solution to \eqref{opt:min_0_NN_reparam_mat} for $p = 0$ is attained on $\Omega$.

The map $\vz \mapsto \| \vz \|_p^p$ is not concave on all of $\R^{2J(d+1)}$, but it is strictly concave on each individual orthant, so to apply the Bauer maximum principle as in the proof of \cref{th:main}, we will relate \eqref{opt:min_p_NN_reparam_mat} to an optimization over a projection of the polytope $\Omega$ to the nonnegative orthant $\R_+^{2J(d+1)}$. To do so, note that the set
\begin{align}
    \Psi := \left\{ (\vz, \vz') \in \R^{2J(d+1)} \times \R^{2J(d+1)}_{+} \  \big \rvert \  \vz \in \Omega, \| \vz' \|_{\infty} \leq R, |\vz| \leq \vz' \right\},
\end{align}
is a polytope in the product space $\R^{2J(d+1)} \times \R^{2J(d+1)}_{+}$. (Here the \textit{module vector} $|\vz|$ is the vector of absolute values of entries of $\vz$.) Because the coordinate projection of a polytope is a polytope (\cite{goemans20093}), the set
\begin{align}
    \Omega' := \left\{ \vz' \in \R^{2J(d+1)}_{+} \ \big\rvert \ \| \vz' \|_{\infty} \leq R, \ \exists \  \vz \in \Omega \ \textrm{s.t.} \  |\vz| \leq \vz' \right\},
\end{align}
which is given by the coordinate projection of $\Psi$ onto the $\vz'$ coordinate, is a polytope in $\R^{2J(d+1)}_{+}$. Furthermore, $\min_{\vz \in \Omega} \| \vz \|_p^p = \min_{\vz' \in \Omega'} \| \vz' \|_p^p$. To see this, note that for any $\vz \in \Omega$, its module vector $|\vz| \in \Omega'$, so $\min_{\vz \in \Omega} \| \vz \|_p^p \geq \min_{\vz' \in \Omega'} \| \vz' \|_p^p$. If that inequality were strict, then there would be some $\vz \in \Omega$ with $|\vz| < \vz_*' \ni  \argmin_{\vz' \in \Omega'} \| \vz' \|_p^p$, but this would imply that $\min_{\vz \in \Omega} \| \vz \|_p^p < \min_{\vz' \in \Omega} \| \vz' \|_p^p$.

As a polytope, $\Omega'$ is compact, convex, and has finitely many extreme points, the set of which we denote $\textrm{Ext}(\Omega')$. Let
\begin{align}
    r := \min \{z_i' > 0 \ \rvert \  \vz' = [z_1', \dots, z_{2J(d+1)}']^\top \in \textrm{Ext}(\Omega') \}
\end{align}
be the smallest nonzero coordinate in any of the extreme points of $\Omega'$.

Next, note that for $0 < p < 1$, the objective $\vz \mapsto \| \vz \|_p^p$ is continuous and strictly concave on the nonnegative orthant $\R_{+}^{2J(d+1)}$, and thus on $\Omega'$. Therefore, by the Bauer maximum principle (\cite{aliprantis2006infinite}, Theorem 4.104), a solution to $\argmin_{\vz' \in \Omega'} \| \vz' \|_p^p$ exists at an extreme point of $\Omega'$. In particular, by strict concavity of $\vz \mapsto \| \vz \|_p^p$, \textit{any} solution to $\argmin_{\vz' \in \Omega'} \| \vz' \|_p^p$ must be at an extreme point of $\Omega'$. (Otherwise, if such a solution had $\vz' = t \va' + (1-t) \vb'$ for distinct $\va', \vb' \in \Omega'$ and $t \in (0,1)$, then $\| \vz' \|_p^p > t \| \va' \|_p^p + (1-t) \| \vb' \|_p^p \geq t \| \vz' \|_p^p + (1-t) \| \vz' \|_p^p = \| \vz' \|_p^p$ which is impossible.) 

\paragraph{Sparse recovery result.} Putting everything together, fix an arbitrary $0 < p < 1$ and let $\vz_p$ be a solution to \eqref{opt:min_p_NN_reparam_mat} for that $p$. The previous paragraph shows that $|\vz_p|$ is a solution to $\argmin_{\vz' \in \Omega'} \| \vz' \|_p^p$, and therefore $|\vz_p| \in \textrm{Ext}(\Omega')$. Then: 
\begin{align}
    \| \vz_p \|_0 &= \| r^{-1} |\vz_p| \|_0 = \lim_{q \downarrow 0} \sum_{i=1}^J \left( \frac{|z_{p,i}|}{r} \right)^q \\
    &\leq \sum_{i=1}^J \left( \frac{|z_{p,i}|}{r} \right)^p = r^{-p} \min_{\vz' \in \Omega'} \| \vz' \|_p^p = r^{-p} \min_{\vz \in \Omega} \| \vz \|_p^p = \left(\frac{R}{r} \right)^p \min_{\vz \in \Omega} \| R^{-1} \vz \|_p^p \\
    &\leq \left(\frac{R}{r} \right)^p \min_{\vz \in \Omega} \| R^{-1} \vz \|_0 = \left(\frac{R}{r} \right)^p \min_{\vz \in \Omega} \|  \vz \|_0
\end{align}
where the inequalities come from the fact that $p \mapsto x^p$ is decreasing for $x \in (0,1)$ and increasing for $x > 1$. Because $\| \vz \|_0$ is a positive integer for any $\vz$, and \eqref{opt:min_0_NN_reparam_mat} attains at least one solution on $\Omega$, the above shows that $\vz_p$ solves \eqref{opt:min_0_NN_reparam_mat} for any $p$ satisfying
\begin{align} \label{eq:R_r_ineq}
    \left(\frac{R}{r} \right)^p \min_{\vz \in \Omega} \|  \vz \|_0 &< \min_{\vz \in \Omega} \|  \vz \|_0 + 1 \\
    \iff p &<  \frac{\log (\min_{\vz \in \Omega} \| \vz \|_0 + 1) - \log(\min_{\vz \in\Omega} \| \vz \|_0)}{\log R - \log r} \label{eq:p_ineq}
\end{align}
if $r < R$, or for any $0 < p < 1$ if $r = R$. (Note that by definition of $\Omega'$, $r \leq R$ always.)

Let $\vtheta_0$ be a solution to \eqref{opt:min_0_NN_multi} and $\vtheta_p$ be a solution to \eqref{opt:min_p_NN_multi} for any $p$ which obeys the inequality in \eqref{eq:p_ineq}, and let $\vtheta_0'$ and $\vtheta_p'$ be the corresponding solutions---constructed from solutions $\vz_p$ and $\vz_0$ to \eqref{opt:min_p_NN_reparam_mat} and \eqref{opt:min_0_NN_reparam_mat}, respectively---as stated in \cref{lemma:multivar_reparam}. We have shown that
\begin{align}
    \| \vtheta_p \|_0 = \| \vtheta_p' \|_0 = \| \vz_p \|_0 = \| \vz_0 \|_0 = \| \vtheta_0' \|_0 = \| \vtheta_0 \|_0
\end{align}
which proves the result.
\end{proof}

\subsection{Experiments} 
All code for the experiments can be found at \url{https://github.com/julianakhleh/sparse_nns_lp}.

\label{appendix:experiments}

\subsubsection{Reweighted $\ell^1$ algorithm} \label{appendix:rw_l1}
To implement our proposed $\ell^p$ path norm regularizer, we use the iteratively reweighted $\ell^1$ algorithm of \cite{candes2008enhancing,figueiredo2007majorization}, which we summarize informally here. The principal motivation is the inequality
\begin{equation}
    |x|^p \leq |x|p |y|^{p-1} + (1-p) |y|^p
\end{equation}
which holds for all $x \in \R$, all $y \in \R \setminus \{0\}$, and all $0 < p \leq 1$, with equality when $p = 1$ and/or when $x = y$. Applied to $x = |v_k w_{k,i}|$, we have
\begin{align} \label{eq:lp_path_norm_upper_bound}
    \sum_{k=1}^K \| v_k \vw_k \|_p^p = \sum_{k=1}^K \sum_{i=1}^d |v_k w_{k,i}|^p \leq \sum_{k=1}^K \sum_{i=1}^d \left( |v_k w_{k,i}| p |y_{k,i}|^{p-1} + (1-p) |y_{k,i}|^p \right)
\end{align}
for any choice of constant $y_{k,i} \in \R \setminus \{0 \}$. The iteratively reweighted $\ell^1$ algorithm attempts to minimize the $\ell^p$ path norm objective on the left hand side of \eqref{eq:lp_path_norm_upper_bound} by minimizing its upper bound on the right. Because the choice of $v_k, w_{k,i}$ which minimizes this upper bound is invariant to the additive constant $(1-p) |y_{k,i}|^p$ term, we can equivalently choose $v_k w_{k,i}$ at each iteration $t$ to minimize only the first term $C_{k,i} |v_k w_{k,i}|$ where $C_{k,i} := p |y_{k,i}|^{p-1}$. Because the upper bound is tighter when $y_{k,i}$ is closer to the optimal values of $v_k w_{k,i}$ for this iteration $t$, we choose the constants $y_{k,i}$ as $v_k^{(t-1)} w_{k,i}^{(t-1)}$, where $v_k^{(t-1)}, w_{k,i}^{(t-1)}$ are the previous iterates. The regularization penalty thus becomes
\begin{align}
    \sum_{k=1}^K \sum_{i=1}^d C_{k,i} |v_k w_{k,i}|
\end{align}
which is simply a separable weighted $\ell^1$ penalty with weights $C_{k,i}$. This objective lends itself to a standard $\ell^1$ proximal gradient update algorithm, with each soft-thresholding step scaled appropriately according to the individual threshold $C_{k,i}$. The full algorithm is summarized in \cref{alg:rw_l1}. 
\begin{algorithm}
    \caption{Iteratively reweighted $\ell^1$ algorithm for $\ell^p$ path norm minimization} \label{alg:rw_l1}
    \textbf{Input}: loss function $\cL$, sparsity parameter $0 < p \leq 1$, learning rate $\gamma > 0$, regularization parameter $\lambda > 0$, total number of iterations $T$. \\
    \begin{algorithmic}
        \For{$t=1, \dots, T$}
        
        \State Compute thresholds: $C_{k,i} \gets \lambda p |v_k^{(t-1)} w_{k,i}^{(t-1)}|^{p-1}$
        \State Gradient update for input weights: $\tilde{w}_{k,i} \gets w_{k,i}^{(t-1)} - \lambda \frac{\partial \cL(\vtheta)}{\partial w_{k,i}} \big \rvert_{w_{k,i}^{(t-1)}}$
        \State Gradient update for output weights: $\tilde{v}_{k} \gets v_{k}^{(t-1)} - \lambda \frac{\partial \cL(\vtheta)}{\partial v_{k}} \big \rvert_{v_{k}^{(t-1)}} $
        \State Reweighted $\ell^1$ prox update:
        $u_{k,i} \gets \textrm{Prox}_{C_{k,i} | \cdot |} = \sgn(\tilde{v}_k \tilde{w}_{k,i}) (|\tilde{v}_k \tilde{w}_{k,i}| - C_{k,i})_+$
        \State Update input weights: $w_{k,i}^{(t)} \gets \sgn(\tilde{w}_{k,i}^{(t)}) \frac{u_{k,i}}{\sqrt{\| \vu_k \|_2 }}$ 
        \State Update output weights: $v_{k}^{(t)} \gets \sgn(\tilde{v}_{k}^{(t)}) \sqrt{\| \vu_k \|_2}$  \Comment{satisfies $u_k = v_k^{(t)} w_{k,i}^{(t)}$ }
        \EndFor
    \end{algorithmic}
\end{algorithm}

We note that there are infinitely many ways to choose the updated input/output weights $w_{k,i}^{(t)}$ and $v_k^{(t)}$  to satisfy $u_k = v_k^{(t)} w_{k,i}^{(t)}$; due to homogeneity of the ReLU (meaning that $(\alpha x)_+ = \alpha(x)_+$ for any $\alpha \geq 0$), any choice $w_{k,i}^{(t)} \leftarrow \alpha u_{k,i}$ and $v_k^{(t)} \leftarrow 1/\alpha$ for any $\alpha > 0$ would satisfy $u_k = v_k^{(t)} w_{k,i}^{(t)}$ and produce the same neural network function. The particular choice described in \cref{alg:rw_l1} additionally satisfies the \textit{balancing} constraint $\| \vw_k^{(t)} \|_2 = |v_k^{(t)}|$, and we find that this selection tends to perform best in practice. We also note that, for univariate input dimension $d = 1$ and sparsity parameter $p = 1$, \cref{alg:rw_l1} is equivalent to the PathProx algorithm of \cite{yang2022pathprox}.

\subsubsection{Setup and results} \label{appendix:experimental_setup_results}

We test our algorithm on two simple synthetic datasets. The first is a univariate ``peak/plateau'' dataset, which consists of the data/label pairs: 
\begin{align} \label{eq:univariate_peak_plateau}
    (-2, 0), (-1, 0), (0, 1), (1, 1), (2, 0), (3, 0)
\end{align}
For this dataset, the theory of \cite{debarre2022sparsest} shows that the sparsest interpolant $f$ is unique, and is represented using 3 ReLU neurons as
\begin{align} \label{eq:f_sparsest_peak_plateau}
    f(x) = (x+1)_+ - 2 (x-1/2)_+ + (x-2)_+
\end{align}
Our theory in \cref{sec:univariate} also shows that this $f$ is a global $\ell^p$-path norm minimizer for any $0 < p \leq 1$, and is the unique such minimizer for any $0 < p < 1$.

\cref{fig:sparsity_over_time_univariate} shows the sparsity over time of our reweighted $\ell^1$ algorithm for three different values of $p \in \{0.4, 0.7, 1 \}$, implemented in PyTorch using the Adam optimizer, along with that of Adam-only (no regularization) and AdamW weight decay. All networks share the same random initialization and are trained with MSE loss for 100,000 epochs with learning rate $\gamma = 0.01$, regularization parameter $\lambda = 0.003$ (except for unregularized Adam-only, which uses $\lambda = 0$), and hidden layer width $K = 80$. All three values of $p$ in our reweighted $\ell^1$ algorithm produce vastly sparser solutions earlier on in training than both Adam-only and AdamW; however, only $p = 0.4$ eventually recovers the true sparsest solution $f$ with 3 ReLU neurons (see \cref{fig:functions_p1_noreg_WD}).

\begin{figure}
    \centering
    \includegraphics[width=\linewidth]{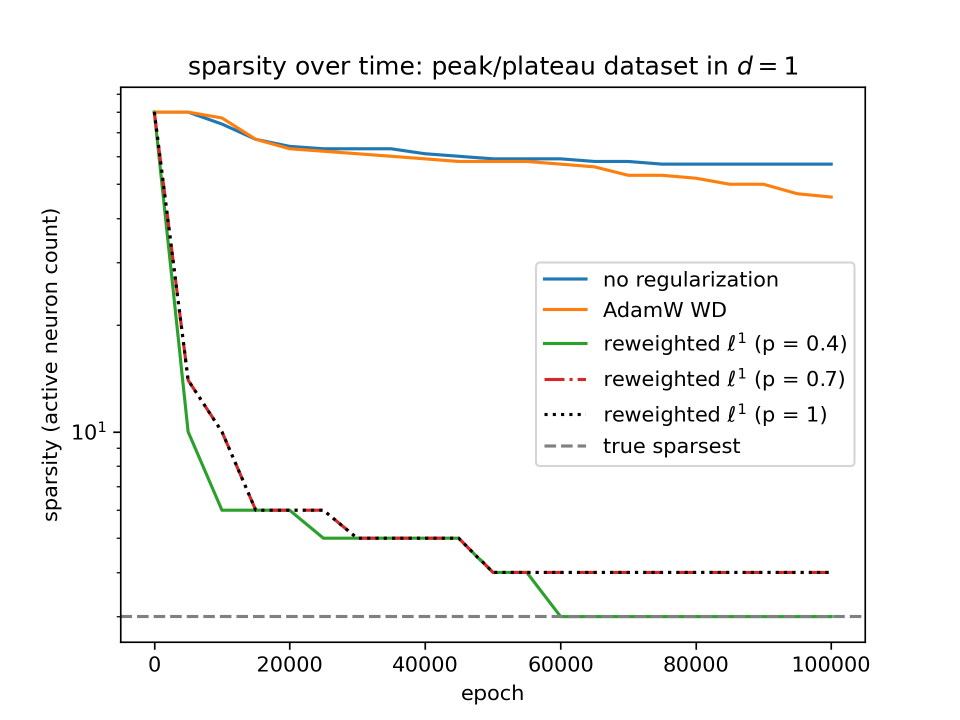}
    \caption{Sparsity over time of five networks trained to interpolation on the univariate peak-plateau dataset \eqref{eq:univariate_peak_plateau}. The reweighted $\ell^1$ algorithm for $\ell^p$ path norm minimization (\cref{alg:rw_l1}) recovers much sparser solutions earlier in training than unregularized Adam or AdamW weight decay regularization, with the smallest value $p = 0.4$ eventually recovering the sparsest possible interpolant \eqref{eq:f_sparsest_peak_plateau}.}
    \label{fig:sparsity_over_time_univariate}
\end{figure}

\begin{figure}
    \centering
    \includegraphics[width=\linewidth]{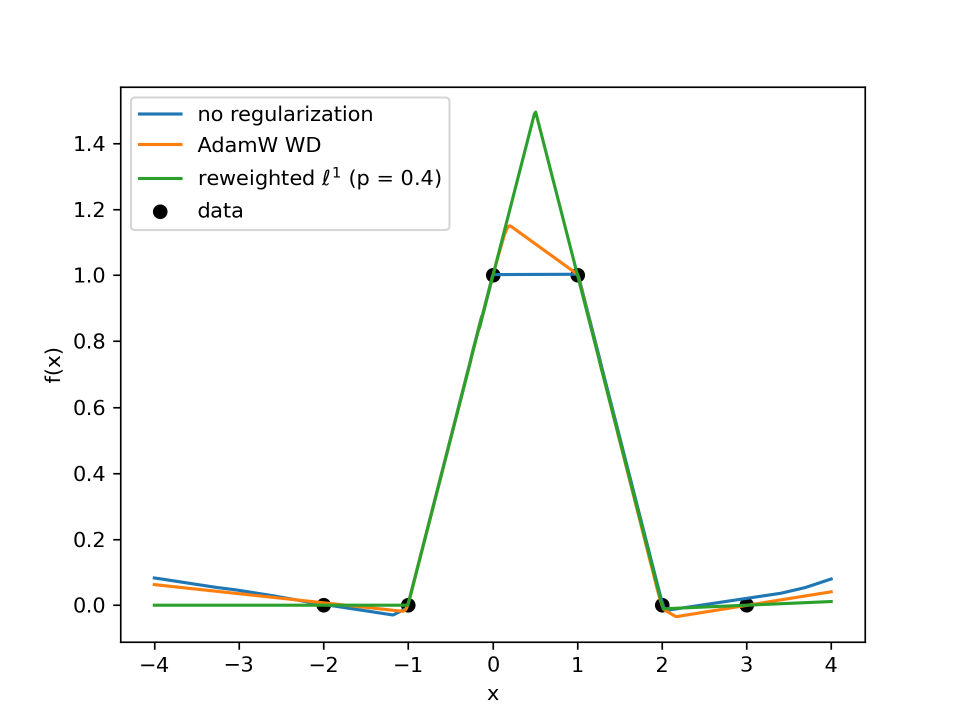}
    \caption{Three interpolants of the peak-plateau dataset, learned after 100,000 epochs using unregularized Adam, AdamW weight decay, and reweighted $\ell^1$ (\cref{alg:rw_l1}) with $p = 0.4$. Only the latter recovers the true sparsest interpolant \eqref{eq:f_sparsest_peak_plateau}.}
    \label{fig:functions_p1_noreg_WD}
\end{figure}

\cref{fig:functions_over_time} shows the functions learned by all five networks throughout the course of training. We see that reweighted $\ell^1$ with $p \in \{0.4, 0.7, 1 \}$ all converge quickly to near-sparsest solutions, and then the small additional kinks inside $[0,1]$ disappear gradually throughout training, with only $p = 0.4$ eliminating them completely (the final solutions for $p \in \{0.7, 1 \}$ have a single extraneous active neuron of small magnitude which activates just before $x = 1/2$).

\begin{figure}
    \centering
    \begin{subfigure}{0.5\textwidth}
          \centering
          \includegraphics[width=\linewidth]{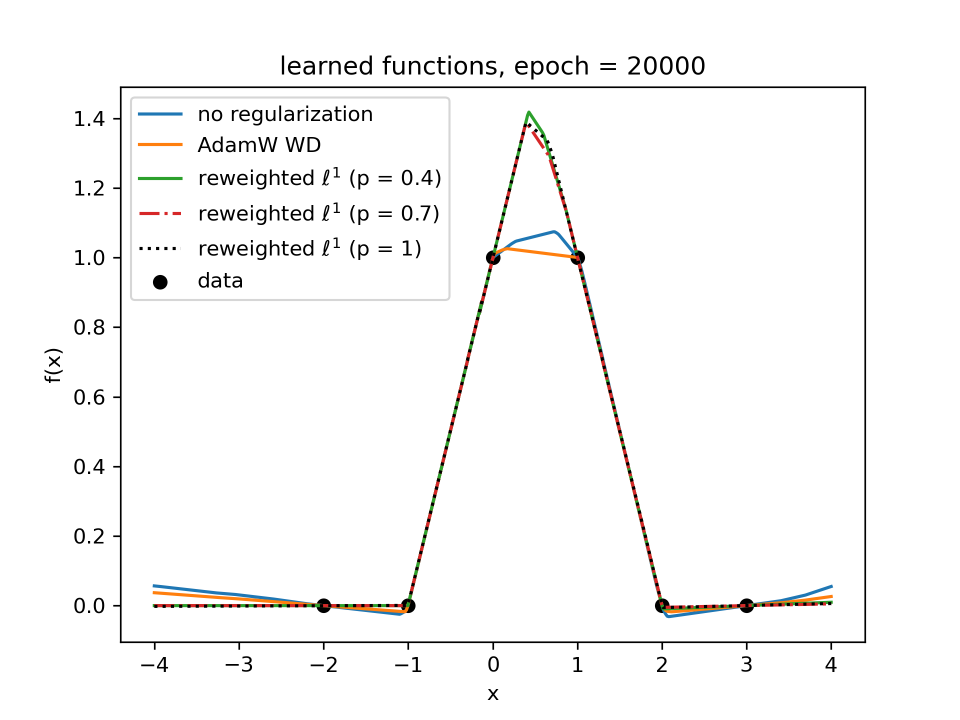}
          \label{fig:functions_over_time_20000}
    \end{subfigure}%
    \begin{subfigure}{0.5\textwidth}
          \centering
          \includegraphics[width=\linewidth]{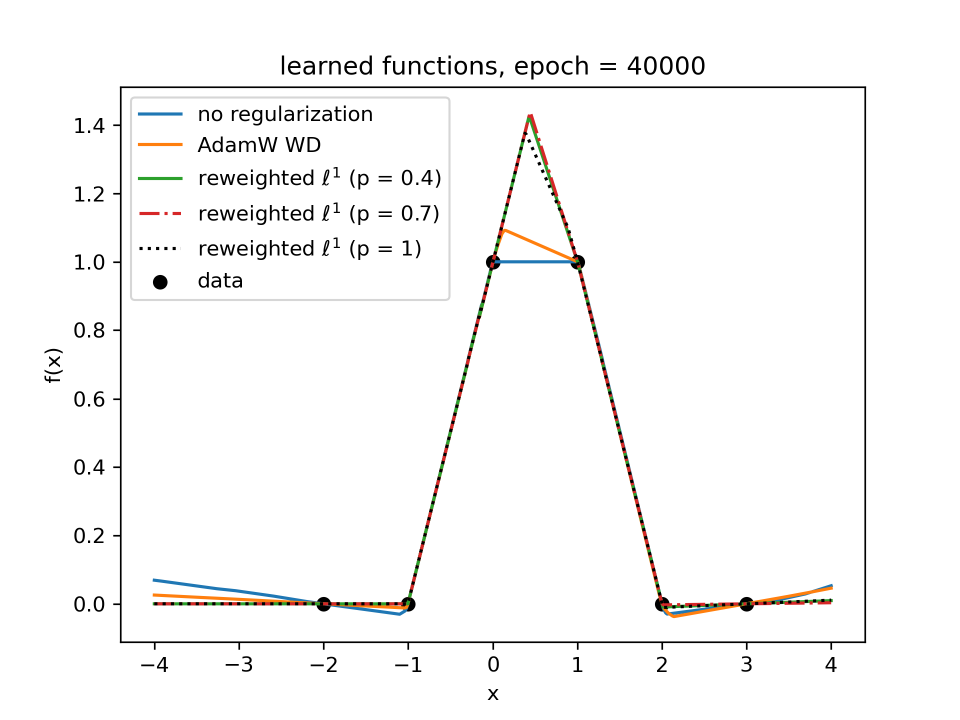}
          \label{fig:function_over_time_40000}
    \end{subfigure}
    \begin{subfigure}{0.5\textwidth}
          \centering
          \includegraphics[width=\linewidth]{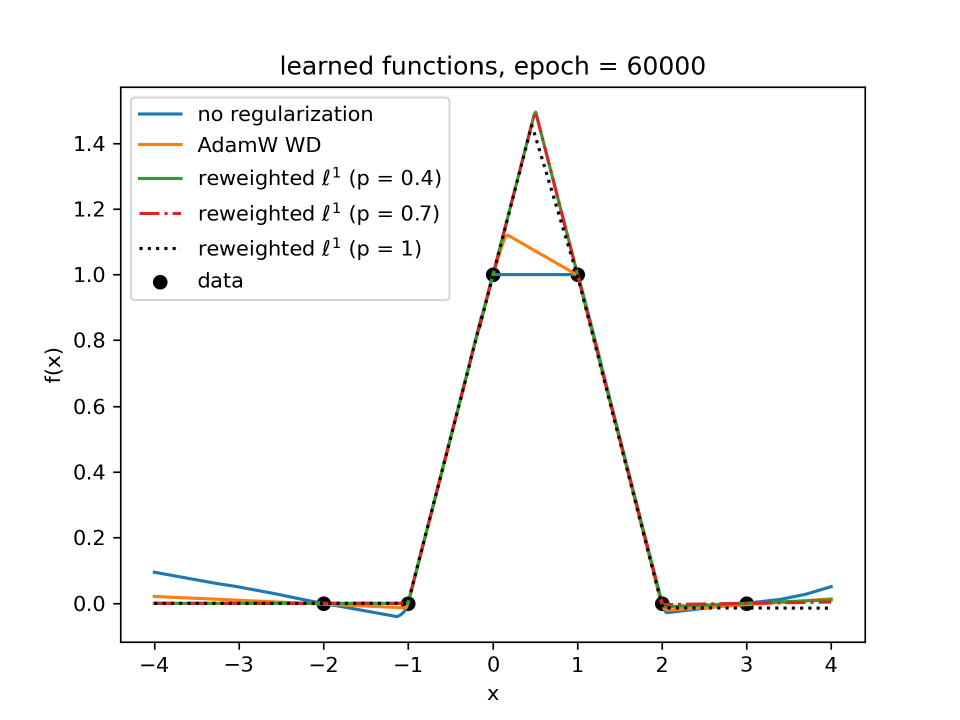}
          \label{fig:functions_over_time_60000}
    \end{subfigure}%
    \begin{subfigure}{0.5\textwidth}
          \centering
          \includegraphics[width=\linewidth]{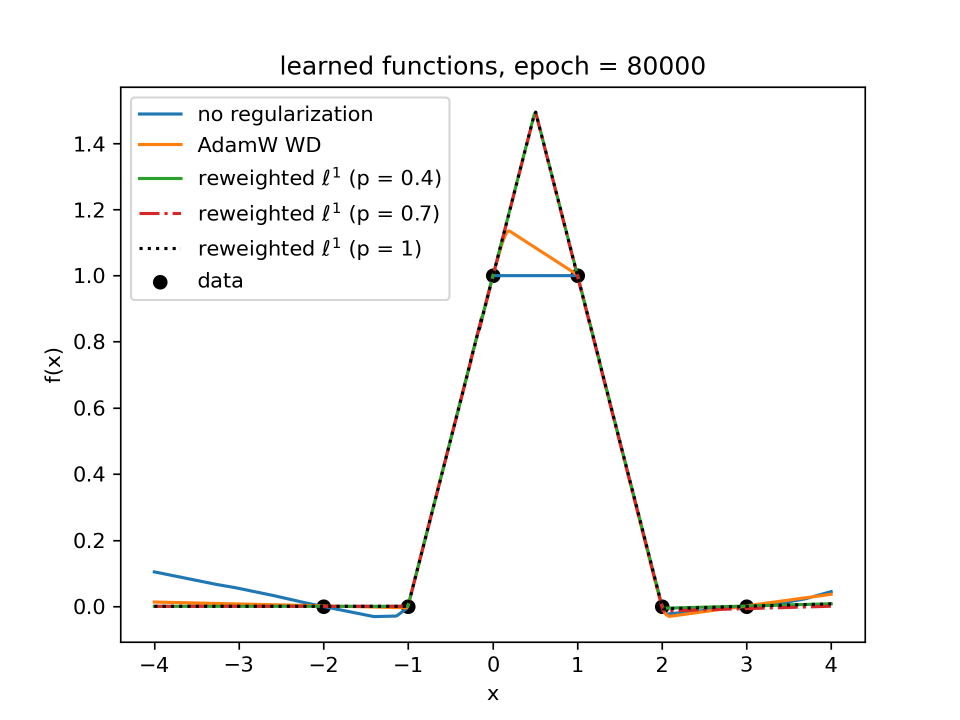}
          \label{fig:functions_over_time_80000}
    \end{subfigure}
    \begin{subfigure}{0.5\textwidth}
          \centering
          \includegraphics[width=\linewidth]{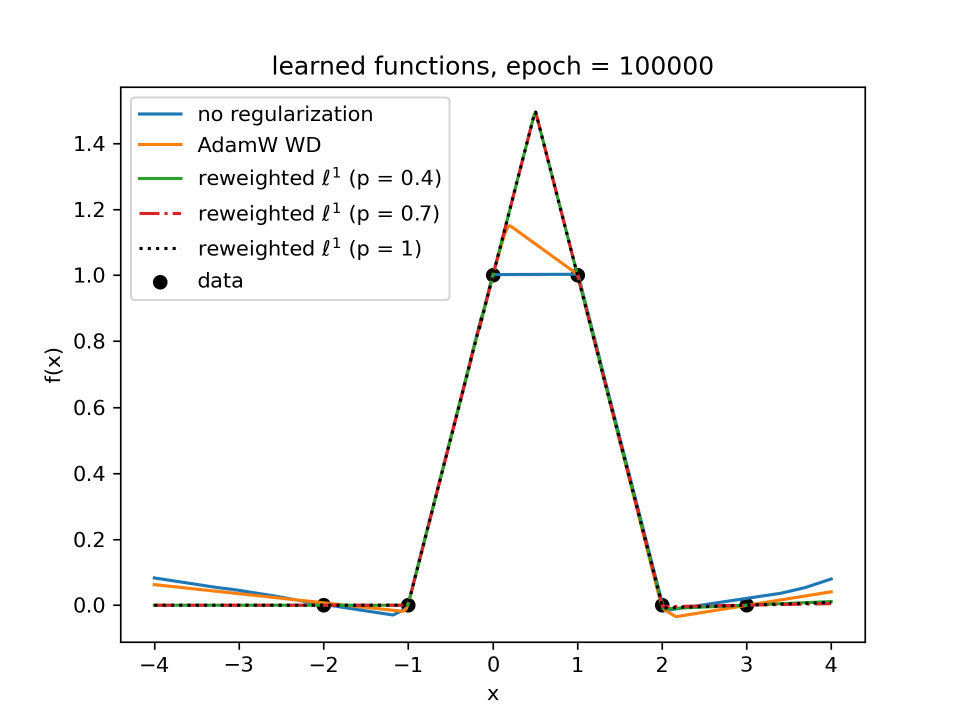}
          \label{fig:functions_over_time_100000}
    \end{subfigure}
    \caption{Learned network functions of five different algorithms throughout the course of training. Reweighted $\ell^1$ with $p \in \{0.4, 0.7, 1 \}$ converge to near-sparsest solutions early on in training, with only $p = 0.4$ eventually eliminating all extraneous neurons to recover the true sparsest solution \eqref{eq:f_sparsest_peak_plateau}.}
    \label{fig:functions_over_time}
\end{figure}

For our second experiment, we consider $N = 10$ data points in $d = 50$ dimensions. The coordinates of each data $\vx_i$ point are drawn i.i.d. from $\textrm{Unif}[-1,1]$, as are the labels $y_i$. As in the univariate case, we compare the sparsity over time of our reweighted $\ell^1$ algorithm for $p \in \{0.4, 0.7, 1 \}$, implemented in PyTorch using the Adam optimizer, against that of Adam-only (no explicit regularization) and AdamW weight decay. All networks are trained using MSE loss for 100,000 epochs with learning rate $\gamma = 0.01$, regularization parameter $\lambda = 0.005$ (except for unregularized Adam-only, which uses $\lambda = 0$), and hidden layer width $K = 100$. \cref{fig:sparsity_over_time_high_d} shows that all values of $p$ produce much sparser solutions than Adam-only and AdamW weight decay, with $p = 0.4$ producing sparser solutions than $p \in \{0.7, 1 \}$. The solutions recovered by $p \in \{0.4, 0.7, 1 \}$ all obey the sparsity upper bound of $2N$ guaranteed by the proof of \cref{prop:width_invariance_multivar}.

\begin{figure}
    \centering
    \includegraphics[width=\linewidth]{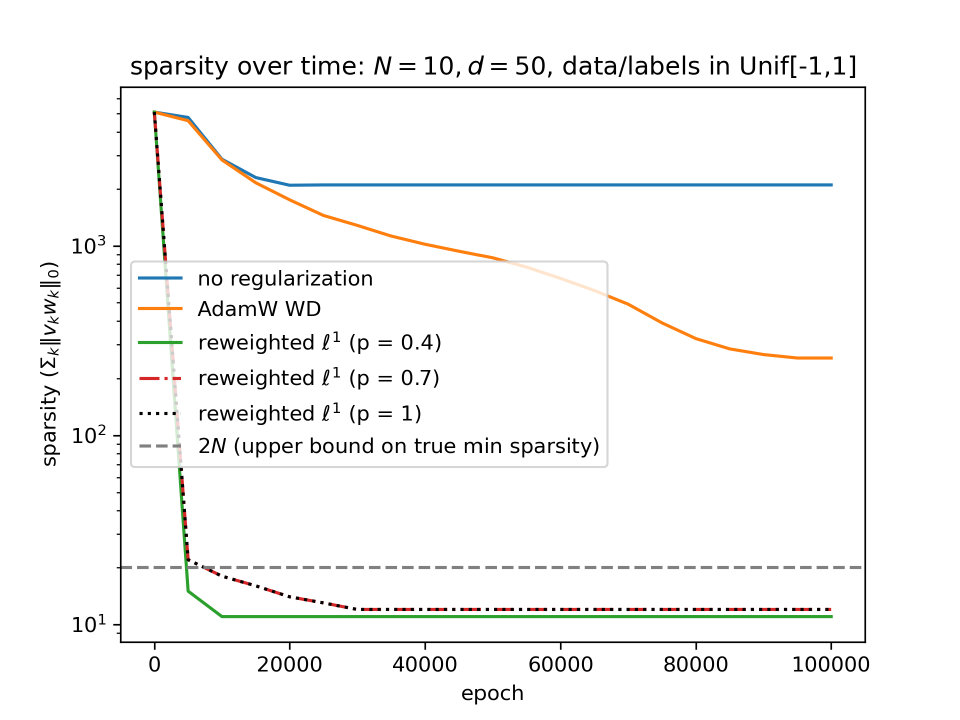}
    \caption{Sparsity over time of five networks trained to interpolation on $N = 10$ uniform random data points in $d = 50$. The solutions obtained by the $\ell^1$ algorithm (\cref{alg:rw_l1}) for $p \in \{0.4, 0.7, 1 \}$ satisfy the sparsity upper bound of $2N$ guaranteed by the proof of \cref{prop:width_invariance_multivar}.}
    \label{fig:sparsity_over_time_high_d}
\end{figure}

\end{appendices}
\end{document}